%% file: main.tex
\documentclass{article} 
\usepackage{iclr2027_conference,times}

\input{math_commands.tex}

\newcommand{\eg}{\textit{e.g.}}
\newcommand{\ie}{\textit{i.e.}}

\usepackage{microtype}
\usepackage{graphicx}
\usepackage{subcaption}
\usepackage{booktabs} 
\usepackage{colortbl}
\usepackage{xcolor}

\usepackage{mathrsfs}
\usepackage{multirow}
\usepackage{array}
\usepackage{algorithm}
\usepackage{algorithmic}
\usepackage{epstopdf}

\usepackage{amsmath}
\usepackage{amssymb}
\usepackage{mathtools}
\usepackage{amsthm}

\usepackage{bm}
\usepackage{wrapfig} 

\usepackage{stfloats}
\usepackage{enumitem}
\usepackage{makecell}
\usepackage{bbm}
\usepackage{hyperref}
\usepackage{url}
\usepackage[capitalize,noabbrev]{cleveref}

\usepackage{diagbox}
\usepackage[flushleft]{threeparttable}
\usepackage{threeparttable}
\definecolor{redorange}{rgb}{1, 0.5, 0}
\definecolor{darkgreen}{rgb}{0.0, 0.6, 0.0}
\definecolor{darkred}{rgb}{0.9, 0.0, 0.0}
\usepackage{pifont}
\usepackage{tikz}
\usepackage{bm}

\usepackage{wasysym}
\usepackage{marvosym}
\newcommand{\with}{\textcolor{darkgreen}{\ding{52}}}  
\newcommand{\without}{\textcolor{darkred}{\ding{56}}}

\theoremstyle{plain}
\newtheorem{theorem}{Theorem}
\newtheorem{proposition}{Proposition}

\newtheorem{assumption}{Assumption}
\theoremstyle{remark}

\AddToHook{cmd/appendix/before}{%
  \crefalias{section}{appendix}%
  \crefalias{subsection}{appendix}%
  \crefalias{subsubsection}{appendix}%
}

\title{Normality Constraint Learning: Adapting Foundation Models for Time Series Anomaly Detection}

\author{
	Xiaohui Zhou$^{1,2}$,
	Yijie Wang$^{1,2*}$,
	Hongzuo Xu$^{3}$,
	Weixuan Liang$^{2}$,
	Guansong Pang$^{4*}$ \\
	$^{1}$National Key Laboratory of Parallel and Distributed Computing \\
	$^{2}$College of Computer Science and Technology, National University of Defense Technology \\
	$^{3}$Intelligent Game and Decision Lab (IGDL) \\
	$^{4}$School of Computing and Information Systems, Singapore Management University
}

\iclrfinalcopy 
\begin{document}

\maketitle
\vspace{-3mm}
\begin{abstract}
Time Series Foundation Models (TSFMs) achieve strong generalization by learning to reconstruct or forecast broad temporal patterns from large-scale time series during pre-training.
Yet this strength can become a weakness for anomaly detection: TSFMs may model rare anomalous patterns as effectively as normal ones, allowing anomalies to be accurately reconstructed or forecasted and thus diminishing their reconstruction/forecasting error-based anomaly scores. 
This paper proposes \textbf{\underline{N}ormality \underline{C}onstraint \underline{L}earning} (\textbf{NCL}), a lightweight plug-and-play framework that adapts pre-trained TSFMs for accurate anomaly detection without modifying their pre-trained parameters. Our key insight is to constrain the broad pattern space of TSFMs to the normal structure of a target time series, preventing their broad modeling capability from obscuring abnormal deviations. Specifically, NCL constructs a compact normality subspace from a few normal patch features and adaptively steers each patch feature toward normality within this subspace, guided by contrastive constraints that form compact and discriminative normality manifolds. 
The calibrated features are aggregated to reinforce normal components and fused with the original TSFM output, amplifying the discrepancy between normal and abnormal observations for the reconstruction/forecasting error-based anomaly scoring.
Extensive experiments across diverse TSFM families and benchmarks show that NCL consistently improves anomaly detection performance, providing a generalizable framework for adapting TSFMs to anomaly detection.

\end{abstract}

\vspace{-2mm}
\section{Introduction}\label{intro}
Time series foundation models (TSFMs) have recently delivered substantial gains in the landscape of time series analysis. Pre-trained on massive and diverse corpora, TSFMs such as Aurora~\citep{wu2026aurora}, ROSE~\citep{wang2025towards}, Chronos~\citep{ansari2024chronos}, and MOMENT~\citep{goswami2024moment} encode rich universal representations of temporal dynamics and demonstrate strong zero-shot and few-shot generalization across various downstream tasks. Time series anomaly detection (TSAD) is one of the key downstream tasks of TSFMs. It aims to identify deviations from expected patterns, playing a critical role in many real-world applications such as industrial control~\citep{wang2022detecting}, disease identification~\citep{pereira2019learning}, and fraud detection~\citep{anandakrishnan2018anomaly}.

Although TSFMs have strong expressive power, their design is fundamentally reconstruction- or forecasting-centric, where internal features are shaped for broad pattern modeling. Since such general-purpose representation capability can often efficiently model both normal and abnormal patterns, the existing common practice of directly repurposing the reconstruction/forecasting errors of TSFMs for anomaly detection is questionable.
As illustrated in \cref{Effciency}, both reconstruction-based and forecasting-based TSFMs can effectively reconstruct or forecast most of the abnormal observations, resulting in small reconstruction/forecasting errors for not only normal data but also anomalies. This phenomenon fundamentally undermines the widely-used TSAD intuition that anomalies are harder to reconstruct or forecast than normal data~\citep{pang2021deep,zamanzadeh2024deep,liu2024elephant,zhu2026foundation}. We also provide a theoretical analysis of this issue in \cref{app:TSFM_AD}.
%
\begin{figure*}[!t]
	\centering
	\vspace{-3mm}
	\begin{subfigure}[t]{0.48\textwidth}
		\centering 
		\includegraphics[width=\linewidth]{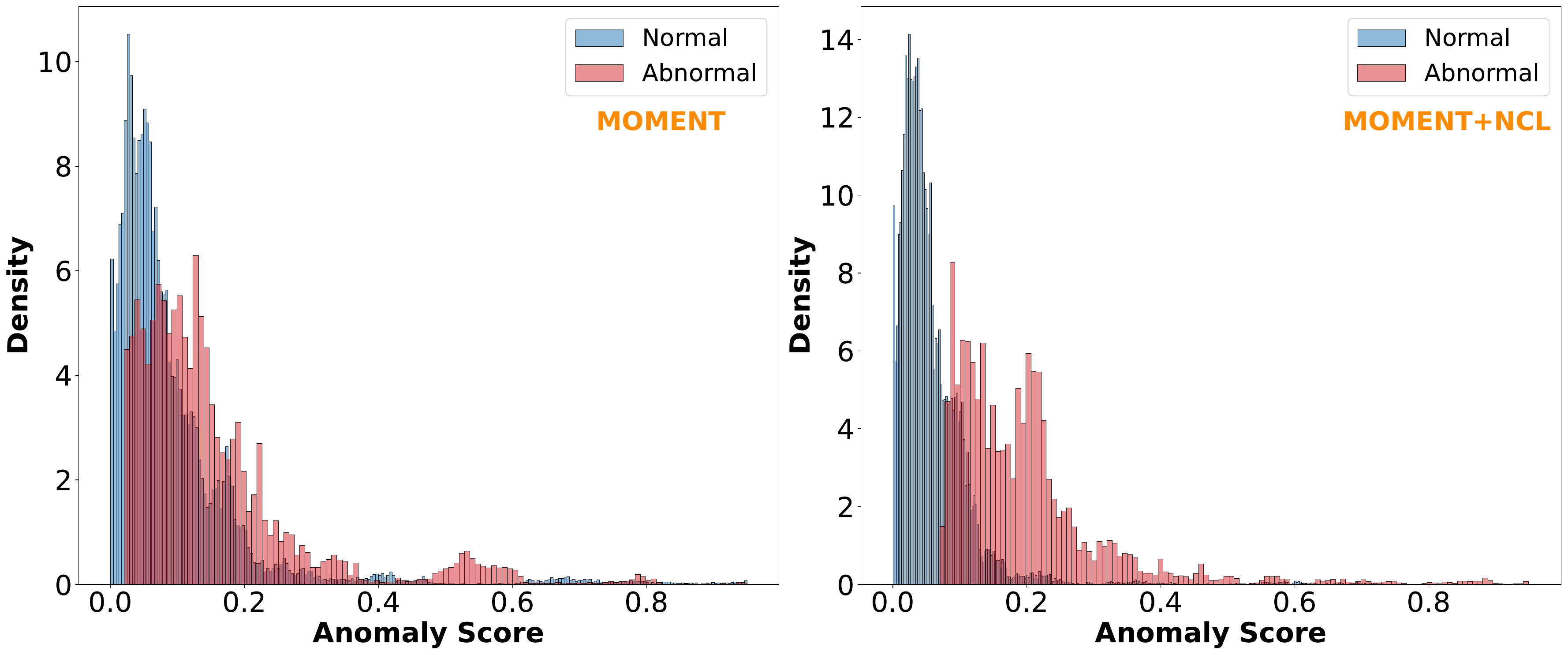} 
		\caption{Anomaly scores for reconstruction-based TSFMs.}
	\end{subfigure}
	\hspace{0.02\textwidth}
	\begin{subfigure}[t]{0.48\textwidth}
		\centering
		\includegraphics[width=\linewidth]{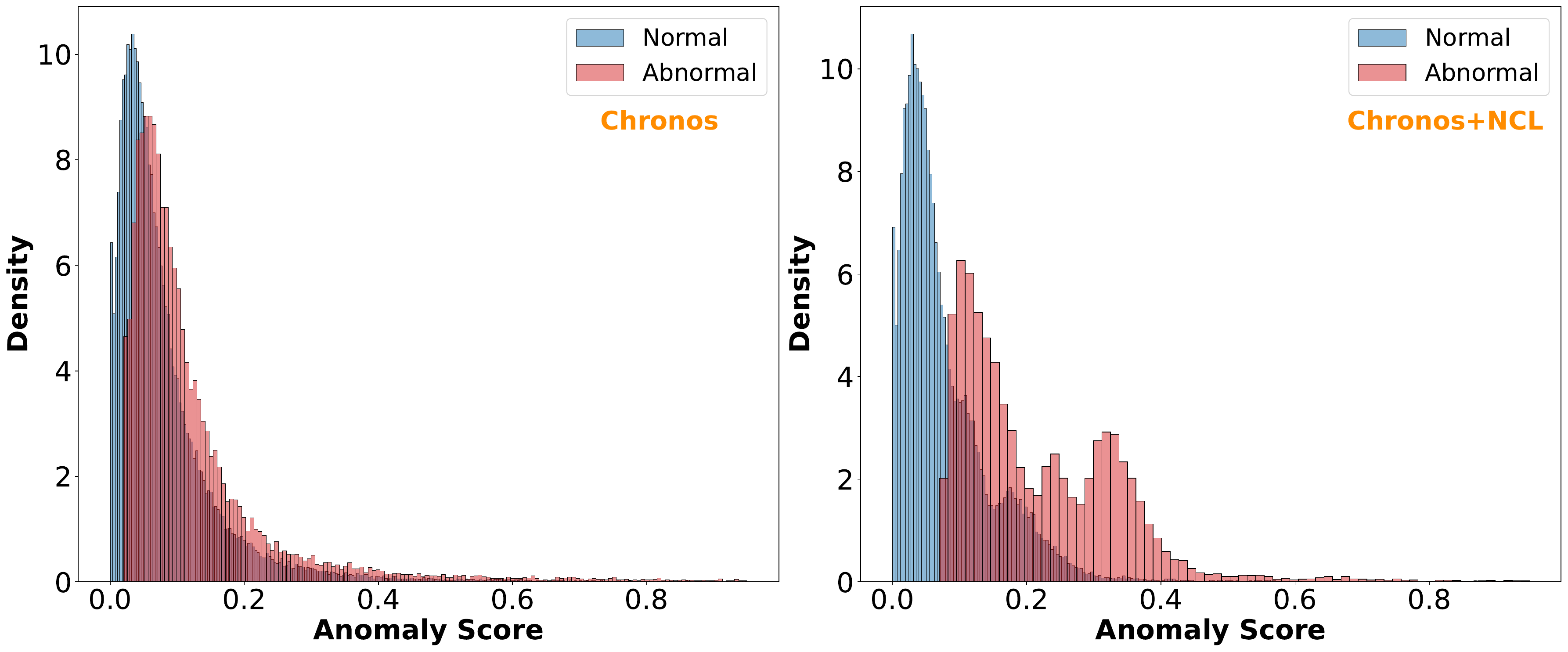}
		\caption{Anomaly scores for forecasting-based TSFMs.}
	\end{subfigure}%
	\caption{(a) and (b) respectively show the anomaly score distributions of normal and abnormal observations for reconstruction-based and forecasting-based TSFMs and their NCL-augmented counterparts on the PSM dataset~\citep{abdulaal2021practical}. NCL enables both types of TSFMs to produce substantially more discriminative anomaly scores. 
    }
	\label{Effciency}
    \vspace{-2mm}
\end{figure*}

To address this challenge, a natural idea is to adapt TSFMs for anomaly detection through fine-tuning strategies~\citep{hu2022lora,qiu2023controlling,park2026comparative}. However, fine-tuning pre-trained foundation models risks catastrophic forgetting of the rich temporal knowledge, on which the anomaly scoring mechanisms in the fine-tuned TSFMs still heavily rely when computing the reconstruction or forecasting errors.
This raises a critical question: \emph{How can we effectively leverage the powerful representational capacity of frozen TSFMs while steering their outputs toward a normality regime so that abnormal deviations stand out more clearly?}

In this paper, we propose \textbf{\underline{N}ormality \underline{C}onstraint \underline{L}earning} (\textbf{NCL}), a lightweight plug-and-play framework that can augment diverse TSFMs for accurate anomaly detection without modifying their pre-trained parameters. Our key insight is that, while TSFMs can capture a broad spectrum of temporal patterns, reconstruction/forecasting-error-based anomaly detection requires identifying and emphasizing the relevant normal patterns only; we therefore first construct a normality subspace from this broad pattern space and then adaptively steer each patch feature toward normality within the subspace, guided by compact and discriminative manifolds learned from a few normal patch features.

To this end, NCL 
is instantiated with three main components, namely \underline{norm}al sub\underline{space} construction (\textbf{NormSpace}), \underline{patch}-adaptive \underline{n}ormality \underline{t}ransformation (\textbf{PatchNT}), and \underline{con}trastive \underline{n}ormality \underline{c}onstraint (\textbf{ConNC}). NormSpace constructs a compact normality subspace via singular value decomposition (SVD) of patch features from a few training subsequences encoded by a frozen TSFM. 
Within this subspace, 
PatchNT adaptively calibrates patch features toward a learnable normality manifold along each principal direction, while ConNC guides this adaptive transformation by encouraging the formation of compact and discriminative normality manifolds.
The calibrated features across all normality directions are then aggregated into consensus features that reinforce normal components.
Finally, these features are decoded and fused with the original TSFM output to enhance the normal patterns, thereby amplifying the discrepancy between normal and abnormal data.
As shown in \cref{Effciency}, augmenting TSFMs with NCL yields more discriminative anomaly score distributions than vanilla TSFMs (see \cref{app:anomly_dis} for additional results).
In summary, our contributions are as follows:
\begin{itemize}
	\item We propose NCL, a novel plug-and-play framework that augments frozen TSFMs with normality constraints to learn calibrated normality features, thereby substantially sharpening the reconstruction/forecasting error-based anomaly scores of the TSFMs.
	\item We further instantiate NCL with three components: NormSpace, PatchNT, and ConNC. NormSpace constructs a normality subspace, while PatchNT adaptively transforms patch features toward the normality regime. This transformation is guided by ConNC to encourage the formation of compact yet discriminative normality manifolds within the regime.
	\item Extensive experiments on popular TSAD datasets demonstrate that NCL consistently and significantly improves the anomaly detection performance of diverse families of TSFMs.
\end{itemize}

\section{Related Work}
\label{related_work}
\subsection{Time Series Anomaly Detection}
Conventional TSAD approaches rely on statistical techniques and machine learning methods, which detect anomalies based on deviation from statistical properties or density estimates~\citep{breunig2000lof,yairi2001fault,liu2008isolation,hariri2019extended}. Recently, the advent of deep learning has significantly advanced TSAD capabilities. Reconstruction-based methods~\citep{CrossAD,wu2025catch,zhangicdiffad,huteamwork} learn to reconstruct normal patterns and identify anomalies through large reconstruction errors, while forecasting-based approaches~\citep{liu2025gcad,cho2025structured,zhou2026modeling} detect anomalies by measuring discrepancies between predicted and actual future observations. Meanwhile, contrastive learning~\citep{luo2026anomalytcn,mou2026roca,park2026paano} has emerged as a powerful paradigm for learning discriminative representations.

Despite these advances, existing methods typically require training a separate model for each new dataset or application
, \ie, they are trained from scratch on large-scale domain-specific normal data and evaluated on the testing dataset from the same domain. This strong dependence on abundant in-domain data severely limits their applicability to scenarios where such data are scarce or costly to obtain \citep{shentu2025towards, lan2026towards,he2026timeradar}.

\subsection{Time Series Foundation Models for Anomaly Detection}
The recent emergence of TSFMs provides a generalization path, modeling universal temporal representations through large-scale time series pre-training. Most existing TSFMs~\citep{gao2024units,shi2025time,wu2026aurora} are task-general, which are pre-trained on massive data with general objectives (\eg, reconstruction or forecasting), and are applied to various downstream tasks such as forecasting and anomaly detection. However, directly applying TSFMs to anomaly detection presents fundamental challenges. As noted by recent studies~\citep{liu2024elephant,zhu2026foundation}, task-general TSFMs effectively model both normal and anomalous patterns. This undermines the widely-used TSAD intuition that anomalies are harder to reconstruct or forecast than normal data.

More recently, TSAD-specific foundation models have also been explored~\citep{shentu2025towards,lan2026towards,he2026timeradar}. For instance, DADA~\citep{shentu2025towards} utilizes injected anomalies to explicitly distinguish between normal and abnormal patterns, while TimeRCD~\citep{lan2026towards} conducts relative context discrepancy pre-training on synthetic anomalies. TimeRadar~\citep{he2026timeradar} is focused on leveraging these synthetic anomalies to enforce contextual deviations. 
These methods require mass synthetic anomalies to enrich the pre-training corpus. However, Low-quality anomalies may cause the foundation models to specialize in artificial deviations that do not transfer to real-world distributions~\citep{xu2024calibrated,zhou2026impact}, thereby leading to high detection errors.

\section{Preliminaries}
\label{pre}

\textbf{Time Series Anomaly Detection (TSAD).} Given a time series $\mathbf{X} = (\mathbf{x}_1, \mathbf{x}_2, \ldots, \mathbf{x}_T) \in \mathbb{R}^{D \times T}$ with $D$ channels and $T$ ordered observations, the task of TSAD is to calculate an anomaly score $s_t \in \mathbb{R}$ for each $\mathbf{x}_t \in \mathbb{R}^D$ at timestamp $t$. The optimal goal is to assign each anomalous observation a larger anomaly score than all the normal observations.

\textbf{Time Series Foundation Models for Anomaly Detection.} Most TSFMs rely on data reconstruction and/or forecasting objectives in their training, and furthermore, reconstruction and forecasting errors are widely-used effective anomaly measures in TSAD. Thus, these models are the backend TSFMs in this work. Formally, a TSFM $\mathcal{F}$ typically regards a time series subsequence $\mathbf{X}_{t-L:t} = (\mathbf{x}_{t-L}, \mathbf{x}_{t-L+1}, \ldots, \mathbf{x}_{t-1}) \in \mathbb{R}^{D \times L}$ as a sample input, and produces either a reconstruction of the input $\hat{\mathbf{X}}_{t-L:t}$ or a forecast of future observations $\hat{\mathbf{X}}_{t:t+F}$ over a horizon $F$. For notational brevity, we unify both paradigms by letting $\mathbf{Y}_t$ denote the prediction target, \ie, $\mathbf{Y}_t = \mathbf{X}_{t-L:t}$ under reconstruction and $\mathbf{Y}_t = \mathbf{X}_{t:t+F}$ under forecasting. The corresponding model output is formulated as $\hat{\mathbf{Y}}_t = \mathcal{F}(\mathbf{X}_{t-L:t})$. The pointwise anomaly score is then derived from the discrepancy between $\mathbf{Y}_t$ and $\hat{\mathbf{Y}}_t$, commonly measured by the squared error $s_t = \|\mathbf{y}_t - \hat{\mathbf{y}}_t\|^2$, where $\mathbf{y}_t\in\mathbf{Y}_t$ and $\hat{\mathbf{y}}_t \in \hat{\mathbf{Y}}_t$.

\section{Methodology}
\label{method}

\subsection{Overview of NCL: Toward Normality-Constrained TSFMs}

TSFMs are pre-trained on large corpora to learn a variety of universal temporal patterns.
However, this generality allows them to reconstruct or forecast normal and anomalous data similarly well, so anomalies may not yield larger errors. 
To mitigate this issue, we introduce the plug-and-play \textbf{normality constraint learning (NCL)} module $\mathcal{G}$ to achieve normality-constrained TSFMs by steering the representational capacity of a TSFM $\mathcal{F}$ toward the normal regime out of its broad pattern space. 

As illustrated in \cref{overview}, NCL consists of three main components. Given a few normal training subsequences, 
its first component, \textcircled{\small{1}} \textbf{Normal Subspace Construction (NormSpace)}, identifies $k$ principal normal directions via SVD on their patch features $\mathbf{Z}$ extracted from the frozen TSFM $\mathcal{F}$. The \textcircled{\small{2}} \textbf{Patch-Adaptive Normality Transformation (PatchNT)} component is then devised to enforce a lightweight yet patch-adaptive alignment of each feature with a learnable normality manifold along each principal direction. The \textcircled{\small{3}} \textbf{Contrastive Normality Constraint (ConNC)} component is further proposed to constrain the transformed features 
tightly cluster around their respective normality modes. In doing so, NCL amplifies the discrepancy signal between normal and abnormal observations that the superior fitting capacity of TSFMs would otherwise diminish. All these three components are done without changing any of the pre-trained parameters of $\mathcal{F}$ itself. Let $\mathcal{F}^{\prime}$ denote the frozen TSFM augmented with the proposed NCL module, the predictions can then be defined as $\tilde{\mathbf{Y}}_t = \mathcal{F}^{\prime}(\mathbf{X}_{t-L:t})$.
\begin{figure}[h]  
	\centering
	\includegraphics[width=1\linewidth]{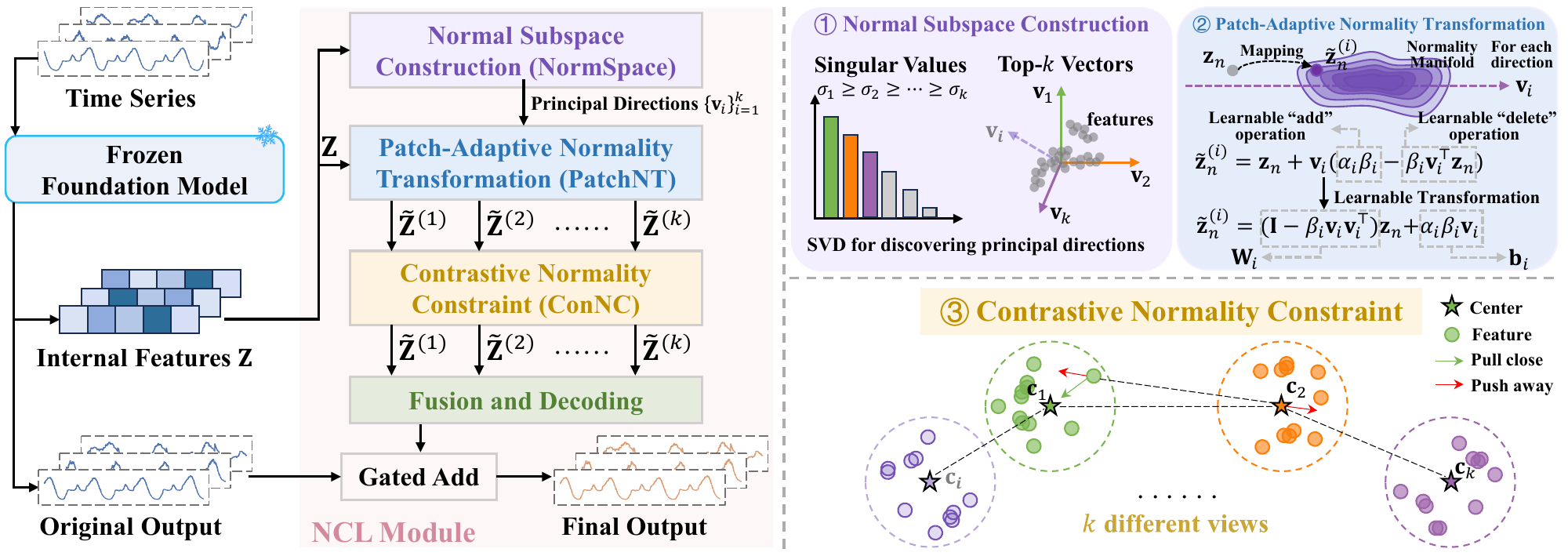} 
	\caption{Overview of NCL. \textcircled{\small{1}} Given few-shot normal training subsequences, a frozen TSFM is used to extract their patch features $\mathbf{Z}$, upon which NormSpace constructs a normal subspace of $k$ principal directions via SVD. \textcircled{\small{2}} PatchNT is then trained to perform patch-adaptive transformation of each feature onto a learnable normality manifold along each principal direction. \textcircled{\small{3}} ConNC is further leveraged to calibrate the transformed features via constraining their distributions tightly cluster around the normal prototypes $\{\mathbf{c}_i\}_{i=1}^{k}$. Finally, these calibrated features are synthesized with the original TSFM output to amplify the error discrepancy between normal and abnormal data.}
	\label{overview}
\end{figure}
\subsection{NormSpace: Normal Subspace Construction via Few-Shot SVD}
\label{sec:esd}
Let $\mathbf{Z}=\left[\mathbf{z}_{1}, \mathbf{z}_{2}, \ldots, \mathbf{z}_{N}\right]^{\top} \in \mathbb{R}^{N \times d}$ be the matrix of normal patch features extracted from $\mathcal{F}$, where $N$ is the total number of patches and $d$ is the feature dimension. To identify the most informative directions of variation within the normal feature space, we perform SVD on $\mathbf{Z}$:
\begin{equation}
	\mathbf{Z} = \mathbf{U}_r \boldsymbol{\Sigma}_r \mathbf{V}_r^\top = \sum_{i=1}^r \sigma_i \mathbf{u}_i \mathbf{v}_i^{\top},
\end{equation}
where $\mathbf{U}_r \in \mathbb{R}^{N \times r}$, $\boldsymbol{\Sigma}_r = \mathrm{diag}(\sigma_1, \sigma_2, \ldots, \sigma_r)$ with $\sigma_1 \geq \sigma_2 \geq \cdots \geq \sigma_r \geq 0$, and $\mathbf{V}_r \in \mathbb{R}^{d \times r}$, with $r = \text{rank}(\mathbf{Z})$. The columns of $\mathbf{V}_r$, denoted $\{\mathbf{v}_1, \mathbf{v}_2, \ldots, \mathbf{v}_r\}$, are the right singular vectors and correspond to the principal directions of $\mathbf{Z}$. 

The singular vectors with the largest singular values capture the dominant and stable patterns shared across all $N$ normal patch features, \ie,  the directions along which normal temporal structure is most concentrated. 
Therefore, we construct a normal-pattern-focused space by retaining the top-$k$ right singular vectors:
\begin{equation}
    \label{top-k}
	\{\mathbf{v}_i\}_{i=1}^{k}, \quad k \ll d,
\end{equation}
which collectively span a $k$-dimensional \emph{normal subspace}. These vectors are stored as running buffers during inference, ensuring that the normal subspace is estimated from the training distribution.
We theoretically show in \cref{app:the_analysis} that \textbf{(i)} performing SVD on the patch features from few-shot normal samples is sufficient to construct an informative normality space and \textbf{(ii)} choosing the top-$k$ vectors to span a $k$-dimensional normal subspace is most effective.

\subsection{PatchNT: Lightweight Patch-Adaptive Normality Transformation}
\label{sec:transform}

The normal subspace spanned by $\{\mathbf{v}_i\}_{i=1}^{k}$ captures the principal directions of normal temporal patterns, but SVD projection alone only preserves feature components along these directions without explicitly aligning them with normality. Moreover, different patches may deviate from normality in different ways and to different degrees, requiring patch-specific calibration. We therefore introduce patch-adaptive normality transformation (PatchNT), which uses only a few learnable parameters to selectively add or delete information from each patch feature, steering its component along $\mathbf{v}_i$ via these learnable ``add'' and ``delete'' operations toward a normality-consistent representation. Applying PatchNT separately along the $k$ directions then yields $k$ complementary views, each calibrating the patch feature toward the normal temporal structure captured along one direction.

Formally, for each patch feature $\mathbf{z}_{n} \in \mathbb{R}^d$, the $i$-th normality transformation proceeds as follows. First, the scalar projection of $\mathbf{z}_{n}$ onto $\mathbf{v}_i$ is computed:
\begin{equation}
	p_n^{(i)} = \mathbf{v}_i^\top \mathbf{z}_{n}.
\end{equation}
Next, inspired by deep delta learning~\citep{zhang2026deep} in residual update, two patch-adaptive vector parameters are defined to specify the learnable ``add'' and ``delete'' operations for steering
patch feature along each direction $\mathbf{v}_i$. Specifically, the target value $\boldsymbol{\alpha}$ and the calibration strength $\boldsymbol{\beta}$ are produced by lightweight linear heads applied to $\mathbf{z}_n$ for the feature steering in all $k$ directions:
\begin{equation}
	\boldsymbol{\alpha} = \sigma(\mathbf{W}_\alpha \mathbf{z}_{n} + \mathbf{b}_\alpha) \in (0, 1)^k, \quad
	\boldsymbol{\beta} = 2 \cdot \sigma(\mathbf{W}_\beta \mathbf{z}_{n} + \mathbf{b}_\beta) \in (0, 2)^k,
\end{equation}
where $\sigma(\cdot)$ is the sigmoid function, and $\mathbf{W}_\alpha, \mathbf{W}_\beta \in \mathbb{R}^{k \times d}$ are the learnable weights. Each target value $\alpha_i \in \boldsymbol{\alpha}$ specifies where the calibrated feature should lie along $\mathbf{v}_i$, while $\beta_i \in \boldsymbol{\beta}$ controls the degree of calibration. Then we can obtain the calibrated feature as follows:
\begin{equation}
	\tilde{\mathbf{z}}_{n}^{(i)} = \mathbf{z}_{n} + \beta_i \mathbf{v}_i(\alpha_i - p_n^{(i)})= \mathbf{z}_{n} + \beta_i \mathbf{v}_i \bigl(\alpha_i - \mathbf{v}_i^\top \mathbf{z}_{n}\bigr) = \mathbf{z}_{n} + \mathbf{v}_i \bigl(\underbrace{\alpha_i \beta_i}_{\text{add}} - \underbrace{\beta_i \mathbf{v}_i^\top \mathbf{z}_{n}}_{\text{delete}}\bigr),
	\label{eq:edit}
\end{equation}
indicating that the introduced patch-adaptive parameters can achieve both information addition and deletion along the direction $\mathbf{v}_i$. However, it is still unclear to what extent the $\beta_i$ will modify the relative distances of the feature from $\alpha_i$ in the direction $\mathbf{v}_i$. We further treat the calibration in \cref{eq:edit} as a rank-1 affine update:
\begin{equation}
	\tilde{\mathbf{z}}_{n}^{(i)} = \mathbf{z}_{n} + \mathbf{v}_i \bigl(\alpha_i \beta_i - \beta_i \mathbf{v}_i^\top \mathbf{z}_{n}\bigr) = (\mathbf{I} - \beta_i \mathbf{v}_i \mathbf{v}_i^\top) \mathbf{z}_{n} + \alpha_i \beta_i \mathbf{v}_i = \mathbf{W}_i \mathbf{z}_{n} + \mathbf{b}_i,
	\label{eq:transform}
\end{equation}
where $\mathbf{W}_i=\mathbf{I} - \beta_i \mathbf{v}_i \mathbf{v}_i^\top$ and $\mathbf{b}_i=\alpha_i \beta_i \mathbf{v}_i$. Based on \cref{eq:transform}, we can then characterize how $\beta_i$ governs the geometric behavior of the transformation by analyzing the spectral structure of $\mathbf{W}_i$. We further propose \cref{prop1} (see \cref{app:prop1} for its proof).
\begin{proposition}[Spectral and Geometric Property] \label{prop1}
{The spectrum of $\mathbf{W}_i=\mathbf{I} - \beta_i \mathbf{v}_i \mathbf{v}_i^\top$ is $\Lambda(\mathbf{W}_i) = \{1 - \beta_i, 1, \ldots, 1\}$ with one eigenvalue $(1-\beta_i)$ for eigenvector $\mathbf{v}_i$ and $(d-1)$ eigenvalues of $1$ in $\mathbf{v}_i^\perp = \{ \mathbf{k} \in \mathbb{R}^d \mid \mathbf{v}_i^\top\mathbf{k}=0 \}$. At $\beta_i = 1$, the affine update $\tilde{\mathbf{z}}_{n}^{(i)}=\mathbf{W}_i \mathbf{z}_{n} + \mathbf{b}_i$ exactly deletes the $\mathbf{v}_i$-component of $\mathbf{z}_{n}$ and adds $\alpha_i$. At $\beta_i < 1$ the update is partial and at $\beta_i > 1$ it slightly overshoots, allowing the model to learn under- or over-corrective behaviors where beneficial.}
\end{proposition}
The proposition shows that PatchNT offers a brand-new affine transformation that requires only selective modification of the $\mathbf{v}_i$-component of $\mathbf{z}$, leaving the orthogonality of all components intact. Applying such transformation to all $N$ patch features across all $k$ directions, we obtain $k$ normal views, each consisting of $N$ transformed features:
\begin{equation}
	\widetilde{\mathbf{Z}}^{(i)} = \bigl[\tilde{\mathbf{z}}_{1}^{(i)}, \tilde{\mathbf{z}}_{2}^{(i)}, \ldots, \tilde{\mathbf{z}}_{N}^{(i)}\bigr]^{\top} \in \mathbb{R}^{N \times d}, \quad i \in \{1, 2, \ldots, k\},
\end{equation}
where each $\tilde{\mathbf{z}}_{n}^{(i)}$ is adaptively calibrated toward normality captured along the $\mathbf{v}_i$ direction. Thus, $\widetilde{\mathbf{Z}}^{(i)}$ is effectively aligned with a direction-specific view of normality.
\subsection{ConNC: Contrastive Normality Constraint}
\label{sec:contrastive}
In the transformed feature space, we aim to promote coherence among patch features calibrated along the same direction while preserving distinctions between features calibrated along different directions. This helps each view to accurately align with the direction-specific normality and reduces redundancy across views. Contrastive learning naturally supports this objective by encouraging within-view compactness and between-view separation. Inspired by self-supervised transformation classification~\citep{he2018triplet,bergman2020goad}, we therefore introduce a contrastive normality constraint that uses the $k$ transformation views as self-supervised labels, encouraging the transformed features to form $k$ compact and discriminative normality manifolds.

Formally, let $\tilde{\mathbf{z}}_{n}^{(i)} \in \mathbb{R}^d$ denote the $i$-th transformation view of patch feature $\mathbf{z}_{n}$, the center $\mathbf{c}_{i}$ be the average feature over all $N$ normal patch features for the $i$-th transformation, \ie, $\mathbf{c}_{i} = \frac{1}{N} \sum_{n=1}^N \tilde{\mathbf{z}}_{n}^{(i)}$, and the contrastive normality constraint can then be formulated as follows:
\begin{equation}
	\mathcal{L}_\mathrm{ConNC} = \frac{1}{N} \sum_{n=1}^{N} \frac{1}{k} \sum_{i=1}^{k}
	\max(0,\; \|\tilde{\mathbf{z}}_n^{(i)} - \mathbf{c}_i\|_2 + \gamma
	- \min_{j \neq i} \|\tilde{\mathbf{z}}_n^{(i)} - \mathbf{c}_j\|_2),
	\label{eq:loss_nc}
\end{equation}
where $\gamma > 0$ is a margin hyperparameter. The above loss encourages each transformation view $\tilde{\mathbf{z}}_{n}^{(i)}$ to be closer to its own cluster center $\mathbf{c}_{i}$ than to any other cluster center $\mathbf{c}_{j}$ ($j \neq i$) by a margin of at least $\gamma$. This ensures that the $k$ normality-calibrated views are not only compact within each cluster but also mutually discriminative, forming a structured partition of the normal feature space that is difficult for anomalous features to satisfy.

\subsection{Training and Inference}
\label{sec:fusion}
\textbf{Training.} Given a normal patch feature $\mathbf{z}_{n}$, following the SVD projection and normality transformation stated above, we obtain its $k$ calibrated features, which are then aggregated into a single normal feature by taking their element-wise mean, \ie, $\tilde{\mathbf{z}}_{n} = \frac{1}{k} \sum_{i=1}^k \tilde{\mathbf{z}}_{n}^{(i)}$.
Fusing over these $k$ principal directions yields a consensus feature that reinforces the normality components. These fused patch features are then concatenated and passed through a linear head to produce the module prediction $\tilde{\mathbf{Y}}_t^{\mathrm{ncl}}$. 
To retain the strong modeling capacity of the original TSFM while amplifying the normality properties, we fuse the module prediction with the frozen TSFM prediction $\hat{\mathbf{Y}}_t$ as follows:
\begin{equation}
	\tilde{\mathbf{Y}}_t = \boldsymbol{\lambda} \odot \tilde{\mathbf{Y}}_t^{\mathrm{ncl}} + (1 - \boldsymbol{\lambda}) \odot \hat{\mathbf{Y}}_t,
	\label{eq:fusion}
\end{equation}
where $\odot$ denotes channel-wise broadcasting multiplication, and $\boldsymbol{\lambda} \in (0, 1)^{D}$ is a learnable gated weight that captures the channel-specific reliance on the normality constraint versus the original TSFM output. This fusion mechanism is to increase reconstruction/forecasting error for abnormal observations through the normality-constrained component $\tilde{\mathbf{Y}}_t^{\mathrm{ncl}}$. The key insight is that for abnormal input, the normality constraint pushes them away from the pre-training support $\mathcal{S}_j$, causing NCL-enhanced TSFMs to produce larger errors. This error amplification effect is theoretically justified via \cref{error_amplification}, with the proof provided in \cref{app::error_amplification}.
\begin{theorem}[Error Amplification]
	\label{error_amplification}
	Let $\mathcal{F}^*$ be a frozen TSFM pre-trained on $\mathcal{S}_j$ with Lipschitz constant $L_{\mathcal{F}}$, 
	and let $(\mathbf{X}_a, \mathbf{Y}_a) \in \mathrm{Supp}(\mathcal{P}_a) \subseteq \mathcal{N}_{\rho}(\mathcal{S}_j)$ 
	be an anomalous pair.
	Under assumption with linear expansion rate $\omega(r) = \kappa r$ where 
	$\kappa \in (0, L_{\mathcal{F}})$, for the fused prediction $\tilde{\mathbf{Y}}_a$, we have:
	\begin{equation}
		\|\mathbf{Y}_a - \tilde{\mathbf{Y}}_a\|_2^2 > \|\mathbf{Y}_a - \hat{\mathbf{Y}}_a\|_2^2,
	\end{equation}
	demonstrating that output fusion amplifies the reconstruction or forecasting errors on anomalies.
\end{theorem}
Accordingly, the normality transformation is optimized by minimizing the following training loss that combines the squared error (for either reconstruction or forecasting) loss with the contrastive normality constraint:
\begin{equation}
	\mathcal{L} = \mathcal{L}_{\mathrm{Pred}} + \mathcal{L}_{\mathrm{ConNC}},
	\label{eq:total_loss}
\end{equation}
where $\mathcal{L}_{\mathrm{Pred}}$ is the squared error between the fused prediction $\tilde{\mathbf{Y}}_t$ and the ground-truth target $\mathbf{Y}_t$. During the training process, only the parameters of the NCL module are updated, while all parameters of the TSFM $\mathcal{F}$ remain frozen. We also conduct detailed theoretical analysis in \cref{regularization} to demonstrate that, in case there is abnormal information in the training normal patch features, $\mathcal{L}_{\mathrm{ConNC}}$ can serve as a regularization term that prevents the NCL module from overfitting to it.

\textbf{Inference.} During inference, the anomaly score consists of two components. One is the amplified reconstruction/forecasting squared error between $\tilde{\mathbf{Y}}_t$ and $\mathbf{Y}_t$. Another anomaly scoring component is to leverage the misalignment between each test sample's normality-calibrated views and the centers. For patch features $\mathbf{Z}_{\mathrm{test}}=\{\mathbf{z}_{1}, \mathbf{z}_{2}, \ldots, \mathbf{z}_{M}\}$ from a test sample, this misalignment-based score for the sample is defined as the maximum over all $M$ patches of the average negative log-probability of each view being assigned to its own center:
\begin{equation}
	s = \mathop{{\max}}_{m \leq M}-\frac{1}{k}\sum_{i=1}^k\log \frac{e^{-\|\tilde{\mathbf{z}}_{m}^{(i)} - \mathbf{c}^{(i)}\|_2}}{\sum_{j=1}^k e^{-\|\tilde{\mathbf{z}}_{m}^{(i)} - \mathbf{c}^{(j)}\|_2}}.
\end{equation}
For a normal sample, each calibrated view $\tilde{\mathbf{z}}_{m}^{(i)}$ lies close to its corresponding center $\mathbf{c}^{(i)}$. For an abnormal one, the normality transformation cannot align the anomalous feature to center equally well. The maximum over all $M$ patches is used as the sample score, as the top anomalous patch conveys the most information about abnormality.

\section{Experiments}
\label{experiments}
\subsection{Experimental Setup}

\textbf{Datasets.} To evaluate the effectiveness of NCL, we conduct comprehensive experiments on ten TSAD datasets, including five univariate datasets—$\textbf{IOPS}$~\citep{ren2019time}, $\textbf{UCR}$~\citep{wu2021current}, $\textbf{NAB}$~\citep{ahmad2017unsupervised}, $\textbf{MSL}$~\citep{benecki2021detecting}, and $\textbf{SMAP}$~\citep{benecki2021detecting}—and five multivariate datasets—$\textbf{PSM}$~\citep{abdulaal2021practical}, $\textbf{SMD}$~\citep{su2019robust}, $\textbf{SWAT}$~\citep{mathur2016swat}, $\textbf{PUMP}$~\citep{feng2021time}, and $\textbf{NYC}$~\citep{cui2016comparative}. Details about the datasets are included in \cref{app:datasets}.

\textbf{Baselines.} We choose both forecasting-based and reconstruction-based TSFMs for anomaly detection as backbones. Forecasting-based TSFMs include Aurora~\citep{wu2026aurora}, ROSE~\citep{wang2025towards}, Chronos~\citep{ansari2024chronos}, and TTM~\citep{ekambaram2024tiny}. For reconstruction-based TSFMs, MOMENT~\citep{goswami2024moment} and UniTS~\citep{gao2024units} are selected. Additionally, we also include TSPulse~\citep{ekambaram2026tspulse} and SEMPO~\citep{he2025sempo}, which employ both forecasting and reconstruction. Other popular TSFMs, such as Time-MoE~\citep{shi2025time} and TimesFM~\citep{das2024decoder}, are also compared in \cref{eff_ana} and \cref{compare_with_tsad}. NCL-augmented TSFMs are also compared with deep TSAD methods, simple one-liner baselines, and TSAD-specific foundation models to further demonstrate the advances made by our NCL, with the results presented in \cref{compare_with_classic} and \ref{compare_with_tsad}. The detailed descriptions of these baselines can be found in \cref{app:baselines}.

\textbf{Implementation Details.} NCL is trained for five epochs with only five normal samples randomly selected from the training set (see \cref{app:utility_few_shot} for other few-shot sizes). All reported results are averaged over three independent runs. More implementation details are given in \cref{app:implementation}. 

\textbf{Evaluation Metrics.}
Recent works have highlighted that point-adjusted metrics may exhibit severe overestimation issues~\citep{liu2024elephant,qiu2025tab}, whereas VUS-PR~\citep{paparrizos2022volume} emerges as a more robust and fair evaluation measure. Therefore, we adopt VUS-PR as our primary metric. Several other PR and ROC metrics are reported in \cref{all_results}.

\subsection{Main Results}

\noindent\textbf{Enabling Anomaly Detection of Diverse TSFMs.}
\input{tab/main_results.tex}
\cref{tab:main} presents the results of eight different TSFMs with/without using our proposed NCL across ten TSAD datasets. The full results (including several PR and ROC metric results with the standard deviations) are provided in \cref{all_results}. It is clear that NCL consistently and significantly improves the anomaly detection performance of diverse TSFMs, with an average improvement ratio of 41.49\% across all backbone models. This dramatic enhancement demonstrates NCL's ability to address the fundamental limitation of TSFMs in distinguishing normal and abnormal patterns.

\noindent\textbf{Comparison with Other TSFM Adaptation Methods.}
To demonstrate the superiority of NCL over alternative adaptation strategies, we compare it against two popular TSFM adaptation methods, including LoRA~\citep{hu2022lora} and CoRA~\citep{cheng2026cora}. As shown in \cref{fig:performance_compare}, NCL consistently outperforms these methods across three representative TSFMs. LoRA can further utilize the limited normal samples to fine-tune TSFMs, while CoRA can enhance the modeling ability across channels; however, neither of them has an explicit normality constraint design for anomaly detection. In contrast, NCL constructs a sample-efficient normality space and constrains the features to align with its normality manifolds, thereby making abnormal deviations stand out more clearly. 
\begin{figure*}[t]
	\centering
	\includegraphics[width=0.93\linewidth]{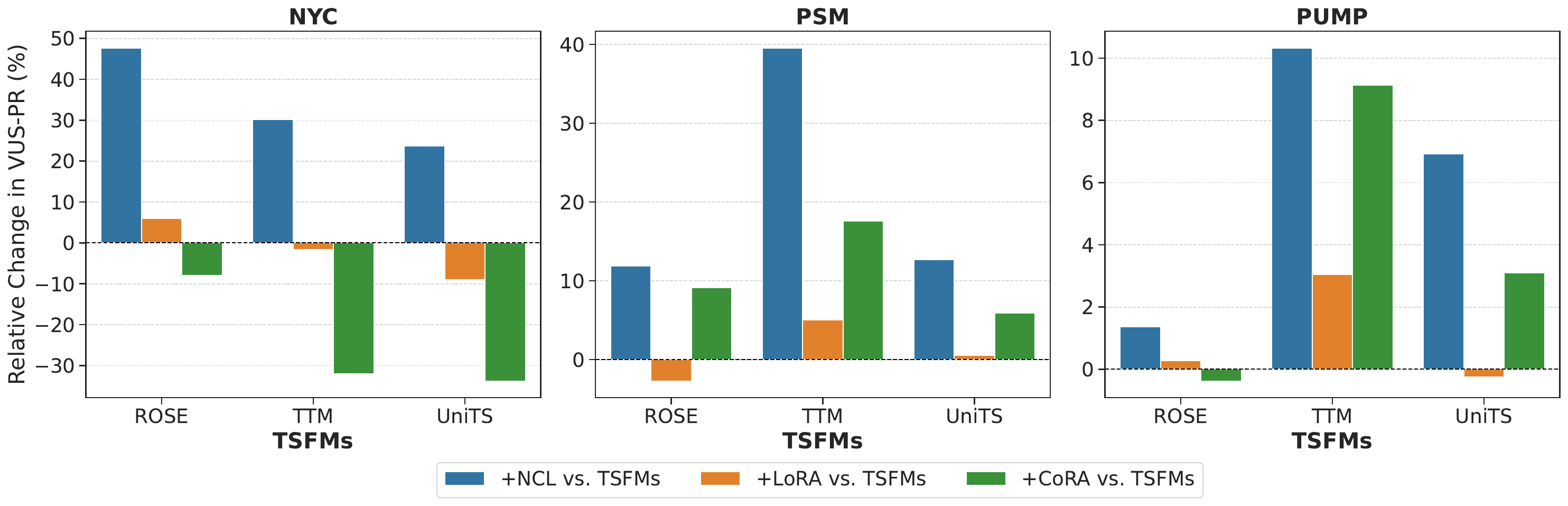}
	\caption{Performance changes relative to vanilla TSFMs for three different TSFMs.}
	\label{fig:performance_compare}
\end{figure*}

\subsection{Ablation Studies}
To validate the contribution of each component in NCL, we conduct comprehensive ablation studies by systematically removing or replacing its individual components. \cref{tab:ablation_studies} presents the results on IOPS and PSM datasets with three representative TSFMs (ROSE, TMM, UniTS).

\noindent\textbf{Normality Grounded by NormSpace.} As a NCL variant that uses only NormSpace,
relying on the top-$k$ singular vectors to perform reconstruction-based anomaly detection directly, $\text{NCL}_{\text{N}}$ gains consistent improvement compared to all three vanilla TSFMs. This indicates that, although only a few patch features are used, NormSpace can achieve a strong sample-efficient normality grounding that effectively narrows down the broad pattern space of TSFMs onto the normality space.
This also confirms Theorems \ref{thm:few_shot_subspace} and \ref{theorem1} in \cref{app:the_analysis}. Its significance can also be observed when \textbf{(i)} comparing $\text{NCL}_{\text{P}}$ to  $\text{NCL}_{\text{NP}}$ or \textbf{(ii)} comparing $\text{NCL}_{\text{PC}}$ to the full NCL method $\text{NCL}_{\text{NPC}}$.
\begin{wraptable}{r}{0.6\columnwidth}
\input{tab/ablation_results.tex}
\vspace{-2mm}
\end{wraptable}
\noindent\textbf{Learnable Transformation to Calibrate Features via PatchNT.} Adding the learnable transformation in PatchNT on top of $\text{NCL}_{\text{N}}$, $\text{NCL}_{\text{NP}}$ further effectively calibrates the features in the normality space across all three TSFMs on both datasets. Even without using NormSpace, the adaptive transformation in PatchNT works well for the calibration as highlighted by the difference between $\text{NCL}_{\text{C}}$ (which uses random affine transformation to replace PatchNT) and $\text{NCL}_{\text{PC}}$. Besides, the features transformed solely by $\text{NCL}_{\text{P}}$ are also better than the vanilla TSFM features. Collectively, these results systematically justify the effectiveness of our learnable transformation design.

\noindent\textbf{Sharpening the Feature Normality Further via ConNC.} 
Adding ConNC on top of $\text{NCL}_{\text{NP}}$, our full method $\text{NCL}_{\text{NPC}}$ gains consistently large improvements across all cases, showcasing the importance of the contrastive constraint for enhancing the discriminability of the calibrated features. This constraint can also work well solely (\ie, $\text{NCL}_{\text{C}}$) compared to the vanilla TSFMs. Overall, these well justify the role of ConNC in guiding PatchNT to further sharpen the normality within the subspace.

\subsection{Efficiency Analysis}\label{eff_ana}

NCL enjoys not only substantial detection improvements but also parameter and training efficiency.
\cref{fig:mean_compare} compares its training costs with LoRA, CoRA, and full fine-tuning TSFMs, with the results averaged over the five multivariate datasets. Our key observations are as follows.
\begin{wrapfigure}{r}{0.5\textwidth}
	\centering
	\includegraphics[width=1.0\linewidth]{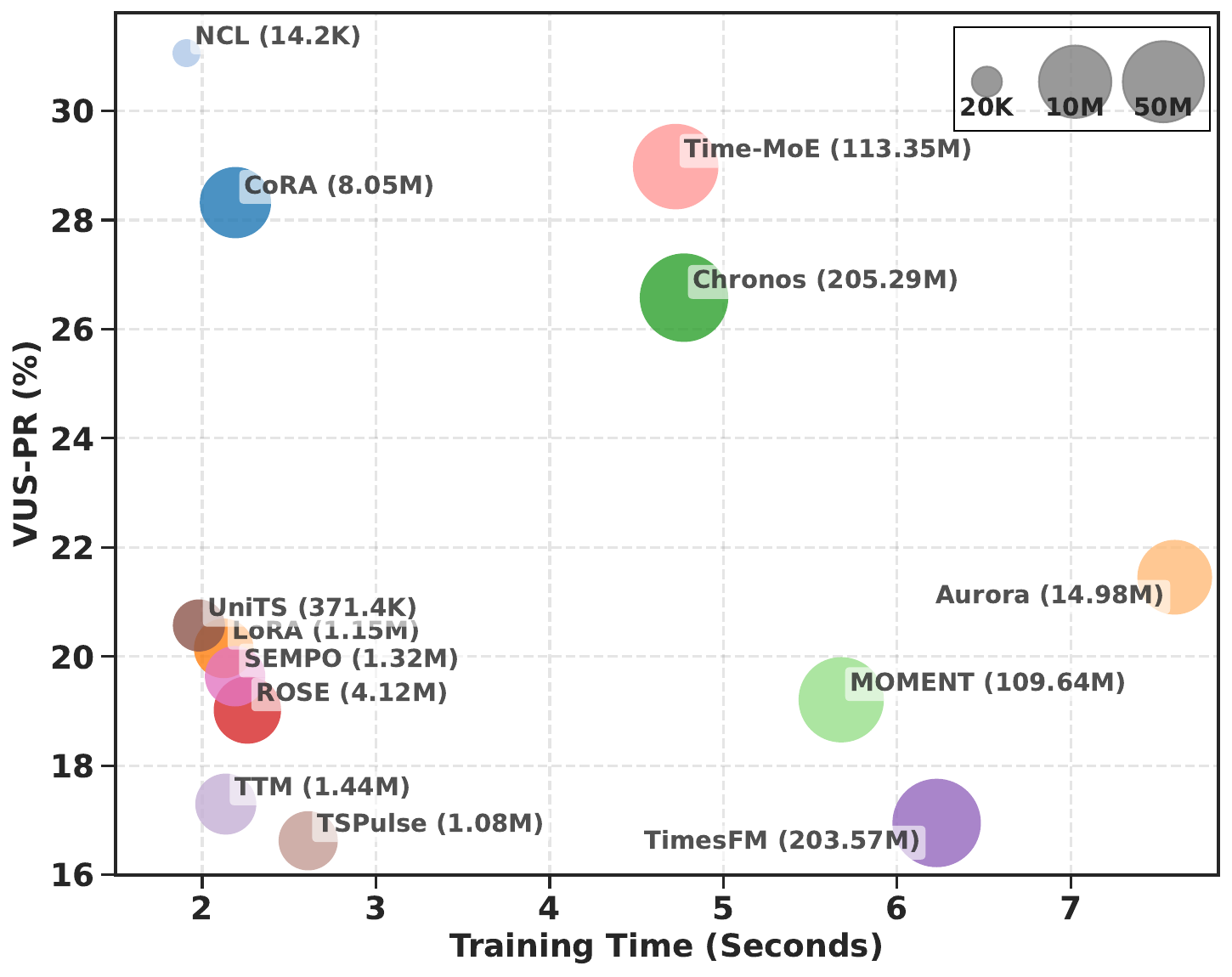}
	\caption{Effectiveness vs. efficiency trade-off.}
	\label{fig:mean_compare}
\vspace{-6mm}
\end{wrapfigure}
\textbf{Parameter Efficiency.} NCL requires an average of only \textbf{14.2K} trainable parameters across different TSFMs, which is significantly fewer than both the parameter-efficient fine-tuning methods like LoRA (1.15M) and CoRA (8.05M) and the full fine-tuning (from 371.4K to 205.29M).

\textbf{Training Speed.} NCL exhibits remarkably fast training time, completing its full training within around \textbf{2 seconds}. The efficiency stems from two main factors: (i) NormSpace requires only few-shot samples to construct an informative normality space and (ii) only lightweight projection heads and fusion weights are learned while the TSFM remains frozen.

Overall, represented by a small bubble in the upper-left of the plot (denoting a lower number of training parameters), NCL achieves substantially better detection effectiveness (higher on the y-axis) while also having a fast training time (further left on the x-axis). In contrast, most competitors concentrate in the lower-right region, requiring longer training times yet inferior VUS-PR. More detailed efficiency comparisons are in \cref{app:efficiency_comparison}.
\section{Conclusion}
We propose NCL, a lightweight plug-and-play framework that empowers frozen TSFMs to achieve accurate anomaly detection. NCL is implemented with NormSpace, PatchNT, and ConNC. NormSpace constructs a normality subspace, while PatchNT adaptively steers patch features toward normality within this subspace, guided by ConNC to encourage the formation of compact yet discriminative normality manifolds. 
We justify the effectiveness of NCL in constraining TSFMs into the normality space through both theoretical and empirical evidence. 
These results imply that further specialization/adaptation of pre-trained TSFMs is required for the downstream anomaly detection task, and NCL provides a generalizable framework and a new perspective on tackling this issue.

\subsection*{AI use statement}
In this work, we used generative AI tools for language polishing of the paper. We have not used generative AI tools for formulating mathematical claims, assisting in the writing of proofs, proposing or refining hypotheses, or implementing methods; and the remaining required-disclosure tasks are not applicable to this work. We have reviewed all AI-assisted work. Specifically, all AI-assisted text was checked and revised by the authors to ensure fidelity to the intended meaning; the AI tools were used solely to improve grammar, clarity, and fluency, and did not generate new scientific content. No AI-generated text was incorporated without author verification. We take responsibility for the final content of this work, including text, claims or artifacts produced with the aid of generative AI.

\subsection*{Ethics statement}
This work provides a methodological contribution and relies exclusively on publicly available, non-proprietary datasets used under their original licenses. It involves neither human subjects nor sensitive data. We believe our results advance the understanding of specialized foundation models and time series anomaly detection while raising no additional ethical concerns.

\subsection*{Reproducibility statement}
We provide the hyperparameter configurations and implementation details for our experiments in \cref{detail_setup}. The core algorithmic component of our method is described in \cref{app:algorithms}. The theoretical results presented in this paper, including their proofs and assumptions, are detailed in \cref{app:TSFM_AD}, \ref{app:the_analysis}, and \ref{app:prop1}. The source code has been submitted as supplemental material and will be made publicly available upon acceptance of this paper.

\bibliography{iclr2027_conference}
\bibliographystyle{iclr2027_conference}

\newpage
\appendix

\section{Theoretical Analysis of Applying TSFMs to Anomaly Detection}
\label{app:TSFM_AD}
In \cref{intro}, we emphasize that both reconstruction-based and forecasting-based TSFMs can effectively reconstruct or forecast most of the abnormal observations, meaning that anomalies do not necessarily incur larger reconstruction or forecasting errors. We attribute this issue to two aspects: \textbf{(i)} TSFMs are typically pre-trained on large-scale time-series datasets encompassing diverse patterns and domains. Consequently, observations that constitute anomalies in a specific application domain may represent normal patterns in other domains within the pre-training corpus, enabling the model to reconstruct or forecast them with high fidelity; and \textbf{(ii)} the large parameter capacity of TSFMs allows them to further capture and internalize a wide range of general temporal patterns. 

These two factors jointly lead to the extremely broad pattern modeling ability of TSFMs, making anomalies not inherently difficult to reconstruct or forecast. This undermines the fundamental assumption that anomalies should yield higher reconstruction or forecasting errors, thereby limiting the effectiveness of applying TSFMs to anomaly detection.
In order to further understand the influence of these two aspects on the broad pattern modeling ability of TSFMs during the pre-training process, we conducted a detailed theoretical analysis in \cref{app:pre_train}.
\subsection{Pre-Training and Anomaly Detection on the TSFMs}
\label{app:pre_train}
Let $\{\mathcal{P}_i\}_{i=1}^{s}$ represent the collection of various time series distributions in a large-scale corpus. Then we know the empirical distribution of the entire pre-training data is $\mathcal{P}_{pre}=\sum_{i=1}^s \eta_i \mathcal{P}_i$, where $\sum_{i=1}^s \eta_i=1$. Pre-training a reconstruction-based or forecasting-based TSFM $\mathcal{F}$ aims to minimize the following objective:
\begin{equation}
	\mathcal{L}_{\mathrm{Pred}}(\mathcal{F}(\cdot)) = \sum_{i=1}^s \eta_i \mathbb{E}_{(\mathbf{X},\mathbf{Y})\sim {\mathcal{P}_i}}\left[ \|\mathbf{Y} - \mathcal{F}(\mathbf{X}) \|_2^2 \right].
	\label{eq:objective}
\end{equation}
By the universal approximation theorem~\citep{hornik1991approximation}, for any continuous function $f$ and any $\epsilon > 0$, there exists a neural network $\mathcal{F}$ with sufficient capacity such that
\begin{equation}
	\sup_{\mathbf{X} \in \mathcal{K}} \|\mathcal{F}(\mathbf{X}) - f(\mathbf{X})\|_2^2 \leq \epsilon,
\end{equation}
where $\mathcal{K}$ is a compact subset. 
For each distribution $\mathcal{P}_i$ in the pre-training corpus, let $f_i^*(\cdot)$ denote the optimal predictor that minimizes the expected squared error:
\begin{equation}
	f_i^*(\mathbf{X}) = \mathbb{E}_{\mathbf{Y}|\mathbf{X} \sim \mathcal{P}_i}[\mathbf{Y}|\mathbf{X}].
\end{equation}
Given sufficient model capacity and training data, the pre-trained TSFM $\mathcal{F}$ approximates a weighted combination of these optimal predictors. Specifically, minimizing \cref{eq:objective} yields:
\begin{equation}
	\mathcal{F}^*(\mathbf{X}) \approx \sum_{i=1}^s \eta_i f_i^*(\mathbf{X}) = \mathbb{E}_{\mathbf{Y}|\mathbf{X} \sim \mathcal{P}_{pre}}[\mathbf{Y}|\mathbf{X}],
	\label{eq:optimal_predictor}
\end{equation}
where the approximation becomes exact as model capacity increases.

Then we consider a downstream anomaly detection task on a target distribution $\mathcal{P}_{t}$, where anomalies correspond to observations drawn from an anomalous distribution $\mathcal{P}_{a}$. The fundamental challenge arises when the support of $\mathcal{P}_{a}$ overlaps with or lies close to the support of pre-training distributions.
Let $\mathcal{S}_j = \mathrm{Supp}(\mathcal{P}_j)$ denote the support of the $j$-th pre-training distribution in the input-output space. For $\rho \geq 0$, we define its $\rho$-neighborhood as:
\begin{equation}
	\mathcal{N}_{\rho}(\mathcal{S}_j) = \left\{(\mathbf{X}, \mathbf{Y}) : \inf_{(\widetilde{\mathbf{X}}, \widetilde{\mathbf{Y}}) \in \mathcal{S}_j} \left\|(\mathbf{X}, \mathbf{Y}) - (\widetilde{\mathbf{X}}, \widetilde{\mathbf{Y}})\right\|_2 \leq \rho\right\}.
\end{equation}
Suppose the support of the anomalous distribution is contained within this neighborhood, \ie, $\mathrm{Supp}(\mathcal{P}_{a}) \subseteq \mathcal{N}_{\rho}(\mathcal{S}_j)$, for some $j \in \{1, \ldots, s\}$ and sufficiently small $\rho > 0$. This condition means that, although these observations are abnormal in the downstream domain, their temporal patterns are close to patterns that appeared in the pre-training corpus.
Assume the pre-trained model achieves uniformly small prediction error on $\mathcal{S}_j$:
\begin{equation}
	\sup_{(\mathbf{X}, \mathbf{Y}) \in \mathcal{S}_j} \|\mathbf{Y} - \mathcal{F}^*(\mathbf{X})\|_2 \leq \sup_{(\mathbf{X}, \mathbf{Y}) \in \mathcal{S}_j} \|\mathbf{Y} - f^*(\mathbf{X})\|_2 + \|f^*(\mathbf{X}) - \mathcal{F}^*(\mathbf{X})\|_2 \leq \delta + \epsilon_j,
	\label{eq:uniform_error}
\end{equation}
where $\delta$ is the irreducible uncertainty error, and $\epsilon_j$ is the approximation error.
Furthermore, assume $\mathcal{F}^*$ is $L_{\mathcal{F}}$-Lipschitz continuous on $\mathcal{N}_{\rho}(\mathcal{S}_j)$:
\begin{equation}
	\|\mathcal{F}^*(\mathbf{X}_1) - \mathcal{F}^*(\mathbf{X}_2)\|_2 \leq L_{\mathcal{F}} \|\mathbf{X}_1 - \mathbf{X}_2\|_2.
	\label{eq:lipschitz}
\end{equation}
For any anomalous pair $(\mathbf{X}_{a}, \mathbf{Y}_{a}) \in \mathrm{Supp}(\mathcal{P}_{a})$, since $\mathrm{Supp}(\mathcal{P}_{a}) \subseteq \mathcal{N}_{\rho}(\mathcal{S}_j)$, there exists $(\widetilde{\mathbf{X}}, \widetilde{\mathbf{Y}}) \in \mathcal{S}_j$ such that
\begin{equation}
	\|\mathbf{X}_{a} - \widetilde{\mathbf{X}}\|_2 \leq \rho, \quad \|\mathbf{Y}_{a} - \widetilde{\mathbf{Y}}\|_2 \leq \rho.
	\label{supp}
\end{equation}
By the triangle inequality:
\begin{align}
	\|\mathbf{Y}_{a} - \mathcal{F}^*(\mathbf{X}_{a})\|_2 
	&\leq \|\mathbf{Y}_{a} - \widetilde{\mathbf{Y}}\|_2 + \|\widetilde{\mathbf{Y}} - \mathcal{F}^*(\widetilde{\mathbf{X}})\|_2 + \|\mathcal{F}^*(\widetilde{\mathbf{X}}) - \mathcal{F}^*(\mathbf{X}_{a})\|_2 \nonumber\\
	&\leq \rho + \delta + \epsilon_j + L_{\mathcal{F}}\rho \nonumber\\
	&= \epsilon_j + \delta + (1 + L_{\mathcal{F}})\rho,
	\label{eq:anomaly_error_bound}
\end{align}
where the second inequality follows from: (i) $\|\mathbf{Y}_{a} - \widetilde{\mathbf{Y}}\|_2 \leq \rho$ by Inequation (\ref{supp}); (ii) $\|\widetilde{\mathbf{Y}} - \mathcal{F}^*(\widetilde{\mathbf{X}})\|_2 \leq \delta + \epsilon_j$ since $(\widetilde{\mathbf{X}}, \widetilde{\mathbf{Y}}) \in \mathcal{S}_j$ and the uniform error bound in Inequation (\ref{eq:uniform_error}) holds for all points in $\mathcal{S}_j$; and (iii) $\|\mathcal{F}^*(\widetilde{\mathbf{X}}) - \mathcal{F}^*(\mathbf{X}_{a})\|_2 \leq L_{\mathcal{F}}\|\widetilde{\mathbf{X}} - \mathbf{X}_{a}\|_2 \leq L_{\mathcal{F}}\rho$ by the Lipschitz continuity in Inequation (\ref{eq:lipschitz}).

Consequently, the squared error satisfies:
\begin{equation}
	\|\mathbf{Y}_{a} - \mathcal{F}^*(\mathbf{X}_{a})\|_2^2 \leq \left[\epsilon_j + \delta + (1 + L_{\mathcal{F}})\rho\right]^2.
	\label{eq:squared_error_bound}
\end{equation}
Taking expectations over $\mathcal{P}_{a}$:
\begin{equation}
	\mathbb{E}_{(\mathbf{X}, \mathbf{Y}) \sim \mathcal{P}_{a}}\left[\|\mathbf{Y} - \mathcal{F}^*(\mathbf{X})\|_2^2\right] \leq \left[\epsilon_j + \delta + (1 + L_{\mathcal{F}})\rho\right]^2.
	\label{eq:expected_anomaly_error}
\end{equation}

This reveals the over-generalization phenomenon. In non-extremely unstable scenarios, $\delta$ is usually very small. When both $\epsilon_j$ and $\rho$ are small, anomalies that fall near the support of pre-training data can also be reconstructed or forecasted with high accuracy, making reconstruction or forecasting error an unreliable anomaly score.
\subsection{Theoretical Justification of Normality Constraint Learning}
\label{app:ncl_theory}
Having established that TSFMs suffer from over-generalization in \cref{app:pre_train}, we now demonstrate how the proposed normality constraint learning mitigates this issue through two complementary mechanisms: \textbf{(i)} error amplification on anomalies via output fusion, and \textbf{(ii)} regularization against anomaly overfitting via contrastive normality constraint.

\subsubsection{Error Amplification via Output Fusion (Proof of \cref{error_amplification})}
\label{app::error_amplification}
\begin{proof}
Recall from \cref{eq:fusion} that the final prediction is formed by:
\begin{equation}
	\tilde{\mathbf{Y}}_t = \boldsymbol{\lambda} \odot \tilde{\mathbf{Y}}_t^{\mathrm{ncl}} + (1 - \boldsymbol{\lambda}) \odot \hat{\mathbf{Y}}_t,
\end{equation}
where $\tilde{\mathbf{Y}}_t^{\mathrm{ncl}}$ is the normality-constrained prediction and $\hat{\mathbf{Y}}_t = \mathcal{F}^*(\mathbf{X}_t)$ is the frozen TSFM prediction. For notational simplicity, we denote the channel-wise fusion weight as $\lambda \in (0, 1)$ for a single-channel analysis.

For a normal pair $(\mathbf{X}_n, \mathbf{Y}_n) \in \mathrm{Supp}(\mathcal{P}_t)$, the normality constraint learning module is trained to minimize $\mathcal{L}_{\mathrm{Pred}}$ in \cref{eq:total_loss}. Since the transformation $\tilde{\mathbf{z}}_{n}^{(i)} = \mathbf{W}_i \mathbf{z}_{n} + \mathbf{b}_i$ projects patch features onto the normal manifold spanned by $\{\mathbf{v}_i\}_{i=1}^{k}$, the module prediction $\tilde{\mathbf{Y}}_n^{\mathrm{ncl}}$ approximates the normal conditional expectation:
\begin{equation}
	\tilde{\mathbf{Y}}_n^{\mathrm{ncl}} \approx \mathbb{E}_{\mathbf{Y}|\mathbf{X}_n \sim \mathcal{P}_t}[\mathbf{Y}|\mathbf{X}_n] = f_t^*(\mathbf{X}_n).
\end{equation}
Assuming both $\hat{\mathbf{Y}}_n$ and $\tilde{\mathbf{Y}}_n^{\mathrm{ncl}}$ achieve small errors on normal data:
\begin{equation}
	\|\mathbf{Y}_n - \hat{\mathbf{Y}}_n\|_2 \leq \epsilon_{\mathcal{F}}, \quad \|\mathbf{Y}_n - \tilde{\mathbf{Y}}_n^{\mathrm{ncl}}\|_2 \leq \epsilon_{\mathrm{ncl}},
\end{equation}
the fused prediction error satisfies:
\begin{align}
	\|\mathbf{Y}_n - \tilde{\mathbf{Y}}_n\|_2 
	&= \|\mathbf{Y}_n - \lambda \tilde{\mathbf{Y}}_n^{\mathrm{ncl}} - (1-\lambda) \hat{\mathbf{Y}}_n\|_2 \nonumber\\
	&\leq \lambda \|\mathbf{Y}_n - \tilde{\mathbf{Y}}_n^{\mathrm{ncl}}\|_2 + (1-\lambda) \|\mathbf{Y}_n - \hat{\mathbf{Y}}_n\|_2 \nonumber\\
	&\leq \lambda \epsilon_{\mathrm{ncl}} + (1-\lambda) \epsilon_{\mathcal{F}}.
	\label{eq:normal_error}
\end{align}
Thus, the fused prediction maintains low error on normal data.

For an anomalous pair $(\mathbf{X}_a, \mathbf{Y}_a) \in \mathrm{Supp}(\mathcal{P}_a)$, the frozen TSFM prediction $\hat{\mathbf{Y}}_a$ achieves a small error due to over-generalization (Inequation (\ref{eq:anomaly_error_bound})):
\begin{equation}
	\label{error}
	\|\mathbf{Y}_a - \hat{\mathbf{Y}}_a\|_2 \leq \epsilon_j + \delta + (1 + L_{\mathcal{F}})\rho.
\end{equation}
Recall that $\mathrm{Supp}(\mathcal{P}_a) \subseteq \mathcal{N}_{\rho}(\mathcal{S}_j)$, meaning abnormal features lie within a $\rho$-neighborhood of the $j$-th pre-training distribution support. The normality constraint learning module transforms abnormal patch features $\mathbf{z}_a$ via:
\begin{equation}
	\tilde{\mathbf{z}}_a^{(i)} = (\mathbf{I} - \beta_i \mathbf{v}_i \mathbf{v}_i^\top) \mathbf{z}_a + \alpha_i \beta_i \mathbf{v}_i,
\end{equation}
where $\{\mathbf{v}_i\}_{i=1}^k$ are the principal directions of the normal distribution from $\mathcal{P}_t$. Since $\mathcal{P}_t \neq \mathcal{P}_j$ (the target normal distribution differs from the pre-training distribution), these transformations calibrate $\mathbf{z}_a$ toward the target normality manifolds rather than the pre-training support $\mathcal{S}_j$.
Let $\tilde{\mathbf{X}}_a$ denote the $\lambda$-fusion of the original input and the input reconstructed from the transformed patch features $\{\tilde{\mathbf{z}}_a^{(i)}\}_{i=1}^k$, the transformation then shifts the abnormal input away from $\mathcal{N}_{\rho}(\mathcal{S}_j)$ toward the support of $\mathcal{P}_t$. Specifically, for any $(\widetilde{\mathbf{X}}, \widetilde{\mathbf{Y}}) \in \mathcal{S}_j$ that was within $\rho$ distance of $(\mathbf{X}_a, \mathbf{Y}_a)$, the transformed input satisfies:
\begin{equation}
	\|\tilde{\mathbf{X}}_a - \widetilde{\mathbf{X}}\|_2 \geq \|\mathbf{X}_a - \widetilde{\mathbf{X}}\|_2 + \tau,
\end{equation}
where $\tau \!\gg\! \rho$ quantifies the displacement induced by the normality transformation.
Consequently, the transformed abnormal input $\tilde{\mathbf{X}}_a$ now lies in a larger neighborhood $\mathcal{N}_{\rho + \tau}(\mathcal{S}_j)$ relative to its original position. Since the pre-trained model $\mathcal{F}^*$ was trained on $\mathcal{S}_j$ and achieves small error only within $\mathcal{N}_{\rho}(\mathcal{S}_j)$, the prediction error on the transformed input increases. Thus, we have Assumption \ref{asm:inverse_modulus}.
\begin{assumption}
	\label{asm:inverse_modulus}
	There exists a strictly increasing $\omega:[0,\infty)\to[0,\infty)$ with $\omega(0)=0$ such that
	for all $\mathbf{X}_1,\mathbf{X}_2\in\mathcal{N}_{\rho + \tau}(\mathcal{S}_j) \setminus \mathcal{N}_{\rho}(\mathcal{S}_j)$,
	\begin{equation}
		\omega\big(\|\mathbf{X}_1-\mathbf{X}_2\|_2\big) \;\le\; \|\mathcal{F}^*(\mathbf{X}_1)-\mathcal{F}^*(\mathbf{X}_2)\|_2
		\;\le\; L_{\mathcal{F}}\|\mathbf{X}_1-\mathbf{X}_2\|_2 .
	\end{equation}
\end{assumption}
Under Assumption \ref{asm:inverse_modulus}, applying the reverse triangle inequality with $\mathcal{F}^*(\widetilde{\mathbf{X}})$ where $\widetilde{\mathbf{X}}$ has $\rho$ distance from $\mathbf{X}_a$, we have:
\begin{align}
	\|\mathbf{Y}_a-\tilde{\mathbf{Y}}_a\|_2
	&= \|\mathbf{Y}_a-\mathcal{F}^*(\tilde{\mathbf{X}}_a)\|_2 \nonumber\\
	&\ge \|(\mathbf{Y}_a-\mathcal{F}^*(\widetilde{\mathbf{X}}))+(\mathcal{F}^*(\widetilde{\mathbf{X}})-\mathcal{F}^*(\tilde{\mathbf{X}}_a))\|_2 \nonumber\\
	&\ge \|\mathcal{F}^*(\tilde{\mathbf{X}}_a)-\mathcal{F}^*(\widetilde{\mathbf{X}})\|_2-\|\mathbf{Y}_a-\mathcal{F}^*(\widetilde{\mathbf{X}})\|_2.
\end{align}
Since $\widetilde{\mathbf{X}}\in \mathcal{S}_j$ and $\mathcal{F}^*$ perform well in $\mathcal{N}_{\rho}(\mathcal{S}_j)$, we can simplify the above as:
\begin{align}
	\|\mathbf{Y}_a-\tilde{\mathbf{Y}}_a^{\mathrm{ncl}}\|_2
	&\ge \|\mathcal{F}^*(\tilde{\mathbf{X}}_a)-\mathcal{F}^*(\widetilde{\mathbf{X}})\|_2-\|\mathbf{Y}_a-\mathcal{F}^*(\mathbf{X}_a)\|_2 \nonumber\\
	&\ge
	\omega\big(\|\tilde{\mathbf{X}}_a-\widetilde{\mathbf{X}}\|_2\big) - B_j \nonumber\\
	&\ge
	\omega\big(\rho + \tau\big) - B_j,
\end{align}
where $B_j:=\epsilon_j+\delta+(1+L_{\mathcal{F}})\rho$ according to Inequation (\ref{error}). By simultaneously discarding the uncertainty and approximation parts in both $\omega\big(\rho + \tau\big)$ and $B_j$, \ie, $\epsilon_j$ and $\delta$, we can approximately regard $\omega\big(\rho + \tau\big)$ as a linear function of $\rho + \tau$, \ie, $\omega\big(\rho + \tau\big)=\kappa \cdot \big(\rho + \tau\big)$, where $0<\kappa<L_{\mathcal{F}}$. Then we have:
\begin{align}
	\|\mathbf{Y}_a-\tilde{\mathbf{Y}}_a^{\mathrm{ncl}}\|_2 - \|\mathbf{Y}_a-\hat{\mathbf{Y}}_a\|_2
	&\ge \|\mathcal{F}^*(\tilde{\mathbf{X}}_a)-\mathcal{F}^*(\widetilde{\mathbf{X}})\|_2-2\|\mathbf{Y}_a-\mathcal{F}^*(\mathbf{X}_a)\|_2 \nonumber\\
	&\ge
	\omega\big(\rho + \tau\big) - 2B_j \nonumber\\
	&\approx \kappa \cdot \big(\rho + \tau\big) - (2+2L_{\mathcal{F}})\rho.
\end{align}
Since $\rho$ denotes a very small neighborhood radius and $\tau \!\gg\! \rho$, $\kappa \in (0, L_{\mathcal{F}})$ can easily satisfy $\kappa > \frac{2+2L_{\mathcal{F}}}{1+\frac{\tau}{\rho}}\rightarrow 0$, so that:
\begin{equation}
	\|\mathbf{Y}_a-\tilde{\mathbf{Y}}_a\|_2^2 > \|\mathbf{Y}_a-\hat{\mathbf{Y}}_a\|_2^2.
\end{equation}
This demonstrates that the output fusion mechanism directly amplifies the reconstruction or forecasting errors on anomalies while preserving low errors on normal data.
\end{proof}

\subsubsection{Regularization Against Anomaly Overfitting}
\label{regularization}
The contrastive normality constraint $\mathcal{L}_{\mathrm{ConNC}}$ in \cref{eq:loss_nc} also acts as a regularization term that prevents the NCL module from overfitting to abnormal patterns. Firstly, the contrastive loss enforces that the $k$ transformation views $\{\tilde{\mathbf{z}}_n^{(i)}\}_{i=1}^k$ form $k$ well-separated clusters in the normality subspace, with each cluster centered at $\mathbf{c}_i$. Formally, for any patch feature $\mathbf{z}_n$ from normal data:
\begin{equation}
	\|\tilde{\mathbf{z}}_n^{(i)} - \mathbf{c}_i\|_2 + \gamma \leq \min_{j \neq i} \|\tilde{\mathbf{z}}_n^{(i)} - \mathbf{c}_j\|_2.
\end{equation}
This margin-based separation ensures that normal features are tightly clustered around their respective centers, with inter-cluster distances exceeding intra-cluster distances by at least $\gamma$.

Then for anomalous patch features $\mathbf{z}_a$ derived from $(\mathbf{X}_a, \mathbf{Y}_a) \sim \mathcal{P}_a$, the transformation $\tilde{\mathbf{z}}_a^{(i)} = \mathbf{W}_i \mathbf{z}_a + \mathbf{b}_i$ attempts to project $\mathbf{z}_a$ onto the normality manifold. However, if $\mathbf{z}_a$ contains abnormal components orthogonal to all $\{\mathbf{v}_i\}_{i=1}^k$, the projection cannot fully eliminate these components. Let $\mathbf{z}_a = \mathbf{z}_a^{\parallel} + \mathbf{z}_a^{\perp}$, where $\mathbf{z}_a^{\parallel} \in \mathrm{span}(\{\mathbf{v}_i\}_{i=1}^k)$ and $\mathbf{z}_a^{\perp} \in \mathrm{span}(\{\mathbf{v}_i\}_{i=1}^k)^{\perp}$. The transformation only modifies $\mathbf{z}_a^{\parallel}$:
\begin{equation}
	\tilde{\mathbf{z}}_a^{(i)} = \mathbf{W}_i (\mathbf{z}_a^{\parallel} + \mathbf{z}_a^{\perp}) + \mathbf{b}_i = \mathbf{W}_i \mathbf{z}_a^{\parallel} + \mathbf{b}_i + \mathbf{z}_a^{\perp}.
\end{equation}
The residual component $\mathbf{z}_a^{\perp}$ causes $\tilde{\mathbf{z}}_a^{(i)}$ to deviate from the cluster centers:
\begin{equation}
	\|\tilde{\mathbf{z}}_a^{(i)} - \mathbf{c}_i\|_2 \geq \|\mathbf{z}_a^{\perp}\|_2 - \|\mathbf{W}_i \mathbf{z}_a^{\parallel} + \mathbf{b}_i - \mathbf{c}_i\|_2.
\end{equation}
If $\|\mathbf{z}_a^{\perp}\|_2$ is sufficiently large, $\tilde{\mathbf{z}}_a^{(i)}$ will not satisfy the margin constraint, leading to zero loss, which does not result in parameter updates and prevents the NCL from adapting to abnormal patterns.

\section{Theoretical Analysis for Normal Subspace Construction}
\label{app:the_analysis}
To illustrate the effectiveness of normal subspace construction, we conduct theoretical analysis from two aspects: \textbf{(i)} few-shot normal samples are sufficient to construct an informative normality space (\cref{few-shot-svd}); and \textbf{(ii)} top-$k$ truncation effectively captures the essential normal information while suppressing noise (\cref{top-k-truncation}). 

\subsection{Few-Shot SVD for Effective Estimation of the Normality Space}\label{few-shot-svd}
In this section, we demonstrate that the construction of an effective normal subspace does not require a large number of normal samples. Generally, normal patch features naturally lie in a low-dimensional subspace with a nontrivial spectral gap~\citep{yu2026understanding}. Then we utilize standard matrix concentration and the Davis--Kahan theorem~\citep{davis1970rotation} to prove that the leading singular subspace estimated from a finite set of normal patches converges to the population normal subspace, with estimation error
decreasing as $O(N^{-1/2})$. Hence, patch features from a small number of normal samples can already provide a reliable estimate of the normality
space.
\begin{theorem}[Few-Shot Normal Subspace Construction]
	\label{thm:few_shot_subspace}
	Let $\mathbf{z}_1,\ldots,\mathbf{z}_N\in\mathbb{R}^d$ be i.i.d. normal patch features with $\mathbb{E}[\mathbf{z}_i\mathbf{z}_i^\top]=\widehat{\mathbf{M}}$, where
	$\widehat{\mathbf{M}}$ has eigenvalues $\lambda_1\geq\lambda_2\geq\cdots\geq\lambda_d\geq 0$, and nontrivial eigengap $\Delta_k=\lambda_k-\lambda_{k+1}>0$. For sample estimation subspace $\mathbf{V}_k$ and population normal subspace $\widehat{\mathbf{V}}_k$, with probability at least $1-\delta$,
	\begin{equation}
		\left\|
		\sin\Theta(\mathbf{V}_k, \widehat{\mathbf{V}}_k)
		\right\|_2
		=O\left(\frac{1}{\sqrt{N}}\right),
	\end{equation}
	demonstrating that estimation error decreases as $O(1/\sqrt{N})$. Therefore, only a relatively small number of normal samples are sufficient to construct an informative normality space.
\end{theorem}
\begin{proof}
According to \cref{top-k}, the sample estimation subspace is spanned by top-$k$ right singular vectors of normal patch features $\mathbf{Z}$, formalized as:
\begin{equation}
	\mathbf{V}_k
	=\mathrm{span}\{\mathbf{v}_1, \mathbf{v}_2, \ldots, \mathbf{v}_k\}.
\end{equation}
Similarly, we can define the population top-$k$ normal subspace as:
\begin{equation}
	\widehat{\mathbf{V}}_k
	=\mathrm{span}\{\widehat{\mathbf{v}}_1, \widehat{\mathbf{v}}_2, \ldots, \widehat{\mathbf{v}}_k\}.
\end{equation}
Let normal patch features $\mathbf{z}_{1}, \mathbf{z}_{2}, \ldots, \mathbf{z}_{N}\in\mathbb{R}^d$ be i.i.d. drawn from the normal-feature distribution $\mathcal{Z}$, satisfying $\|\mathbf{z}_i\|_2^2\leq l$ almost surely. Then we can define the population and empirical second-moment matrices as:
\begin{equation}
\widehat{\mathbf{M}}=\mathbb{E}[\mathbf{z} \mathbf{z}^{\top}],
\quad
\mathbf{M}
=\frac{1}{N}\mathbf{Z}^{\top}\mathbf{Z}=\frac{1}{N} \sum_{i=1}^N \mathbf{z}_i \mathbf{z}_i^\top.
\end{equation}
Since $\mathbf{Z} = \mathbf{U}_r \boldsymbol{\Sigma}_r \mathbf{V}_r^\top$,
where $\boldsymbol{\Sigma}_r = \mathrm{diag}(\sigma_1, \sigma_2, \ldots, \sigma_r)$ with $\sigma_1 \geq \sigma_2 \geq \cdots \geq \sigma_r \geq 0$,
we have:
\begin{equation}
	\mathbf{M}
	=\mathbf{V}_r\frac{\boldsymbol{\Sigma}_r^2}{N}\mathbf{V}_r^\top.
\end{equation}
The $r$ eigenvectors of $\mathbf{M}$ are exactly the $r$ right singular vectors from $\mathbf{Z}$, and the corresponding eigenvalues are $\lambda_{j}=\frac{\sigma_j^{2}}{N}$, where $j \in \{1, 2, \ldots, r\}$.
Then according to Davis--Kahan theorem~\citep{davis1970rotation}, we have:
\begin{equation}
	\label{Davis}
	\left\|
	\sin\Theta(\mathbf{V}_k, \widehat{\mathbf{V}}_k)
	\right\|_2
	\leq
	\frac{2\|\mathbf{M}-\widehat{\mathbf{M}}\|_2}{\Delta_k},
\end{equation}
where $\Theta(\mathbf{V}_k, \widehat{\mathbf{V}}_k)=\mathrm{diag}(\mathrm{arccos} \langle \mathbf{v}_1, \widehat{\mathbf{v}}_1 \rangle, \ldots, \mathrm{arccos} \langle \mathbf{v}_k, \widehat{\mathbf{v}}_k \rangle)$, and $\Delta_k=\lambda_{k}-\lambda_{k+1}>0$ is the eigengap. 
Then for each $i \in \{1, 2, \ldots, N\}$, we can define the random matrix as:
\begin{equation}
	\label{random_m}
	\mathbf{H}_i = \mathbf{z}_i \mathbf{z}_i^\top - \widehat{\mathbf{M}}.
\end{equation}
By construction, $\mathbb{E}[\mathbf{H}_i] = \mathbf{0}$. The matrices $\{\mathbf{H}_i\}_{i=1}^{N}$ are independent and zero-mean.

Since $\mathbf{z}_i \mathbf{z}_i^\top$ is a rank-one positive semidefinite matrix with $\|\mathbf{z}_i \mathbf{z}_i^\top\|_2 = \|\mathbf{z}_i\|_2^2 \leq l$, and $\widehat{\mathbf{M}} \preceq l \cdot \mathbf{I}$ (because $\widehat{\mathbf{M}} = \mathbb{E}[\mathbf{z} \mathbf{z}^\top]$ and $\|\mathbf{z}\|_2^2 \leq l$), we have:
\begin{equation}
	\|\mathbf{H}_i\|_2 = \left\| \mathbf{z}_i \mathbf{z}_i^\top - \widehat{\mathbf{M}} \right\|_2 \leq \max\{\|\mathbf{z}_i \mathbf{z}_i^\top\|_2, \|\widehat{\mathbf{M}}\|_2\} \leq l.
\end{equation}
Then we compute:
\begin{equation}
	\mathbb{E}[\mathbf{H}_i^2] = \mathbb{E}[ (\mathbf{z}_i \mathbf{z}_i^\top - \widehat{\mathbf{M}} ] = \mathbb{E}[ (\mathbf{z}_i \mathbf{z}_i^\top)^{2} ] - \widehat{\mathbf{M}}^2 \preceq \mathbb{E}[ (\mathbf{z}_i \mathbf{z}_i^\top)^2].
\end{equation}
Since
\begin{equation}
	(\mathbf{z}_i \mathbf{z}_i^\top)^2 = \mathbf{z}_i \mathbf{z}_i^\top \mathbf{z}_i \mathbf{z}_i^\top = \|\mathbf{z}_i\|_2^2 \, \mathbf{z}_i \mathbf{z}_i^\top \preceq l \cdot \mathbf{z}_i \mathbf{z}_i^\top,
\end{equation}
taking expectations, we have:
\begin{equation}
	\mathbb{E}[\mathbf{H}_i^2] \preceq \mathbb{E}[(\mathbf{z}_i \mathbf{z}_i^\top)^2] \preceq l \cdot \mathbb{E}[\mathbf{z}_i \mathbf{z}_i^\top] = l \cdot \widehat{\mathbf{M}}.
\end{equation}
Therefore, we can obtain:
\begin{equation}
	\left\| \sum_{i=1}^N \mathbb{E}[\mathbf{H}_i^2] \right\|_2 \leq N \cdot l \cdot \|\widehat{\mathbf{M}}\|_2 = N \cdot l \cdot \widehat{\lambda}_1,
\end{equation}
Where $\widehat{\lambda}_1=\|\widehat{\mathbf{M}}\|_2$ is the largest eigenvalue of the population second-moment matrix. Then by the matrix Bernstein inequality~\citep{tropp2015introduction}, for any $t\geq0$, we have:
\begin{equation}
	\mathbb{P}\left( \left\| \sum_{i=1}^N \mathbf{H}_i \right\|_2 \geq t \right) \leq 2d \cdot \mathrm{exp}\left( \frac{-t^2/2}{N l \widehat{\lambda}_1 + lt/3}\right),
\end{equation}
where $d$ is the feature dimension. Therefore, for any $\delta \in (0, 1)$, with probability at least $1 - \delta$, we have:
\begin{equation}
	\left\| \sum_{i=1}^N \mathbf{H}_i \right\|_2 
	\leq \sqrt{2 N l \widehat{\lambda}_1 \ln\frac{2d}{\delta}} + \frac{2l}{3} \ln\frac{2d}{\delta}.
\end{equation}
Dividing both sides by $N$, we obtain:
\begin{equation}
	\left\| \frac{1}{N} \sum_{i=1}^N \mathbf{H}_i \right\|_2 
	\leq \sqrt{\frac{2 l \widehat{\lambda}_1}{N} \log\frac{2d}{\delta}} + \frac{2l}{3N} \log\frac{2d}{\delta}.
\end{equation}
According to \cref{random_m}, since
\begin{equation}
	\frac{1}{N} \sum_{i=1}^N \mathbf{H}_i 
	= \frac{1}{N} \sum_{i=1}^N (\mathbf{z}_i \mathbf{z}_i^\top - \widehat{\mathbf{M}})
	= \frac{1}{N} \mathbf{Z}^\top \mathbf{Z} - \widehat{\mathbf{M}}
	= \mathbf{M} - \widehat{\mathbf{M}},
\end{equation}
we can establish the concentration inequality:
\begin{equation}
	\|\mathbf{M} - \widehat{\mathbf{M}}\|_2 
	\leq \sqrt{\frac{2 l \widehat{\lambda}_1}{N} \log\frac{2d}{\delta}} + \frac{2l}{3N} \log\frac{2d}{\delta}.
\end{equation}
Combining this concentration bound with Inequality (\ref{Davis}), with probability at least $1 - \delta$, we have:
\begin{equation}
	\left\|
	\sin\Theta(\mathbf{V}_k, \widehat{\mathbf{V}}_k)
	\right\|_2
	\leq
	\frac{2}{\Delta_k} \left[ \sqrt{\frac{2 l \widehat{\lambda}_1}{N} \log\frac{2d}{\delta}} + \frac{2l}{3N} \log\frac{2d}{\delta} \right].
\end{equation}
Thus, regarding $d, l, \delta, \widehat{\lambda}_1, \Delta_k$ as constants, asymptotically, we have:
\begin{equation}
	\left\|
	\sin\Theta(\mathbf{V}_k, \widehat{\mathbf{V}}_k)
	\right\|_2
	=O\left(\frac{1}{\sqrt{N}}\right).
\end{equation}
The above analyses show that the error diminishes at a rate of $O(1/\sqrt{N})$ with a certain probability. For instance, when $N$ is $5 \times p$ with only five normal samples, where $p$ is the number of patches in a single sample, the error can be reduced to 0.1118 when $p=16$; and it would be 0.0791 when $p=32$.

Therefore, only a relatively small number of normal samples are sufficient to construct an informative normality space. The proof is completed.
\end{proof}
\subsection{Effectiveness of the Top-$k$ Truncation}\label{top-k-truncation}

Retaining only the top-$k$ singular vectors from SVD to preserve dominant directions is a common practice in many tasks~\citep{yan2022addressing,stoica2025model,li2026model}. This approach is intuitive and reasonable, since it removes many small singular value components that are not helpful for the current task, further highlighting the main patterns and eliminating potential interference.

However, there is still a lack of in-depth theoretical analysis to explore why the top-$k$ truncation is effective. Focusing on our $k$-dimensional normal subspace construction, we conduct a detailed theoretical analysis from the following three aspects: \textbf{(i)} top-$k$ truncated approximation is more effective than arbitrary $k$ approximation; \textbf{(ii)} top-$k$ approximation effectively retains the main information; \textbf{(iii)} top-$k$ approximation effectively reduces noise interference. Then we derive Theorems \ref{theorem1}, \ref{theorem2}, and \ref{theorem3} (detailed proofs in \cref{app:theorem1}, \ref{app:theorem2}, and \ref{app:theorem3}) to explore the above three aspects.
\begin{theorem}
	\label{theorem1}
	Consider a well-trained decoder with weights $\bm{w}$ that can effectively recover patch feature matrix $\mathbf{Z}$ to the original series. Compared to the approximation of arbitrary $k$ singular vectors $\bar{\mathbf{Z}} = \sum_{j=i_1}^{i_k} \sigma_j \mathbf{u}_j \mathbf{v}_j^{\top}$, the top-k truncated approximation $\hat{\mathbf{Z}} = \sum_{i=1}^k \sigma_i \mathbf{u}_i \mathbf{v}_i^{\top}$ can achieve lower output error:
	\begin{equation}
		\mathbb{E}_{\mathbf{Z}\sim\mathcal{Z}} \left[ \|\mathbf{Z} \bm{w} - \hat{\mathbf{Z}} \bm{w} \|_2^2 \right] \leq \mathbb{E}_{\mathbf{Z}\sim\mathcal{Z}} \left[ \|\mathbf{Z} \bm{w} - \bar{\mathbf{Z}} \bm{w} \|_2^2 \right].
	\end{equation}
\end{theorem}
\cref{theorem1} indicates that the top-$k$ truncated approximation can retain more effective information than arbitrary $k$ approximation, thus achieving lower output error. However, it remains unclear to what extent this information is being retained. Then we derive \cref{theorem2} to illustrate that for each patch feature, there is an upper bound to the difference between the top-$k$ approximation information and the original information, thus effectively retaining the dominant patterns.
\begin{theorem}
	\label{theorem2}
	$ \forall \mathbf{z}_{n} \in \mathbf{Z}$, $\exists \mathbf{z}_{n}'$ from the $k$-dimensional normal subspace, such that $\| \mathbf{z}_{n}-\mathbf{z}_{n}' \|_2 \leq \sigma_{k+1}$.
\end{theorem}
Theorems \ref{theorem1} and \ref{theorem2} mainly prove the validity of the $k$-dimensional normal subspace spanned by the top-$k$ right singular vectors from different perspectives of dominant information retention. Then we would like to investigate whether our top-$k$ approximation can effectively eliminate the noise interference if there are noise disturbances in the original patch feature matrix $\mathbf{Z}$. For instance, the noise interference could cause once there is inherent noise in the original input or irrelevant information introduced during the encoding process. Finally, we derive \cref{theorem3} to demonstrate the effectiveness of using top-$k$ approximation to reduce noise.
\begin{theorem}
	\label{theorem3}
	Suppose that $\mathbf{Z}$ is composed of a clean data matrix $\mathbf{E}$ and a white noise matrix $\mathbf{D}$, \ie, $\mathbf{Z}=\mathbf{E}+\mathbf{D}$, and the clean data is uncorrelated with the noise. To estimate $\mathbf{E}$, the best rank $k$ matrix in the least-squares sense is $\hat{\mathbf{E}}_\mathrm{LS}=\sum_{i=1}^k \sigma_i \mathbf{u}_i \mathbf{v}_i^{\top}$, which minimizes the following squared error:
	\begin{equation}
		\mathop{{\min}}_{\hat{\mathbf{E}}}\|\hat{\mathbf{E}}-\mathbf{E}\|_{F}^{2}.
	\end{equation}
\end{theorem}
\subsubsection{Proof of \cref{theorem1}}
\label{app:theorem1}
\begin{proof}
The singular value decomposition of $\mathbf{Z} \in \mathbb{R}^{N \times d}$ is $\mathbf{Z} = \mathbf{U}_r \boldsymbol{\Sigma}_r \mathbf{V}_r^\top = \sum_{i=1}^r \sigma_i \mathbf{u}_i \mathbf{v}_i^{\top}$, where $r = \text{rank}(\mathbf{Z})$. We denote the sorted index subset obtained by random selection as $\mathcal{I}=\{i_1, i_2, \ldots, i_k\}$, where $i_1<i_2<\cdots<i_k$, and the remaining sorted index subset is $\hat{\mathcal{I}}=\{i_{k+1}, i_{k+2}, \ldots, i_{r}\}$. Then we can calculate the error matrix resulting from the arbitrary approximation:
\begin{equation}
	\mathbf{Z} - \bar{\mathbf{Z}} = \sum_{j=i_{k+1}}^{i_{r}} \sigma_j \mathbf{u}_j \mathbf{v}_j^{\top}.
\end{equation}
Similarly, the error matrix resulting from the top-k truncated approximation is:
\begin{equation}
	\mathbf{Z} - \hat{\mathbf{Z}} = \sum_{i=k+1}^{r} \sigma_i \mathbf{u}_i \mathbf{v}_i^{\top}.
\end{equation}
Consider the well-trained decoder with weights $\bm{w}=\{\bm{w}_1, \bm{w}_2, \ldots, \bm{w}_s\} \in \mathbb{R}^{d \times s}$, where $s$ is the patch size. It is worth noting that $\mathbf{Z} \bm{w}$ denotes the best restoration of the original series, thus, a lower error compared to $\mathbf{Z} \bm{w}$ implies that more effective information of the matrix $\mathbf{Z}$ has been approximately retained. For $\bm{w}_l \in \bm{w}$, the corresponding point errors for arbitrary approximation and top-k truncated approximation are as follows:
\begin{equation}
	\begin{aligned}
	(\mathbf{Z} - \bar{\mathbf{Z}})\bm{w}_l = \left(\sum_{j=i_{k+1}}^{i_{r}} \sigma_j \mathbf{u}_j \mathbf{v}_j^{\top}\right)\bm{w}_l = \sum_{j=i_{k+1}}^{i_{r}} \sigma_j \mathbf{u}_j (\mathbf{v}_j^{\top}\bm{w}_l), \\
	(\mathbf{Z} - \hat{\mathbf{Z}})\bm{w}_l = \left(\sum_{i=k+1}^{r} \sigma_i \mathbf{u}_i \mathbf{v}_i^{\top}\right)\bm{w}_l = \sum_{i=k+1}^{r} \sigma_i \mathbf{u}_i (\mathbf{v}_i^{\top}\bm{w}_l).
	\end{aligned}
\end{equation}
Considering the above point-by-point error calculation where $\mathbf{v}_j^{\top}\bm{w}_l$ is a scalar, we can sum up the above to form the linear combination of orthogonal normalized vectors $\mathbf{u}_i$. Then we have:
\begin{equation}
	\label{scalar}
	\begin{aligned}
		\|(\mathbf{Z} - \bar{\mathbf{Z}})\bm{w}\|_2^2 = \left\| \sum_{l=1}^{s} \sum_{j=i_{k+1}}^{i_{r}} (\sigma_j \mathbf{v}_j^{\top}\bm{w}_l) \mathbf{u}_j \right\|_2^2 = \sum_{l=1}^{s} \sum_{j=i_{k+1}}^{i_{r}} (\sigma_j \mathbf{v}_j^{\top}\bm{w}_l)^2, \\
		\|(\mathbf{Z} - \hat{\mathbf{Z}})\bm{w}\|_2^2 = \left\| \sum_{l=1}^{s} \sum_{i=k+1}^{r} (\sigma_i \mathbf{v}_i^{\top}\bm{w}_l) \mathbf{u}_i \right\|_2^2 = \sum_{l=1}^{s} \sum_{i=k+1}^{r} (\sigma_i \mathbf{v}_i^{\top}\bm{w}_l)^2.
	\end{aligned}
\end{equation}
Since $\bm{w}$ is assumed to be well-trained, the contribution strength of the right singular vector $\mathbf{v}_i$ to the decoding output monotonically increases with the corresponding singular value $\sigma_i$. Thus we have: $c_i \triangleq \sum_{l=1}^{s} (\mathbf{v}_i^{\top}\bm{w}_l)^2$, where $c_1 \geq c_2 \geq \cdots \geq c_r \geq 0$. Then we can rewrite \cref{scalar} as follows:
\begin{equation}
	\label{simple}
	\begin{aligned}
		\|(\mathbf{Z} - \bar{\mathbf{Z}})\bm{w}\|_2^2 = \sum_{j=i_{k+1}}^{i_{r}} \sigma_j^2 c_j, \quad \|(\mathbf{Z} - \hat{\mathbf{Z}})\bm{w}\|_2^2 = \sum_{i=k+1}^{r} \sigma_i^2 c_i.
	\end{aligned}
\end{equation}
Noting $\sigma_1 \geq \sigma_2 \geq \cdots \geq \sigma_r$, and $\hat{\mathcal{I}} = \{i_{k+1}, \ldots, i_r\}$ represents indices from the arbitrary selection while $\{\sigma_{k+1}, \ldots, \sigma_r\}$ represents the smallest singular values, thus, we have: $\sigma_{i_r} \geq \sigma_r, \cdots, \sigma_{i_{k+1}} \geq \sigma_{k+1}$ and $c_{i_r} \geq c_r, \cdots, c_{i_{k+1}} \geq c_{k+1}$, with the fact $i_r \leq r, \cdots, i_{k+1} \leq k+1$. Then for each item in \cref{simple}, we know $\sigma_{i_m}^2 c_{i_m} \geq \sigma_m^2 c_m$ since $i_m \leq m$, and obtain:
\begin{equation}
	\begin{aligned}
		\sum_{j=i_{k+1}}^{i_{r}} \sigma_j^2 c_j \geq \sum_{i=k+1}^{r} \sigma_i^2 c_i.
	\end{aligned}
\end{equation}
By taking the expectation over the feature distribution $\mathcal{Z}$, we can obtain the following inequality:
\begin{equation}
	\mathbb{E}_{\mathbf{Z}\sim\mathcal{Z}} \left[ \|\mathbf{Z} \bm{w} - \hat{\mathbf{Z}} \bm{w} \|_2^2 \right] \leq \mathbb{E}_{\mathbf{Z}\sim\mathcal{Z}} \left[ \|\mathbf{Z} \bm{w} - \bar{\mathbf{Z}} \bm{w} \|_2^2 \right].
\end{equation}
\end{proof}

\subsubsection{Proof of \cref{theorem2}}
\label{app:theorem2}
\begin{proof}
Since $\mathbf{Z}$ is a $N \times d$ matrix, \ie, $\mathbf{Z}=\left[\mathbf{z}_{1}, \mathbf{z}_{2}, \ldots, \mathbf{z}_{N}\right]^{\top}$, where $\mathbf{z}_{n} \in \mathbb{R}^d$ is the patch feature. After performing SVD on $\mathbf{Z}$, \ie, $\mathbf{Z} = \mathbf{U}_r \boldsymbol{\Sigma}_r \mathbf{V}_r^\top$, we have:
\begin{gather}
	\left[
	\begin{matrix}
		\mathbf{z}_{1}^{\top}   \\
		\mathbf{z}_{2}^{\top}   \\
		\vdots  \\
		\mathbf{z}_{N}^{\top} \\
	\end{matrix} 
	\right] = 
	\left[
	\begin{matrix}
		\Vec{\mathbf{u}}_1   \\
		\Vec{\mathbf{u}}_2   \\
		\vdots \\
		\Vec{\mathbf{u}}_N \\
	\end{matrix} 
	\right] \cdot
	\left[
	\begin{matrix}
		\sigma_1 & 0 & \cdots & 0   \\
		0 & \sigma_2  & \cdots & 0   \\
		\vdots & \vdots & \ddots & \vdots \\
		0 & 0 & \cdots & \sigma_r \\
	\end{matrix} 
	\right]  \cdot
	\left[
	\begin{matrix}
		\mathbf{v}_1^{\top}   \\
		\mathbf{v}_2^{\top}   \\
		\vdots  \\
		\mathbf{v}_r^{\top} \\
	\end{matrix} 
	\right],
	\label{eq:disc-joint2}
\end{gather}
where $\mathbf{U}_r = [\Vec{\mathbf{u}}_1^{\top}, \Vec{\mathbf{u}}_2^{\top}, \ldots, \Vec{\mathbf{u}}_N^{\top}]^{\top} \in \mathbb{R}^{N \times r}$, $\mathbf{V}_r = [\mathbf{v}_1, \mathbf{v}_2, \ldots, \mathbf{v}_r] \in \mathbb{R}^{d \times r}$. Then we expand each item $u_{ij}$ in $\mathbf{U}_r$, and obtain:
\begin{gather}
	\left[
	\begin{matrix}
		\mathbf{z}_{1}^{\top}   \\
		\mathbf{z}_{2}^{\top}   \\
		\vdots  \\
		\mathbf{z}_{N}^{\top} \\
	\end{matrix} 
	\right] = 
	\left[
	\begin{matrix}
		\sigma_1 \cdot u_{11} & \sigma_2 \cdot u_{12} & \cdots & \sigma_r \cdot u_{1r}   \\
		\sigma_1 \cdot u_{21} & \sigma_2 \cdot u_{22} & \cdots & \sigma_r \cdot u_{2r}   \\
		\vdots & \vdots & \ddots & \vdots \\
		\sigma_1 \cdot u_{N1} & \sigma_2 \cdot u_{N2} & \cdots & \sigma_r \cdot u_{Nr}   \\
	\end{matrix} 
	\right]  \cdot
	\left[
	\begin{matrix}
		\mathbf{v}_1^{\top}   \\
		\mathbf{v}_2^{\top}   \\
		\vdots  \\
		\mathbf{v}_r^{\top} \\
	\end{matrix} 
	\right].
	\label{eq:disc-joint2}
\end{gather}
For each patch feature $\mathbf{z}_{n} \in \mathbf{Z}$, we have:
\begin{gather}
	\mathbf{z}_{n}^{\top} = 
	\left[
	\begin{matrix}
		\sigma_1 \cdot u_{n1} & \sigma_2 \cdot u_{n2} & \cdots & \sigma_r \cdot u_{nr}
	\end{matrix} 
	\right]  \cdot
	\left[
	\begin{matrix}
		\mathbf{v}_1^{\top}   \\
		\mathbf{v}_2^{\top}   \\
		\vdots  \\
		\mathbf{v}_r^{\top} \\
	\end{matrix} 
	\right].
	\label{eq:disc-joint2}
\end{gather}
Then we can get the summation form:
\begin{equation}
	\begin{aligned}
		\mathbf{z}_{n}^{\top} = \sum_{i=1}^{r} \sigma_i \cdot u_{ni} \mathbf{v}_i^{\top} = \sum_{i=1}^{k} \sigma_i \cdot u_{ni} \mathbf{v}_i^{\top} + \sum_{i=k+1}^{r} \sigma_i \cdot u_{ni} \mathbf{v}_i^{\top}.
	\end{aligned}
\end{equation}
Note that we can construct feature $\mathbf{z}_{n}'$ from the $k$-dimensional normal subspace, \ie, $\mathbf{z}_{n}'^{\top} = \sum_{i=1}^{k} \sigma_i \cdot u_{ni} \mathbf{v}_i^{\top}$. Then we have:
\begin{equation}
	\begin{aligned}
		\left\|\mathbf{z}_{n}^{\top} - \mathbf{z}_{n}'^{\top} \right\|_2 =  \left\|\sum_{i=k+1}^{r} \sigma_i \cdot u_{ni} \mathbf{v}_i^{\top} \right\|_2 = \sqrt{\sum_{i=k+1}^{r} (\sigma_i \cdot u_{ni})^{2}}.
	\end{aligned}
\end{equation}
Since all the vectors in $\mathbf{U}_r$ are mutually orthogonal and normalized, we know $\left\|\Vec{\mathbf{u}}_n\right\|=1$, and $\left\|\Vec{\mathbf{u}}_n^{[k+1,r]}\right\|\leq1$, where $\Vec{\mathbf{u}}_n^{[k+1,r]}$ is the sub-vectors of $\Vec{\mathbf{u}}_n$ from $k+1$-th to $r$-th dimensions. Noting that the singular values are decreasing, \ie, $\sigma_1 \geq \sigma_2 \geq \cdots \geq \sigma_r$, then we have the following upper bound:
\begin{equation}
	\begin{aligned}
		\left\|\mathbf{z}_{n}^{\top} - \mathbf{z}_{n}'^{\top} \right\|_2 \leq \sigma_{k+1} \sqrt{\sum_{i=k+1}^{r} u_{ni}^{2}} \leq \sigma_{k+1}.
	\end{aligned}
\end{equation}
\end{proof}

\subsubsection{Proof of \cref{theorem3}}
\label{app:theorem3}
\begin{proof}
Since clean data is assumed to be uncorrelated with the white noise, we know that the clean data matrix $\mathbf{E}$ is orthogonal to the noise matrix $\mathbf{D}$, \ie, $\mathbf{E}^{\top}\mathbf{D}=0$. We also have $\mathbf{D}^{\top}\mathbf{D}=\delta^2\mathbf{I}$ since noise is white and uncorrelated, where $\delta^2$ is the noise variance. Our goal is to minimize $\|\hat{\mathbf{E}}-\mathbf{E}\|_{F}^{2}$. Because of $\mathbf{Z}=\mathbf{E}+\mathbf{D}$, we have:
\begin{equation}
	\begin{aligned}
		\|\mathbf{Z}-\hat{\mathbf{E}}\|_{F}^{2} = \|\mathbf{E}-\hat{\mathbf{E}}+\mathbf{D}\|_{F}^{2} = \|\mathbf{E}-\hat{\mathbf{E}}\|_{F}^{2} + \|\mathbf{D}\|_{F}^{2} + 2 \mathrm{tr}[(\mathbf{E}-\hat{\mathbf{E}})^{\top}\mathbf{D}].
	\end{aligned}
\end{equation}
Since $\mathbf{E}^{\top}\mathbf{D}=0$ and the estimated signal is uncorrelated with noise in the expectation sense, thus $\mathrm{tr}[(\mathbf{E}-\hat{\mathbf{E}})^{\top}\mathbf{D}]=0$, and we have $\|\mathbf{D}\|_{F}^{2}=\mathrm{tr}(\mathbf{D}^{\top}\mathbf{D})=\mathrm{tr}(\delta^2\mathbf{I})=\delta^2\cdot d$, then we can obtain:
\begin{equation}
	\begin{aligned}
		\|\mathbf{Z}-\hat{\mathbf{E}}\|_{F}^{2} = \|\mathbf{E}-\hat{\mathbf{E}}\|_{F}^{2} + \delta^2\cdot d,
	\end{aligned}
\end{equation}
where $\delta^2\cdot d$ is a constant term. Therefore, the minimization of $\|\hat{\mathbf{E}}-\mathbf{E}\|_{F}^{2}$ is equivalent to the minimization of $\|\mathbf{Z}-\hat{\mathbf{E}}\|_{F}^{2}$. Suppose that clean data matrix $\mathbf{E}$ is rank deficient, \ie, $\text{rank}(\mathbf{E})=k$, then we also restrict $\hat{\mathbf{E}}$ to the matrix with a rank not exceeding $k$. Therefore, minimizing $\|\hat{\mathbf{E}}-\mathbf{E}\|_{F}^{2}$ becomes:
\begin{equation}
	\mathop{{\min}}_{\text{rank}(\hat{\mathbf{E}})\leq k}\|\mathbf{Z}-\hat{\mathbf{E}}\|_{F}^{2}.
\end{equation}
According to the Eckart–Young–Mirsky theorem~\citep{Low_rank_approx}, for any matrix $\mathbf{Z}$, the optimal approximation with a rank no greater than $k$ is the truncation of top-$k$ singular values of $\mathbf{Z}$. Therefore, the best least-squares estimate of the clean data matrix $\mathbf{E}$ is the top-k truncated approximation of $\mathbf{Z}$:
\begin{equation}
	\hat{\mathbf{E}}_\mathrm{LS}=\hat{\mathbf{Z}}=\sum_{i=1}^k \sigma_i \mathbf{u}_i \mathbf{v}_i^{\top},
\end{equation}
where $\sigma_i$, $\mathbf{u}_i$, and $\mathbf{v}_i$ denote the singular values, left singular vectors, and right singular vectors of $\mathbf{Z}$ respectively. We complete the proof.
\end{proof}

\section{Proof of \cref{prop1} (Spectral and Geometric Property)}
\label{app:prop1}
\begin{proof}
For any $\mathbf{k} \in \mathbf{v}_i^\perp$ (\ie, $\mathbf{v}_i^\top\mathbf{k}=0$), we have:
\begin{equation}
	\begin{aligned}
		\mathbf{W}_i\mathbf{k} = (\mathbf{I} - \beta_i \mathbf{v}_i \mathbf{v}_i^\top)\mathbf{k} = \mathbf{k} - \beta_i\mathbf{v}_i(\mathbf{v}_i^\top\mathbf{k}) = \mathbf{k}.
	\end{aligned}
\end{equation}
Thus each vector in $\mathbf{v}_i^\perp$ is an eigenvector with eigenvalue $1$. Since $\mathbf{v}_i \in \mathbb{R}^d$, we can select the remaining $d-1$ directions that are orthogonal to $\mathbf{v}_i$ and combine them to form $\mathbf{v}_i^\perp$. Then we can know that $\mathbf{W}_i$ has $(d-1)$ eigenvalues of $1$ since $\dim(\mathbf{v}_i^\perp) = d - 1$. 

Similarly, for the vector $\mathbf{v}_i$ itself, we have:
\begin{equation}
	\begin{aligned}
		\mathbf{W}_i\mathbf{v}_i = (\mathbf{I} - \beta_i \mathbf{v}_i \mathbf{v}_i^\top)\mathbf{v}_i = \mathbf{v}_i - \beta_i\mathbf{v}_i(\mathbf{v}_i^\top\mathbf{v}_i) = (1-\beta_i)\mathbf{v}_i.
	\end{aligned}
\end{equation}
Therefore, $\mathbf{W}_i$ has an eigenvector $\mathbf{v}_i$ with eigenvalue $1-\beta_i$. Overall, The spectrum of $\mathbf{W}_i$ is $\Lambda(\mathbf{W}_i) = \{1 - \beta_i, 1, \ldots, 1\}$. This indicates that since $\mathbf{W}_i$ is a rank-1 perturbation of the identity, it leaves all components orthogonal to $\mathbf{v}_i$ unchanged and only rescales the projection along $\mathbf{v}_i$ by a factor of $(1-\beta_i)$.

To more intuitively present how $\beta_i$ governs the geometric behavior of the transformation $\tilde{\mathbf{z}}_{n}^{(i)}=\mathbf{W}_i \mathbf{z}_{n} + \mathbf{b}_i$, we compute the corresponding projection of $\tilde{\mathbf{z}}_{n}^{(i)}$ onto $\mathbf{v}_i$:
\begin{align}
		\mathbf{v}_i^\top\tilde{\mathbf{z}}_{n}^{(i)}
		&=\mathbf{v}_i^\top(\mathbf{I} - \beta_i \mathbf{v}_i \mathbf{v}_i^\top)\mathbf{z}_{n} + \mathbf{v}_i^\top\alpha_i \beta_i \mathbf{v}_i \nonumber\\
		&= (1-\beta_i)\mathbf{v}_i^\top\mathbf{z}_{n} + \alpha_i \beta_i \nonumber\\
		&= \alpha_i + (1-\beta_i)(\mathbf{v}_i^\top\mathbf{z}_{n}-\alpha_i).
	\label{eq:project}
\end{align}
\cref{eq:project} reflects explicit geometric regimes: \textbf{(i) identity mapping} ($\beta_i < 1, \beta_i \to 0$): the update tends to retain partial old information, and remains unchanged when $\beta_i=0$. \textbf{(ii) complete replacement} ($\beta_i = 1$): the updated projection is exactly replaced by $\alpha_i$ to reflect normality. \textbf{(iii) over-calibration} ($\beta_i > 1$): the update tends to further reversely calibrate the old information. 

Since each patch feature involves varying degrees of normality, our learnable parameters $\alpha_i$ and $\beta_i$ can adaptively adjust each feature to the normality manifold. This allows the model to learn under-, moderate, and over-corrective behaviors where beneficial.
\end{proof}

\section{Experimental Details}\label{detail_setup}

\subsection{Datasets}\label{app:datasets}
We conduct comprehensive experiments on ten TSAD datasets, including five univariate and five multivariate datasets. \cref{tab:datasets} summarizes the detailed statistics of these datasets. The characteristics of each dataset are introduced below:
\begin{itemize}[leftmargin=*]
	\item \textbf{IOPS}~\citep{ren2019time}: A large-scale dataset containing Key Performance Indicators (KPIs) from Internet companies, with time series collected from real-world IT operations. The dataset includes various anomaly patterns such as spikes, level shifts, and gradual changes.
	\item \textbf{UCR}~\citep{wu2021current}: A collection of univariate time series from diverse domains including medicine, power demand, air temperature, and so on. The anomalies are typically point anomalies or short subsequence anomalies.
	\item \textbf{NAB}~\citep{ahmad2017unsupervised}: The Numenta Anomaly Benchmark consists of labeled real-world and artificial time series with point and contextual anomalies. It includes data from various sources such as online advertisement clicks, machine temperature sensors, and traffic data.
	\item \textbf{MSL}~\citep{benecki2021detecting}: Mars Science Laboratory dataset containing telemetry data from spacecraft monitoring systems. It consists of time series with anomalies caused by various spacecraft operational issues. Since only one feature represents a sensor measurement and the remaining features are binary-encoded commands, we transform the original multivariate time series into univariate ones, which is consistent with TSB-UAD~\citep{paparrizos2022tsb} and TAB~\citep{qiu2025tab}.
	\item \textbf{SMAP}~\citep{benecki2021detecting}: Soil Moisture Active Passive satellite dataset containing spacecraft telemetry data. Anomalies represent spacecraft component failures and unusual sensor behaviors. Like the processing of MSL, the multivariate time series are transformed into univariate ones.
	\item \textbf{PSM}~\citep{abdulaal2021practical}: Pooled Server Metrics dataset is collected internally from multiple application server nodes as eBay, containing monitoring server performance metrics such as CPU usage, memory consumption, and network traffic. The dataset includes complex multivariate anomaly patterns.
	\item \textbf{SMD}~\citep{su2019robust}: Server Machine Dataset is collected from a large Internet company's server monitoring system. It contains various types of server anomalies, including hardware failures and software bugs. Anomalies are annotated by domain experts.
	\item \textbf{SWAT}~\citep{mathur2016swat}: Secure Water Treatment dataset is collected from a real-world water treatment testbed. Anomalies are caused by attack scenarios targeting critical infrastructure, making it particularly challenging.
	\item \textbf{PUMP}~\citep{feng2021time}: A dataset collected from a water pumping system with multiple sensors monitoring pressure, flow rate, and motor status. The anomalies are caused by mechanical failures and sensor malfunctions.
	\item \textbf{NYC}~\citep{cui2016comparative}: New York City taxi demand dataset containing traffic volume measurements. Anomalies correspond to unusual traffic patterns caused by special events or incidents.
\end{itemize}
\input{tab/data_statistic.tex}

For all univariate time series, we follow the data filtering strategy proposed by TAB benchmark~\citep{qiu2025tab} to ensure data quality. Finally, we include a total of 361 univariate and multivariate time series.
For all datasets, we follow the standard train-test split provided by the original sources~\citep{qiu2025tab}. Only five normal samples from the training set are utilized to learn the normality constraint, while the test set includes both normal and anomalous observations for evaluation.

\subsection{Baselines}\label{app:baselines}

We evaluate NCL on ten state-of-the-art foundation models, including six forecasting-based, two reconstruction-based, and two forecasting-reconstruction-combined TSFMs:
\begin{itemize}[leftmargin=*]
	\item \textbf{Aurora}~\citep{wu2026aurora}: A multimodal foundation model pre-trained on a cross-domain multimodal time series corpus with generative probabilistic forecasting. It demonstrates strong zero-shot forecasting capabilities on both unimodal and multimodal scenarios.
	\item \textbf{ROSE}~\citep{wang2025towards}: A foundation model pre-trained with both reconstruction and forecasting tasks. For downstream tasks, only the pre-trained prediction head is used to fine-tune the model.
	\item \textbf{Time-MoE}~\citep{shi2025time}: A mixture-of-experts foundation model that uses a sparse gating mechanism to select specialized experts for different temporal patterns. Pre-trained on 300 billion time points with multi-resolution forecasting objectives.
	\item \textbf{Chronos}~\citep{ansari2024chronos}: A probabilistic forecasting model based on language modeling techniques. It tokenizes time series into discrete bins and applies autoregressive forecasting pre-training on massive time series datasets.
	\item \textbf{TimesFM}~\citep{das2024decoder}: A decoder-only foundation model pre-trained on Google's internal time series datasets with point forecasting objectives.
	\item \textbf{TTM}~\citep{ekambaram2024tiny}: A lightweight foundation model with channel-mixing and time-mixing layers. Despite its compact size, it achieves competitive performance through efficient architectural design.
	\item \textbf{MOMENT}~\citep{goswami2024moment}: A masked autoencoder foundation model pre-trained on the Time Series Pile dataset. It uses a transformer encoder to reconstruct randomly masked time series patches, learning robust temporal representations.
	\item \textbf{UniTS}~\citep{gao2024units}: A unified multi-task foundation model pre-trained with masked reconstruction. It demonstrates strong generalization across various downstream tasks.
    \item \textbf{TSPulse}~\citep{ekambaram2026tspulse}: An ultra-lightweight foundation model with disentangled representation learning. It pre-trains with masked reconstruction augmented by explicit disentanglement across temporal, spectral, and semantic spaces, and employs lightweight post-hoc fusers to selectively combine these views for task-specific adaptation.
	\item \textbf{SEMPO}~\citep{he2025sempo}: A lightweight foundation model featuring energy-aware spectral decomposition and mixture-of-prompts enabled Transformer. It models both high-energy and low-energy informative frequency signals for efficient pre-training, and employs dataset-specific prompts with adaptive routing for parameter-efficient cross-domain adaptation.
\end{itemize}

For anomaly detection, all TSFMs are applied in a zero-shot manner without fine-tuning their pre-trained parameters. We use the forecasting error (for forecasting-based TSFMs), reconstruction error (for reconstruction-based TSFMs), or combined error of forecasting and reconstruction (for forecasting-reconstruction-combined TSFMs) as the baseline anomaly score. NCL is then applied as a plug-and-play framework to enhance their anomaly detection capabilities.

Additionally, we compare NCL-augmented TSFMs with several deep TSAD methods, simple one-liner baselines, and TSAD-specific foundation models in \cref{compare_with_classic}, \ref{compare_with_tsad}, and \ref{app:one_to_one}, respectively. 

Conventional deep TSAD methods (\ie, methods trained from scratch on each dataset) include CroAD~\citep{CrossAD}, ConAD~\citep{zhuang2024noise}, TsNet~\citep{wutimesnet}, and Dlinear~\citep{zeng2023transformers}. 

One-liner baselines (\ie, simple baseline methods requiring no training) include LOF (Local Outlier Factor), Last-1 (last value as prediction), Last-3 (average of last 3 values), Var-100 (100-window moving variance), and Centered-5 (average of the 2 neighboring values before and after).

TSAD-specific foundation models include DADA~\citep{shentu2025towards}, TimeRCD~\citep{lan2026towards}, and TimeRadar~\citep{he2026timeradar}. These approaches aim to learn discriminative representations by pre-training on large-scale datasets augmented with synthetic anomalies. DADA employs dual adversarial decoders to explicitly distinguish normal from abnormal patterns using injected anomalies, while TimeRCD conducts relative context discrepancy pre-training on synthetic anomalies. Pre-training on the anomaly-injected time series corpus released by DADA, TimeRadar identifies diverse anomalies in a domain-rotatable time–frequency space.

\subsection{Implementation Details}
\label{app:implementation}

We utilize the TAB \citep{qiu2025tab} code repository for unified evaluation and follow its default parameter configuration. For all ten TSFMs, we rely on code from: 
\begin{itemize}[leftmargin=*]
	\item \textbf{Aurora}: \url{https://github.com/decisionintelligence/Aurora}
	\item \textbf{ROSE}: \url{https://github.com/decisionintelligence/ROSE}
	\item \textbf{Time-MoE}: \url{https://github.com/Time-MoE/Time-MoE}
	\item \textbf{Chronos}: \url{https://github.com/amazon-science/chronos-forecasting}
	\item \textbf{TimesFM}: \url{https://github.com/google-research/timesfm}
	\item \textbf{TTM}: \url{https://github.com/ibm-granite/granite-tsfm/tree/main/tsfm_public/models/tinytimemixer}
	\item \textbf{MOMENT}: \url{https://github.com/moment-timeseries-foundation-model/moment}
	\item \textbf{UniTS}: \url{https://github.com/mims-harvard/UniTS}
    \item \textbf{TSPulse}: \url{https://github.com/ibm-granite/granite-tsfm}
	\item \textbf{SEMPO}: \url{https://github.com/mala-lab/SEMPO}
\end{itemize}
Several deep TSAD methods (CroAD, ConAD, TsNet, and Dlinear) are adopted from \url{https://github.com/decisionintelligence/TAB}.
LOF is adapted from the pyod package \url{https://github.com/yzhao062/pyod}.
Last-1, Last-3, Var-100, and Centered-5 are adopted from \url{https://gitlab.kuleuven.be/m-group-campus-brugge/dtai_public/publications/iclr2026_timeseriesfoundationmodelsad/}.
DADA is from the official code repository \url{https://github.com/iambowen/DADA}.
TimeRCD is from \url{https://github.com/thu-sail-lab/Time-RCD}. TimeRadar is from \url{https://github.com/mala-lab/TimeRadar}.

For NCL hyperparameters, the number of principal directions (top-$k$ singular vectors) is set to $k=16$ for all experiments, and we also conduct a comprehensive sensitivity analysis in Appendix~\ref{app:top_k}. The margin $\gamma$ in the contrastive normality constraint (\cref{eq:loss_nc}) is set to 1.0. The learning rate is 1e-4 with the Adam optimizer. We train NCL for 5 epochs with only 5 normal samples randomly selected from the training set.
During training, we freeze all parameters of the pre-trained TSFM and only optimize the parameters in NCL.
All experiments are conducted on a workstation with an Intel(R) Xeon(R) Gold 6426Y CPU and a single NVIDIA A100 GPU (80GB memory).

\newpage

\section{Algorithms Details}
\label{app:algorithms}

\cref{alg:training} details the complete training procedure of NCL, including normal subspace construction, patch-adaptive normality transformation, contrastive normality constraint, and output fusion. Only the NCL module parameters are updated while the TSFM $\mathcal{F}$ remains frozen. The training procedure emphasizes five key aspects: \textbf{(i)} The TSFM $\mathcal{F}$ is kept frozen throughout training, preserving its pre-trained knowledge; \textbf{(ii)} The SVD identifies the principal directions of normal variation in the feature space; \textbf{(iii)} Transforming each feature onto the normality manifold along each direction; \textbf{(iv)} The contrastive normality constraint encourages that transformed features form well-separated clusters corresponding to distinct normal directions; and \textbf{(v)} The gated fusion mechanism adaptively combines NCL predictions with original TSFM outputs to amplify normal properties.

\begin{algorithm}[H]
	\caption{Training Procedure of NCL}
	\label{alg:training}
	\begin{algorithmic}[1]
		\renewcommand{\REQUIRE}{\item[\textbf{Input:}]}
		\renewcommand{\ENSURE}{\item[\textbf{Output:}]}
		\REQUIRE Frozen TSFM $\mathcal{F}$, normal training samples $\{\mathbf{X}_t, \mathbf{Y}_t\}_{t=1}^{N_{\text{train}}}$, learning rate $\eta$, margin $\gamma$, number of epochs $E$, number of principal directions $k$, trainable parameters $\mathbf{W}_\alpha, \mathbf{W}_\beta, \mathbf{b}_\alpha, \mathbf{b}_\beta$, decoder weights, fusion weights $\boldsymbol{\lambda}$
		\ENSURE Trained NCL parameters, top-$k$ right singular vectors $\{\mathbf{v}_i\}_{i=1}^k$, cluster centers $\{\mathbf{c}_i\}_{i=1}^k$
		\STATE Initialize learnable parameters: $\mathbf{W}_\alpha, \mathbf{W}_\beta \in \mathbb{R}^{k \times d}$, $\mathbf{b}_\alpha, \mathbf{b}_\beta \in \mathbb{R}^k$
		\STATE Initialize decoder head and fusion weights $\boldsymbol{\lambda} \in (0,1)^D$
		\FOR{epoch $e = 1$ to $E$}
		\STATE Select all the samples from $\{\mathbf{X}_t, \mathbf{Y}_t\}_{t=1}^{N_{\text{train}}}$ to form batch $\{(\mathbf{X}_t, \mathbf{Y}_t)\}_{t \in \mathcal{B}}$
		\STATE $\mathbf{Z} = \{\mathbf{z}_n\}_{n=1}^{N} = \mathcal{F}(\mathbf{X}_{\mathcal{B}})$ \COMMENT{Extract patch features from frozen TSFM}
		\STATE $\hat{\mathbf{Y}}_{\mathcal{B}} = \text{Decoder}_{\mathcal{F}}(\mathbf{Z})$ \COMMENT{Obtain TSFM predictions}
		\STATE $\mathbf{Z} = \mathbf{U}_r \boldsymbol{\Sigma}_r \mathbf{V}_r^\top$ with $\sigma_1 \geq \sigma_2 \geq \cdots \geq \sigma_r$ \COMMENT{Perform SVD on feature matrix}
		\STATE $\{\mathbf{v}_i\}_{i=1}^k \gets \text{columns 1 to } k \text{ of } \mathbf{V}_r$ \COMMENT{Retain top-$k$ right singular vectors}
		\FOR{each patch feature $\mathbf{z}_n$ in $\mathbf{Z}$}
		\STATE \quad $\boldsymbol{\alpha}_n = \sigma(\mathbf{W}_\alpha \mathbf{z}_n + \mathbf{b}_\alpha) \in (0,1)^k$
		\STATE \quad $\boldsymbol{\beta}_n = 2 \cdot \sigma(\mathbf{W}_\beta \mathbf{z}_n + \mathbf{b}_\beta) \in (0,2)^k$
		\FOR{each principal direction $i = 1$ to $k$}
		\STATE $\tilde{\mathbf{z}}_{n}^{(i)} = (\mathbf{I} - \beta_i \mathbf{v}_i \mathbf{v}_i^\top) \mathbf{z}_{n} + \alpha_i \beta_i \mathbf{v}_i$ \COMMENT{Normality Transformation, \cref{eq:transform}}
		\ENDFOR
		\STATE $\tilde{\mathbf{z}}_n = \frac{1}{k} \sum_{i=1}^k \tilde{\mathbf{z}}_n^{(i)}$ \COMMENT{Aggregate views}
		\ENDFOR
		\FOR{each direction $i = 1$ to $k$}
		\STATE $\mathbf{c}_i = \frac{1}{N} \sum_{n=1}^{N} \tilde{\mathbf{z}}_n^{(i)}$ 
		\ENDFOR
		\STATE Initialize $\mathcal{L}_{\mathrm{ConNC}} = 0$
		\FOR{each patch $n = 1$ to $N$}
		\FOR{each direction $i = 1$ to $k$}
		\STATE $d_{\text{pos}} = \|\tilde{\mathbf{z}}_n^{(i)} - \mathbf{c}_i\|_2$ \COMMENT{Compute positive pair distance}
		\STATE $d_{\text{neg}} = \min_{j \neq i} \|\tilde{\mathbf{z}}_n^{(i)} - \mathbf{c}_j\|_2$ \COMMENT{Compute negative pair distances}
		\STATE $\mathcal{L}_{\mathrm{ConNC}} \gets \mathcal{L}_{\mathrm{ConNC}} + \max(0, d_{\text{pos}} + \gamma - d_{\text{neg}})$ \COMMENT{Accumulate triplet loss}
		\ENDFOR
		\ENDFOR
		\STATE $\mathcal{L}_{\mathrm{ConNC}} = \mathcal{L}_{\mathrm{ConNC}} / (N \cdot k)$ \COMMENT{\cref{eq:loss_nc}}
		\STATE $\tilde{\mathbf{Y}}_{\mathcal{B}}^{\mathrm{ncl}} = \text{Decoder}_{\mathrm{ncl}}([\tilde{\mathbf{z}}_1, \ldots, \tilde{\mathbf{z}}_{N}])$ \COMMENT{Decode calibrated features}
		\STATE $\tilde{\mathbf{Y}}_{\mathcal{B}} = \boldsymbol{\lambda} \odot \tilde{\mathbf{Y}}_{\mathcal{B}}^{\mathrm{ncl}} + (1 - \boldsymbol{\lambda}) \odot \hat{\mathbf{Y}}_{\mathcal{B}}$ \COMMENT{Fuse with TSFM output, \cref{eq:fusion}}
		\STATE $\mathcal{L}_{\mathrm{Pred}} = \frac{1}{|\mathcal{B}|} \sum_{t \in \mathcal{B}} \|\tilde{\mathbf{Y}}_t - \mathbf{Y}_t\|_2^2$ \COMMENT{Compute prediction loss}
		\STATE $\mathcal{L} = \mathcal{L}_{\mathrm{Pred}} + \mathcal{L}_{\mathrm{ConNC}}$ \COMMENT{Total loss, \cref{eq:total_loss}}
		\STATE $\theta_{\mathrm{NCL}} \gets \theta_{\mathrm{NCL}} - \eta \nabla_{\theta_{\mathrm{NCL}}} \mathcal{L}$ \COMMENT{Update NCL parameters via gradient descent}
		\ENDFOR
		\RETURN Trained NCL parameters, $\{\mathbf{v}_i\}_{i=1}^k$, $\{\mathbf{c}_i\}_{i=1}^k$
	\end{algorithmic}
\end{algorithm}

\newpage

\section{Computational Complexity Analysis}

We further analyze the computational complexity of NCL. Given $N$ patches from few-shot normal training samples, feature dimension $d$, and $k$ principal directions, the normal subspace construction via SVD requires $\mathcal{O}(N^2 \cdot d)$ one-time cost. Then for per-iteration complexity, TSFM feature extraction requires $\mathcal{O}(N \cdot d \cdot C_{\mathcal{F}})$ where $C_{\mathcal{F}}$ is TSFM complexity (frozen, no gradients); normality transformation is $\mathcal{O}(N \cdot k \cdot d)$ for all $k$ views; $\mathcal{O}(N \cdot k^2 \cdot d)$ is needed for distance computations; and lightweight projection requires $\mathcal{O}(k \cdot d)$. Since $k \ll d$ and the TSFM is frozen, we can effectively conduct NCL training in approximate $\mathcal{O}(N \cdot k^2 \cdot d)$ time, which is linear with dimensionality $d$. The comparisons in \cref{eff_ana} and \cref{app:efficiency_comparison} empirically validate the computational efficiency of NCL.

\section{Detailed Efficiency Comparison with Other TSFM Adaptation Methods}\label{app:efficiency_comparison}

\cref{fig:plugin_pra_time_compare_all} provides a detailed efficiency comparison of NCL against other TSFM adaptation methods across different TSFMs on multivariate time series. The results consistently demonstrate NCL's superior efficiency. By freezing the TSFM backbone and only optimizing lightweight projection heads, NCL requires minimal parameter update overhead during training. In contrast, LoRA must store gradients for updating key TSFM parameters, and CoRA needs more parameters to model the cross-channel association. It is worth noting that NCL has slightly more trainable parameters compared to LoRA on reconstruction-based TSFMs (MOMENT and UniTS). This is because the reconstruction is performed on the entire series, and the final linear head of NCL thus introduces more parameters. 
Beyond parameter overhead, NCL also exhibits the fastest training times across nearly all TSFMs. For MOMENT, NCL completes training in approximately 2.59 seconds, while LoRA requires 2.71 seconds and CoRA needs 3.32 seconds. Similar patterns are observed for other TSFMs including Chronos, ROSE, and SEMPO. The efficiency advantage makes NCL particularly attractive for practical deployment scenarios where frequent model updates are required.

\begin{figure*}[ht]
	\centering
	\includegraphics[width=1.0\linewidth]{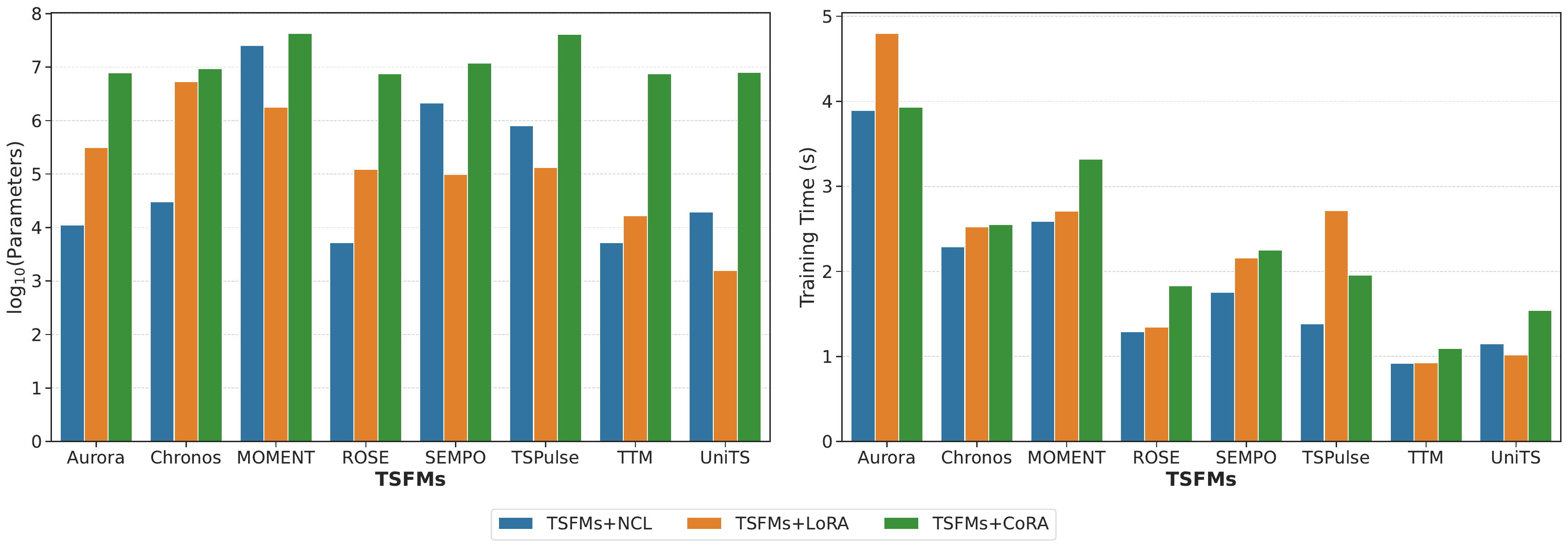}
	\caption{Efficiency comparison of NCL and other TSFM adaptation methods w.r.t. different TSFMs.}
	\label{fig:plugin_pra_time_compare_all}
\end{figure*}

\section{Comparison with Deep TSAD and One-Liner Methods}\label{compare_with_classic}
\input{tab/all_PR_compare.tex}
To further demonstrate the effectiveness of NCL, we compare it against advanced deep TSAD methods (trained from scratch on each dataset) and simple one-liner baselines (requiring no training). \cref{tab:PR} presents comprehensive precision-recall-based metrics including the Area Under the Precision-Recall Curve (A-P), the Range Area Under the Precision-Recall Curve (R-A-P), and the Volume Under the Surface of Precision-Recall (VUS-PR, abbreviated as V-PR) across all ten datasets. \cref{tab:ROC} reports comprehensive receiver-operating-characteristic-based metrics including the Area under the Receiver Operating Characteristics Curve (A-R),  the Range Area under the Receiver Operating Characteristics Curve (R-A-R), and the Volume Under the Surface of Receiver Operating Characteristic (VUS-ROC, abbreviated as V-ROC) across all ten datasets.

\textbf{Comparison with Deep TSAD Methods.} NCL consistently outperforms sophisticated deep TSAD methods such as CroAD, ConAD, TsNet, and Dlinear. Specifically, UniTS+NCL achieves the best average rank of \textbf{2.13} across all PR metrics and datasets, significantly better than CroAD (5.13) and TsNet (6.27). The key advantage of NCL over deep TSAD methods is its generalization capability. Conventional deep TSAD methods must be trained from scratch for each new dataset, while NCL requires only a few normal samples to adapt pre-trained TSFMs. This makes NCL significantly more practical for real-world scenarios where training data is scarce or expensive to obtain.

\textbf{Comparison with One-Liner Baselines.} Interestingly, simple moving-window variance baselines (Var-100) achieve competitive performance with an average rank of 4.00, even outperforming sophisticated deep learning methods. For instance, on UCR and MSL datasets, Var-100 achieves the best A-P scores. This phenomenon has been noted in recent studies~\citep{liu2024elephant,zhu2026foundation}, suggesting that many existing methods fail to substantially improve over simple statistical baselines. However, NCL-enhanced TSFMs (TMM+NCL and UniTS+NCL with average ranks of 3.77 and 2.13 in PR metrics, and average ranks of 3.47 and 2.73 in ROC metrics) consistently surpass these baselines, demonstrating that normality constraint learning successfully addresses the fundamental limitation of TSFMs in anomaly detection.

\textbf{Significant improvement for TSFMs.} As analyzed above, NCL-enhanced TSFMs achieve better performance compared to deep TSAD and one-liner methods. More intuitively, in terms of the PR metrics, the average rank of the original TMM was only 7.47, significantly lower than 5.13 for CroAD and 4.00 for Var-100. After adding NCL, the rank improves to 3.77. It is the gains brought by NCL that enable TSFMs to achieve the best ranks compared to the baselines. \eg, 4.40 $\rightarrow$ 2.13 (UniTS on PR metrics), 7.23 $\rightarrow$ 3.47 (TMM on ROC metrics), 4.93 $\rightarrow$ 2.73 (UniTS on ROC metrics). 
The consistent improvement of NCL across diverse metrics and datasets validates that the proposed normality constraint effectively steers frozen TSFMs toward enhanced anomaly discrimination.

\input{tab/all_ROC_compare.tex}

\section{Comparison with TSAD-Specific Foundation Models}\label{compare_with_tsad}

\input{tab/ad_TSFM_compare.tex}
We further compare NCL against state-of-the-art TSAD-specific foundation models DADA~\citep{shentu2025towards}, TimeRCD~\citep{lan2026towards}, and TimeRadar~\citep{he2026timeradar} in both zero-shot and few-shot settings. We apply NCL to four representative general-purpose TSFMs: Time-MoE, TimesFM, TMM, and UniTS, and the results are presented in \cref{tab:ad_tsfm_compare}.
NCL demonstrates substantial improvements when integrated with general-purpose TSFMs, enabling them to achieve competitive or superior performance compared to TSAD-specific foundation models. Specifically, UniTS+NCL achieves the best performance on 4 out of 10 datasets, while TMM+NCL excels on 2 datasets. Both Time-MoE+NCL and TimesFM+NCL are competitive on datasets such as PSM and SMD. This indicates that NCL effectively transforms them into powerful anomaly detectors.

While DADA, TimeRCD, and TimeRadar are specifically engineered for anomaly detection, they show inconsistent performance across datasets. DADA$_\text{five}$ performs well on PUMP (0.2541) and NYC (0.3744), but struggles on other datasets like IOPS (0.1453) and SWAT (0.4262). TimeRCD$_\text{zero}$ achieves notable performance on SWAT (0.5721) but falls short on most other benchmarks. TimeRadar$_\text{zero}$ achieves competitive performance on certain datasets such as SWAT (0.5802), yet underperforms on other datasets. In contrast, NCL-augmented TSFMs demonstrate more robust performance across diverse domains, achieving the higher average performance across all datasets. This suggests that general-purpose TSFMs, when properly adapted through NCL, can leverage their broad pre-trained knowledge more effectively than models trained specifically for anomaly detection from scratch.

\section{One-to-One Comparison with One-Liner Baselines}\label{app:one_to_one}

To provide a more granular comparison against simple one-liner baselines, we conduct one-to-one comparisons across all 361 time series from the ten datasets. \cref{chronos_two} and \cref{moment_two} visualize the win-loss records of Chronos and MOMENT against three representative baselines, respectively.
\begin{figure}[!b]
	\centering
	\begin{subfigure}{1\linewidth}
		\includegraphics[width=1.0\textwidth]{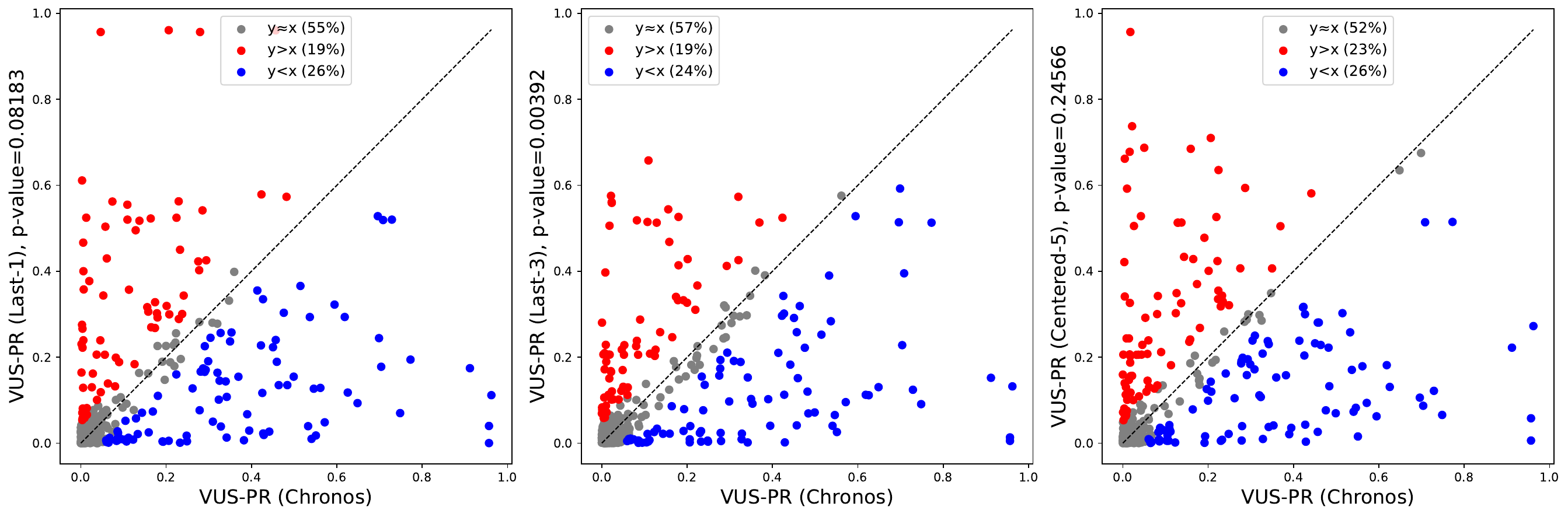}
		\caption{Comparison between Chronos and the squared-difference baselines.}
		\label{Chronos_compare}
	\end{subfigure}
	\vspace{0.01\textwidth}
	
	\begin{subfigure}{1\linewidth}
		\includegraphics[width=1.0\textwidth]{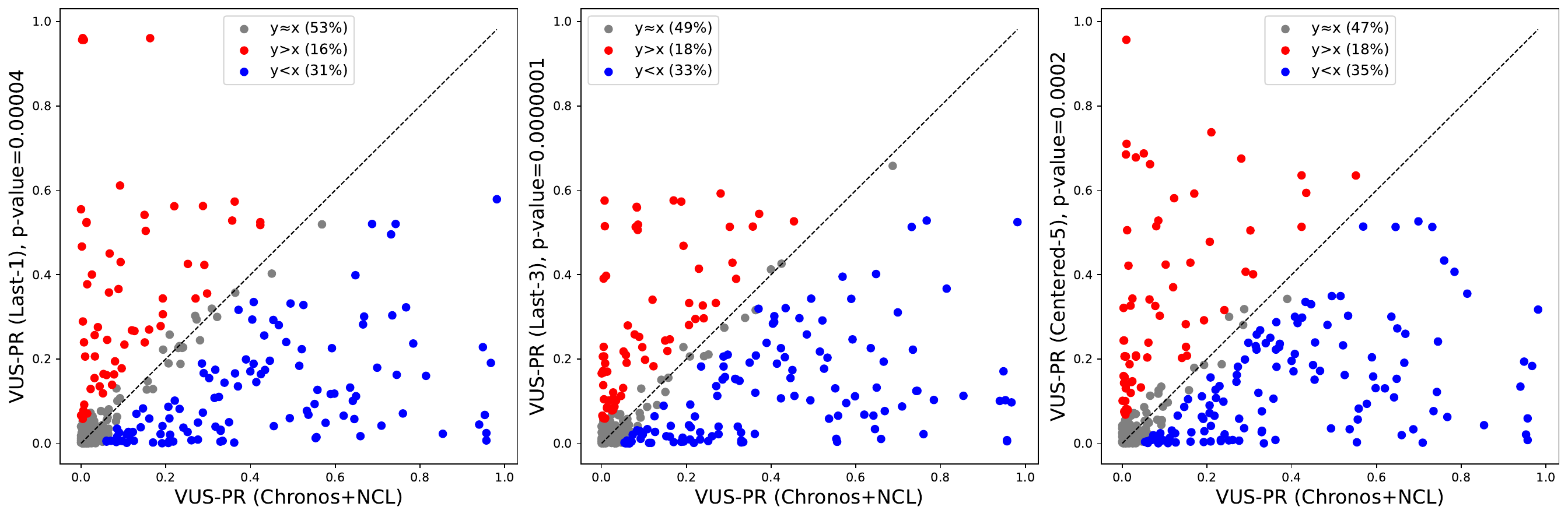}
		\caption{Comparison between Chronos+NCL and the squared-difference baselines.}
		\label{Chronos+ncl_compare}
	\end{subfigure}
	\caption{One-to-one comparison between Chronos (without and with NCL) and three squared-difference baselines across 361 time series from ten TSAD datasets.}
	\label{chronos_two}
\end{figure}

\textbf{Chronos vs. Squared-Difference Baselines.} \cref{Chronos_compare} shows that the vanilla Chronos struggles against simple squared-difference baselines (Last-1, Last-3, Centered-5), winning only 26\% against Last-1, 24\% against Last-3, and 26\% against Centered-5. More importantly, the $p$-value between Chronos and Centered-5 is 0.246, indicating that the statistics are not significant. This confirms the observation in recent studies~\citep{zhu2026foundation} that foundation models without proper normality constraints often fail to outperform trivial baselines. By contrast, as shown in \cref{Chronos+ncl_compare}, Chronos+NCL dramatically improves the win rate, achieving \textbf{31\%} against Last-1, \textbf{33\%} against Last-3, and \textbf{35\%} against Centered-5. In the meantime, the $p$-value between Chronos+NCL and Centered-5 is 0.0002, indicating a significant difference. This substantial improvement demonstrates NCL's effectiveness in steering the foundation model toward normality-biased predictions that amplify anomalies.

\textbf{MOMENT vs. Moving-Window Variances.} \cref{MOMENT_compare} reveals a similar pattern for reconstruction-based TSFMs. Vanilla MOMENT wins only 34\% against Var-96, 38\% against Var-128, and 39\% against Var-512, and the $p$-value between MOMENT and Var-96 is 0.4306. After applying NCL, as shown in \cref{MOMENT+ncl_compare}, MOMENT+NCL significantly improves to \textbf{40\%} against Var-96, \textbf{40\%} against Var-256, and \textbf{43\%} against Var-512; and reduces the $p$-value between MOMENT and Var-96 to 0.0303. The improvements are particularly pronounced against the stronger Var-96 baseline, indicating that NCL effectively captures fine-grained normal patterns that simple variance measures cannot.

The consistent improvements across hundreds of individual time series provide strong statistical evidence for NCL's effectiveness. While NCL substantially improves overall performance, there are time series where simple baselines outperform NCL-enhanced TSFMs. These cases stem from two main aspects: (i) time series with extremely simple periodic patterns where one-liner methods naturally excel; and (ii) time series with non-stationary training samples where the normal subspace construction becomes unstable. Future work could explore stationary processing strategies to address the latter limitation.

\begin{figure}[h]
	\centering
	\begin{subfigure}{1\linewidth}
		\includegraphics[width=1.0\textwidth]{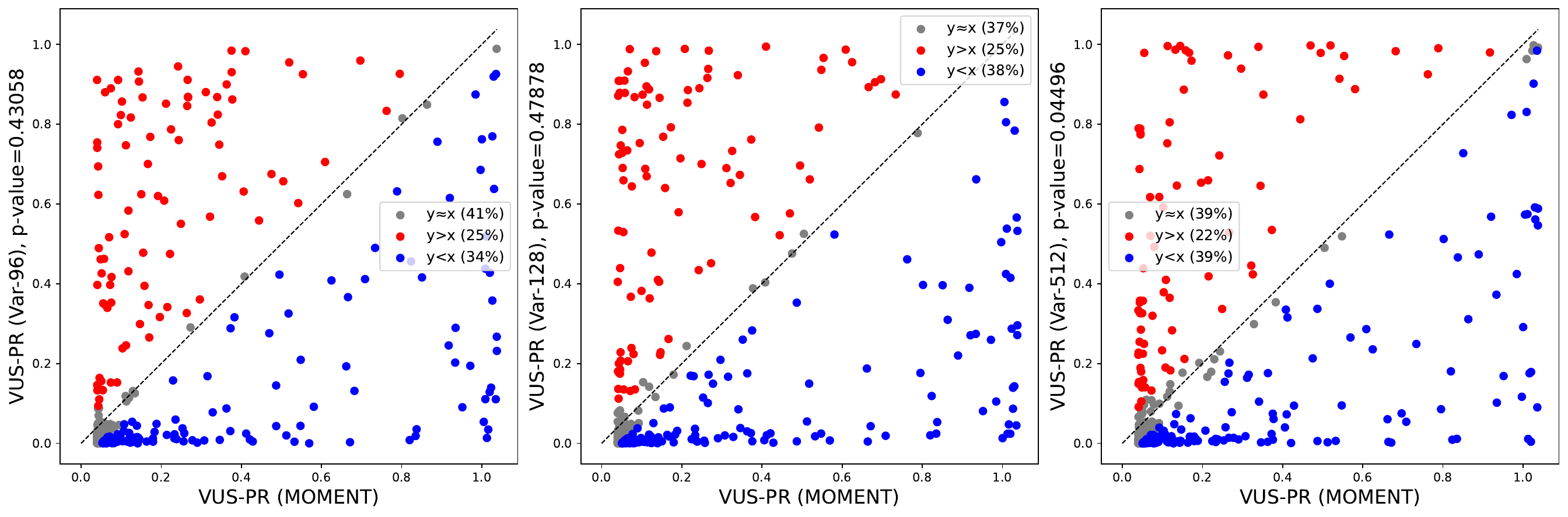}
		\caption{Comparison between MOMENT and the moving-window variances.}
		\label{MOMENT_compare}
	\end{subfigure}
	\vspace{0.01\textwidth}
	
	\begin{subfigure}{1\linewidth}
		\includegraphics[width=1.0\textwidth]{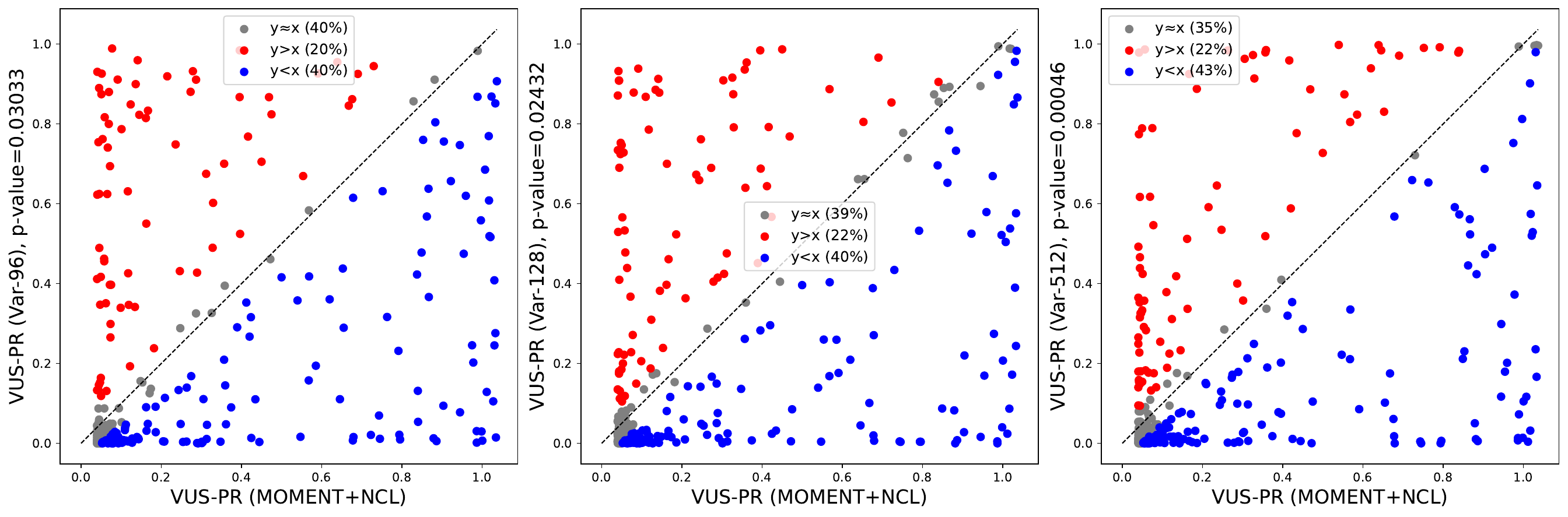}
		\caption{Comparison between MOMENT+NCL and the moving-window variances.}
		\label{MOMENT+ncl_compare}
	\end{subfigure}
	\caption{One-to-one comparison between MOMENT (without and with NCL) and three moving-window variances across 361 time series from ten TSAD datasets.}
	\label{moment_two}
\end{figure}

\section{Distinguishable Distributions of Anomaly Scores}\label{app:anomly_dis}

\cref{Two_dis} provides additional evidence for NCL's effectiveness by visualizing the anomaly score distributions on different TSFMs and datasets. The results consistently show that NCL successfully widens the gap between normal and abnormal score distributions.
As shown in \cref{SMD_Aurora}, vanilla Aurora produces highly overlapping distributions between normal and abnormal observations, with substantial overlap in the 0-0.2 score range. This overlap makes it extremely difficult to set an appropriate threshold for anomaly detection. After applying NCL, the distributions become well-separated: normal observations concentrate around 0-0.1, while anomalous observations shift toward 0.1-0.4 with a longer tail extending to higher scores, creating a clear separation margin. The reduced overlap translates directly to improved detection performance with fewer false positives and false negatives. \cref{PUMP_UniTS} reveals a similar pattern for reconstruction-based TSFMs. Vanilla UniTS exhibits significant distribution overlap with both normal and abnormal scores concentrated in the 0-0.2 range, making discrimination challenging. With NCL, the normal distribution becomes more compact and concentrated near zero, while the abnormal distribution shifts rightward.

The consistent distributional improvements across different TSFMs (forecasting-based Aurora and reconstruction-based UniTS) and datasets (multivariate SMD and PUMP) demonstrate that NCL's normality constraint is universally effective. Regardless of the TSFM architecture or data characteristics, NCL successfully steers the model toward normality-biased predictions that amplify the normal-abnormal gap.

\begin{figure}[h]
	\centering
	\begin{subfigure}{1\linewidth}
		\includegraphics[width=1.0\textwidth]{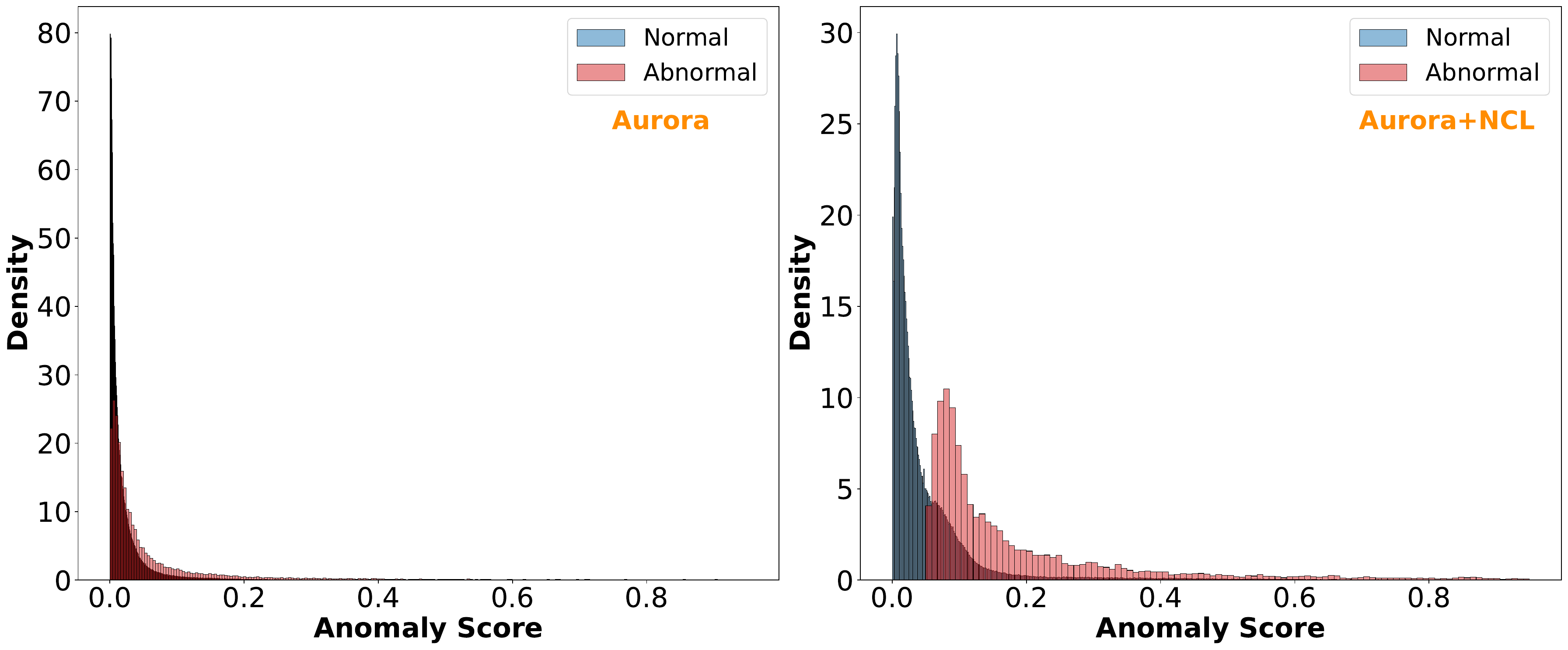}
		\caption{Anomaly scores for Aurora on the SMD dataset.}
		\label{SMD_Aurora}
	\end{subfigure}
	\vspace{0.02\textwidth}

	\begin{subfigure}{1\linewidth}
		\includegraphics[width=1.0\textwidth]{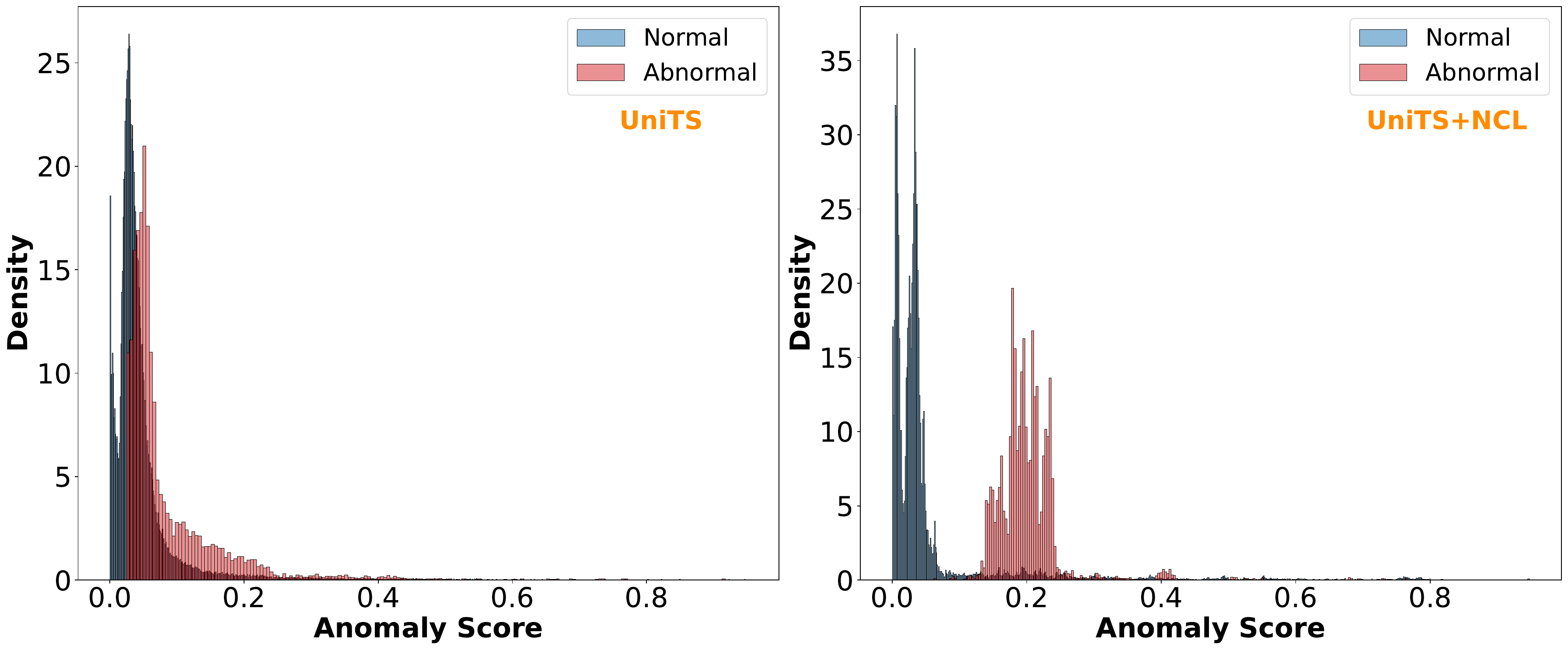}
		\caption{Anomaly scores for UniTS on the PUMP dataset.}
		\label{PUMP_UniTS}
	\end{subfigure}
	\caption{Anomaly score distributions for different TSFMs and datasets.}
	\label{Two_dis}
\end{figure}
\section{Visualization of Feature Space}\label{app:visual_feature}

To understand how NCL sharpens the feature normality within the feature space, we utilize t-SNE dimensionality reduction to visualize patch features from vanilla and NCL-augmented TSFMs. \cref{latent} presents the feature space visualization on two representative TSFM-dataset combinations.
As shown in \cref{TMM_PUMP}, vanilla TMM produces a feature space where normal and abnormal patches are heavily intermingled with no clear boundary. The normal features are scattered across a large region, and many abnormal features fall within the normal region, making discrimination extremely difficult. After applying NCL (as shown in \cref{TMM_PUMP_NCL}), the feature space undergoes a dramatic transformation. The abnormal features are clearly separated from the normal clusters.
\cref{UniTS_PSM} and \ref{UniTS_PSM_NCL} show a similar pattern for reconstruction-based TSFMs.
The feature visualizations provide intuitive geometric understanding of how NCL improves anomaly detection. By explicitly constructing an informative normality subspace and constraining features to align with the normality manifolds, NCL transforms the originally broad pattern space into one where normal and abnormal patterns occupy geometrically distinct regions. This spatial separation makes subsequent reconstruction/forecasting error-based anomaly scoring much more discriminative.

\begin{figure}[!htbp]
	\centering
	\begin{subfigure}[t]{0.48\textwidth}
		\centering 
		\includegraphics[width=\linewidth]{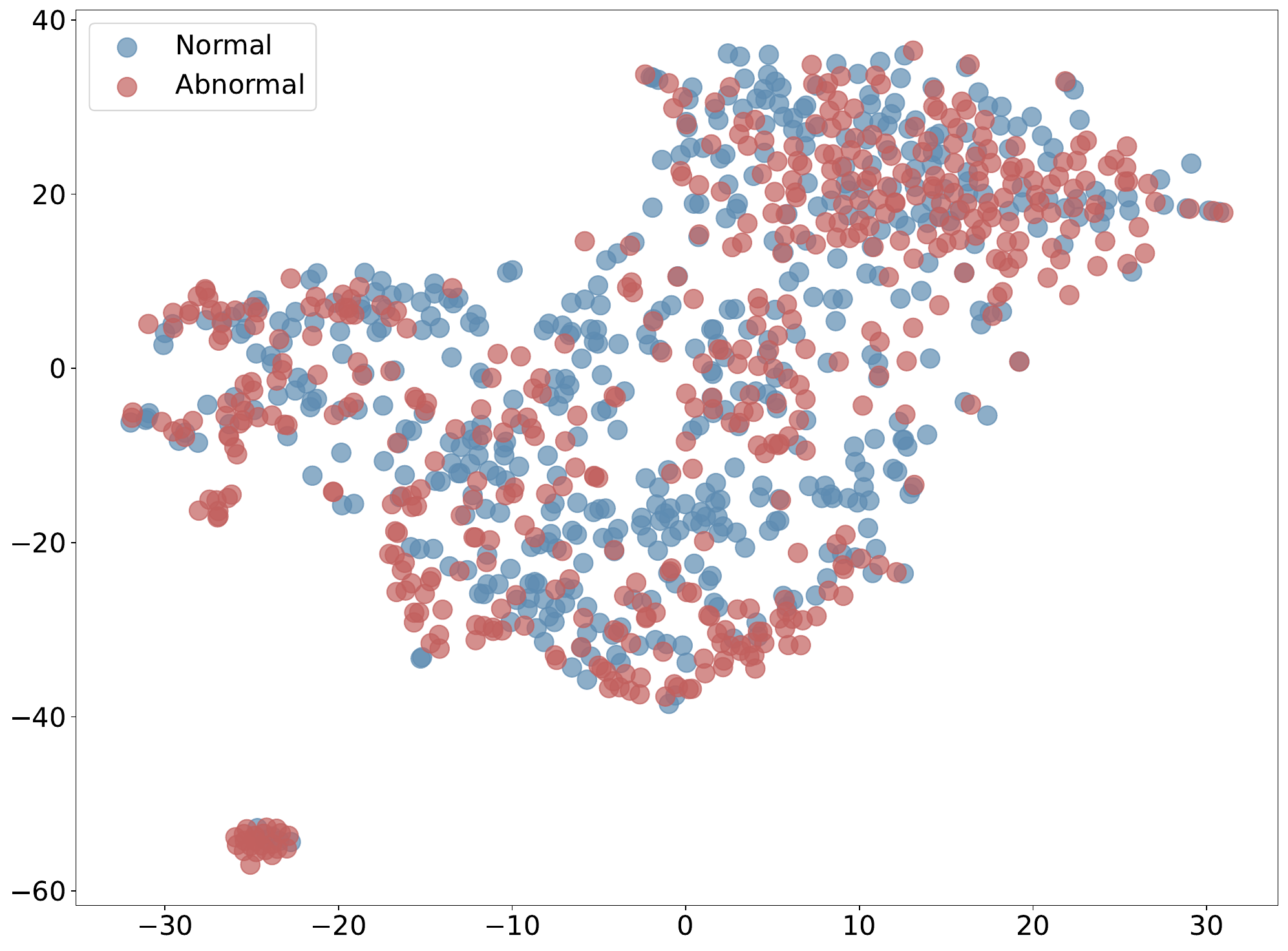} 
		
		\caption{Vanilla TMM on the PUMP dataset.}
		\label{TMM_PUMP}
	\end{subfigure}
	\hspace{0.02\textwidth}
	\begin{subfigure}[t]{0.48\textwidth}
		\centering
		\includegraphics[width=\linewidth]{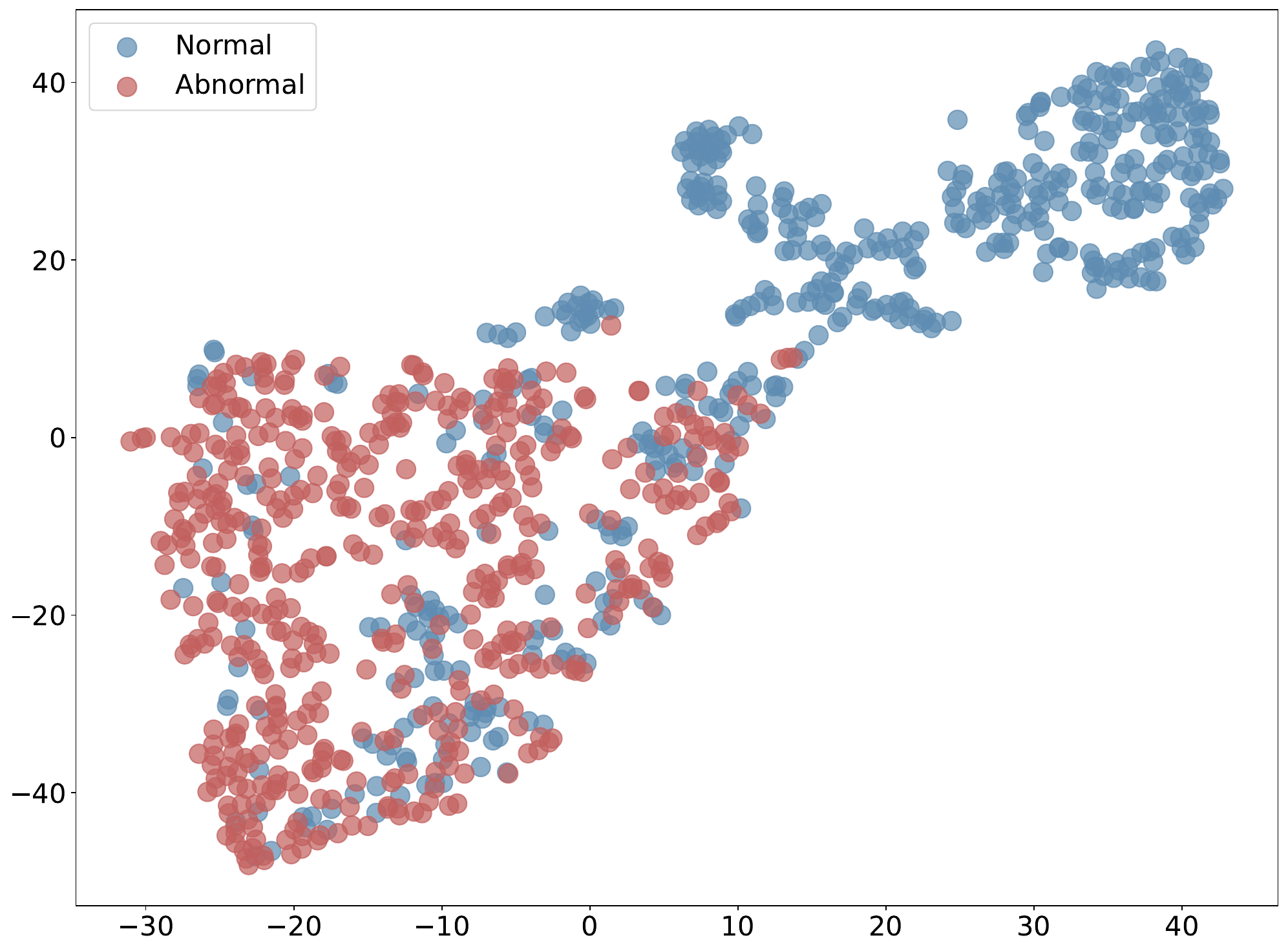}
		\caption{NCL-augmented TMM on the PUMP dataset.}
		\label{TMM_PUMP_NCL}
	\end{subfigure}  \\
	\vspace{0.02\textwidth}
	
	\begin{subfigure}[t]{0.48\textwidth}
		\centering 
		\includegraphics[width=\linewidth]{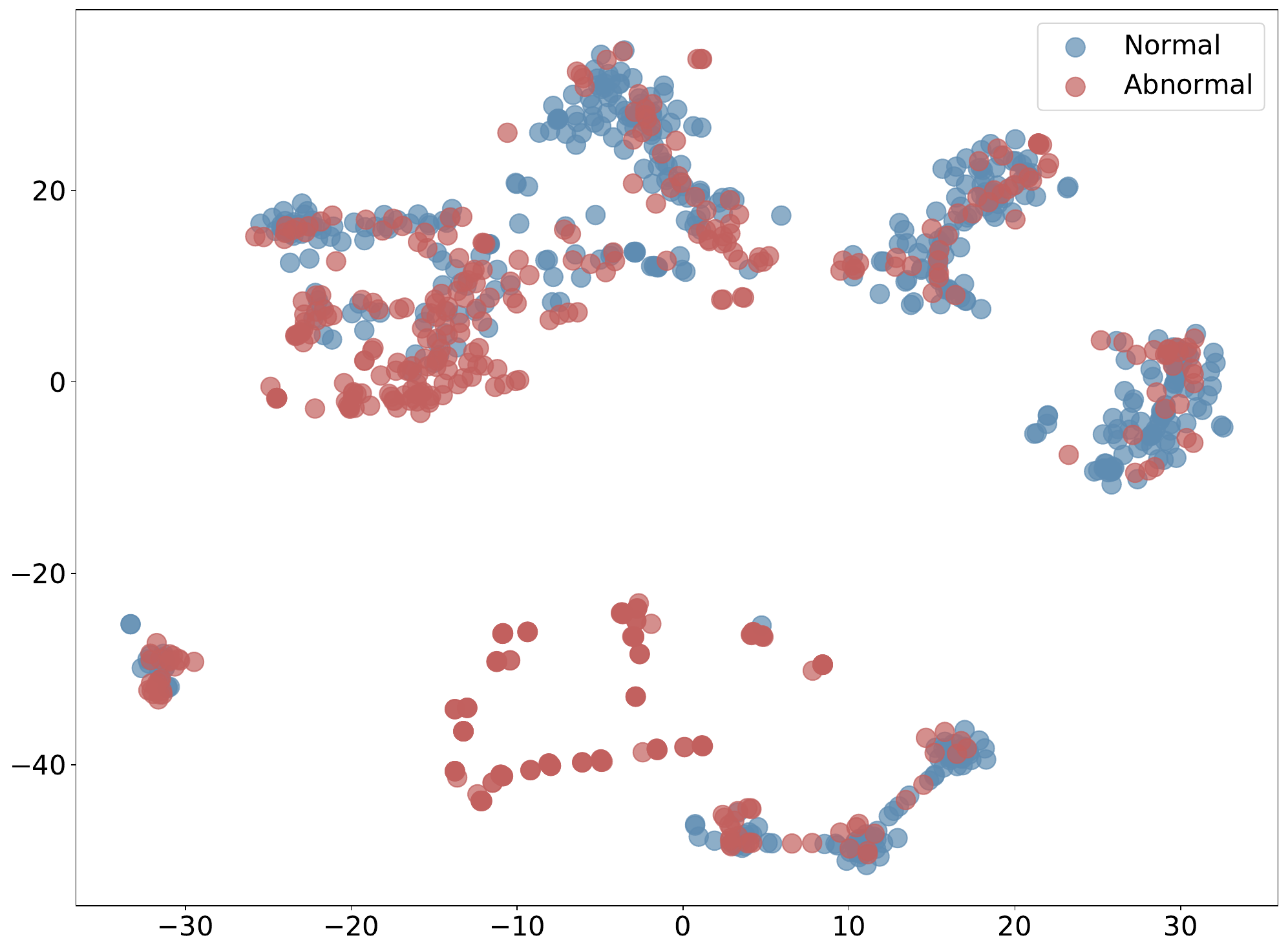} 
		
		\caption{Vanilla UniTS on the PSM dataset.}
		\label{UniTS_PSM}
	\end{subfigure}
	\hspace{0.02\textwidth}
	\begin{subfigure}[t]{0.48\textwidth}
		\centering
		\includegraphics[width=\linewidth]{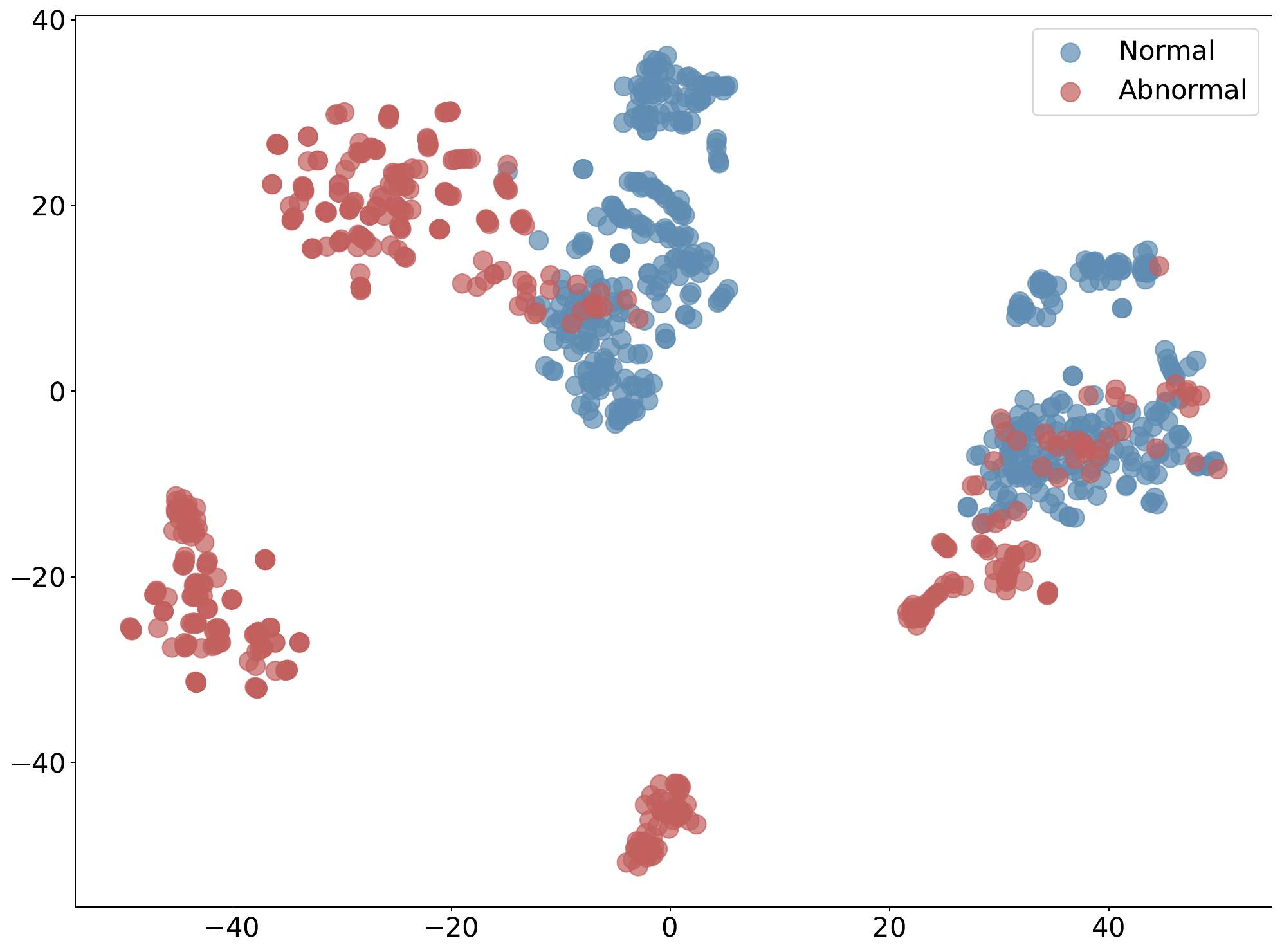}
		\caption{NCL-augmented UniTS on the PSM dataset.}
		\label{UniTS_PSM_NCL}
	\end{subfigure}%
	\caption{The t-SNE visualization of internal features from vanilla and NCL-augmented TSFMs.}
	\label{latent}
\end{figure}

\section{Horizon and Window Size Analysis for TSFMs}\label{seq_analysis}

To investigate the impact of forecasting horizons and reconstruction window sizes on anomaly detection performance, we conduct comprehensive experiments on both forecasting-based and reconstruction-based TSFMs. \cref{horizon} presents the VUS-PR scores across different horizon/window configurations on four representative datasets.

\textbf{Forecasting Horizon Analysis.} For forecasting-based TMM, we vary the forecasting horizon $F \in \{1, 2, 4, 8, 16, 32, 64, 128, 256\}$ while fixing the input window length. As shown in \cref{TMM_horizon} and \ref{TMM_horizon_NCL}, compared to vanilla TMM, NCL-enhanced TMM exhibits relatively stable performance across different horizons, and achieves excellent detection results even with small horizons.
This robustness stems from NCL's explicit construction of a normal subspace. Even when the forecasting task becomes harder, the contrastive normality constraint ensures that features are consistently aligned with normality manifolds, preventing the model from fitting anomalous patterns equally well.

\textbf{Reconstruction Window Size Analysis.} For reconstruction-based UniTS, we evaluate different window sizes $L \in \{32, 48, 64, 80, 96, 128, 192, 256, 512\}$. \cref{UniTS_win} and \ref{UniTS_win_NCL} reveal that NCL can stably enhance the detection performance of vanilla UniTS under different window sizes. For instance, on PSM, NCL continuously increases the VUS-PR from around 40 to nearly 50.
This validates that NCL's normality transformation operates effectively across different temporal scales. By constructing the normal subspace from patch features rather than raw observations, NCL captures scale-invariant normal patterns that remain discriminative across various window configurations.

The robustness to hyperparameter choices is crucial for practical deployment. Performing extensive grid search to identify optimal horizons or window sizes for each dataset is computationally expensive and may not generalize to new domains. NCL alleviates this burden by providing stable performance across a wide range of configurations, making it easier to deploy TSFMs for anomaly detection in real-world scenarios.

\begin{figure}[!tbp]
	\centering
	\begin{subfigure}[t]{0.48\textwidth}
		\centering 
		\includegraphics[width=\linewidth]{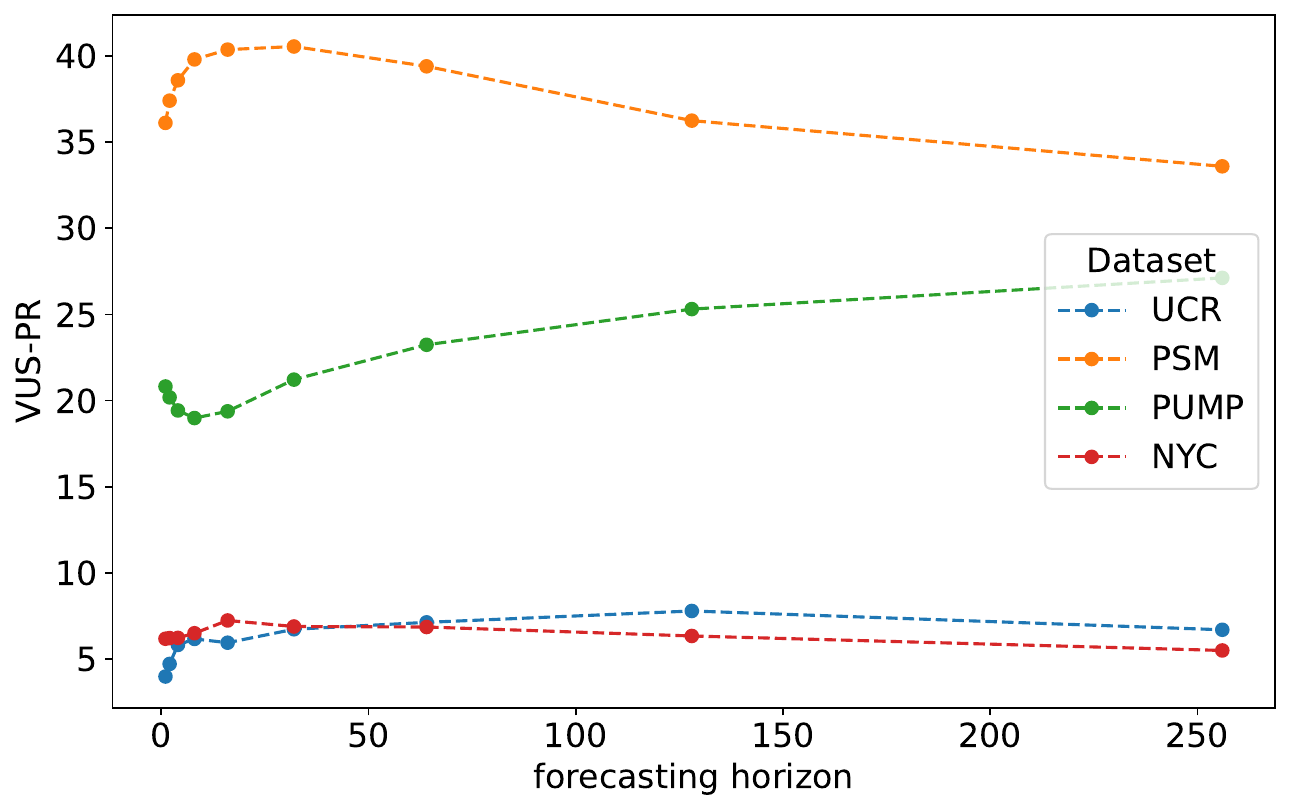} 
		
		\caption{Forecasting-based TMM}
		\label{TMM_horizon}
	\end{subfigure}
	\hspace{0.02\textwidth}
	\begin{subfigure}[t]{0.48\textwidth}
		\centering
		\includegraphics[width=\linewidth]{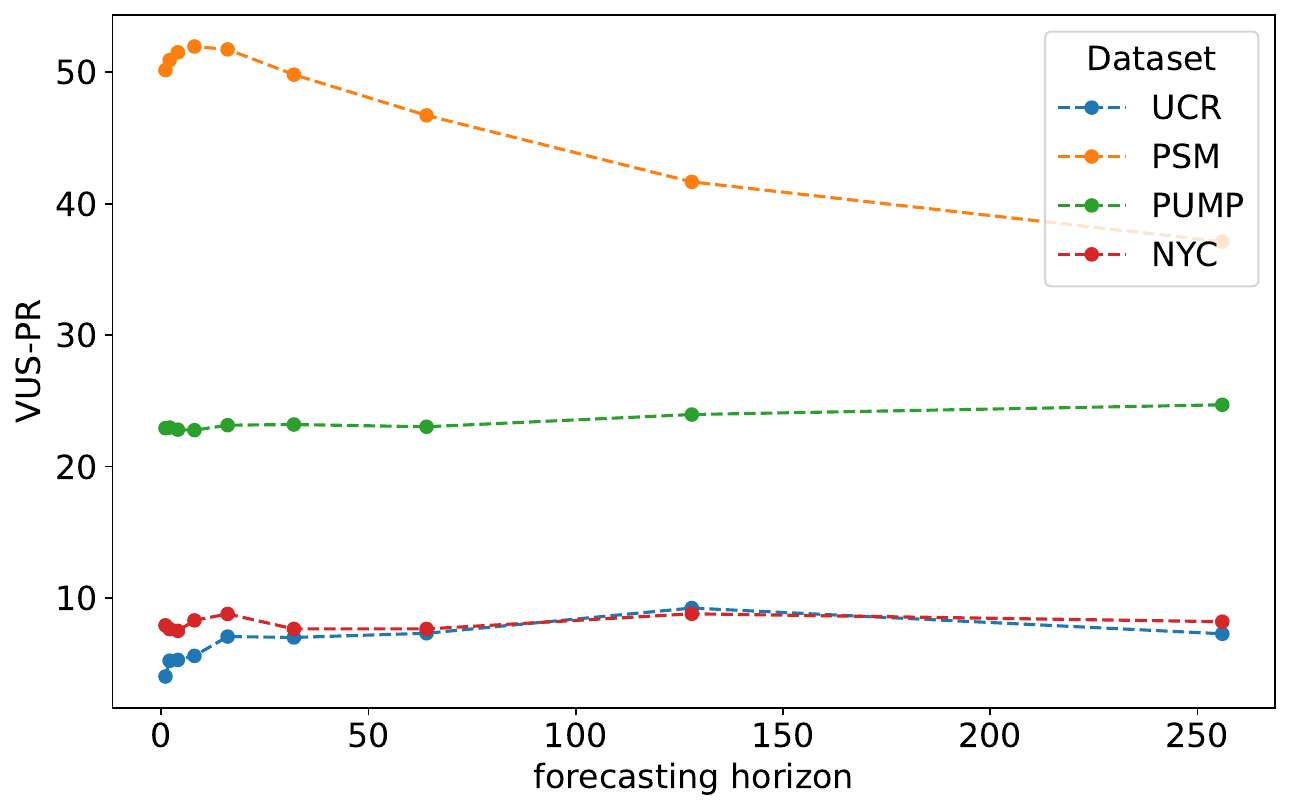}
		\caption{Forecasting-based TMM with NCL}
		\label{TMM_horizon_NCL}
	\end{subfigure}  \\
	\vspace{0.02\textwidth}
	
	\begin{subfigure}[t]{0.48\textwidth}
		\centering 
		\includegraphics[width=\linewidth]{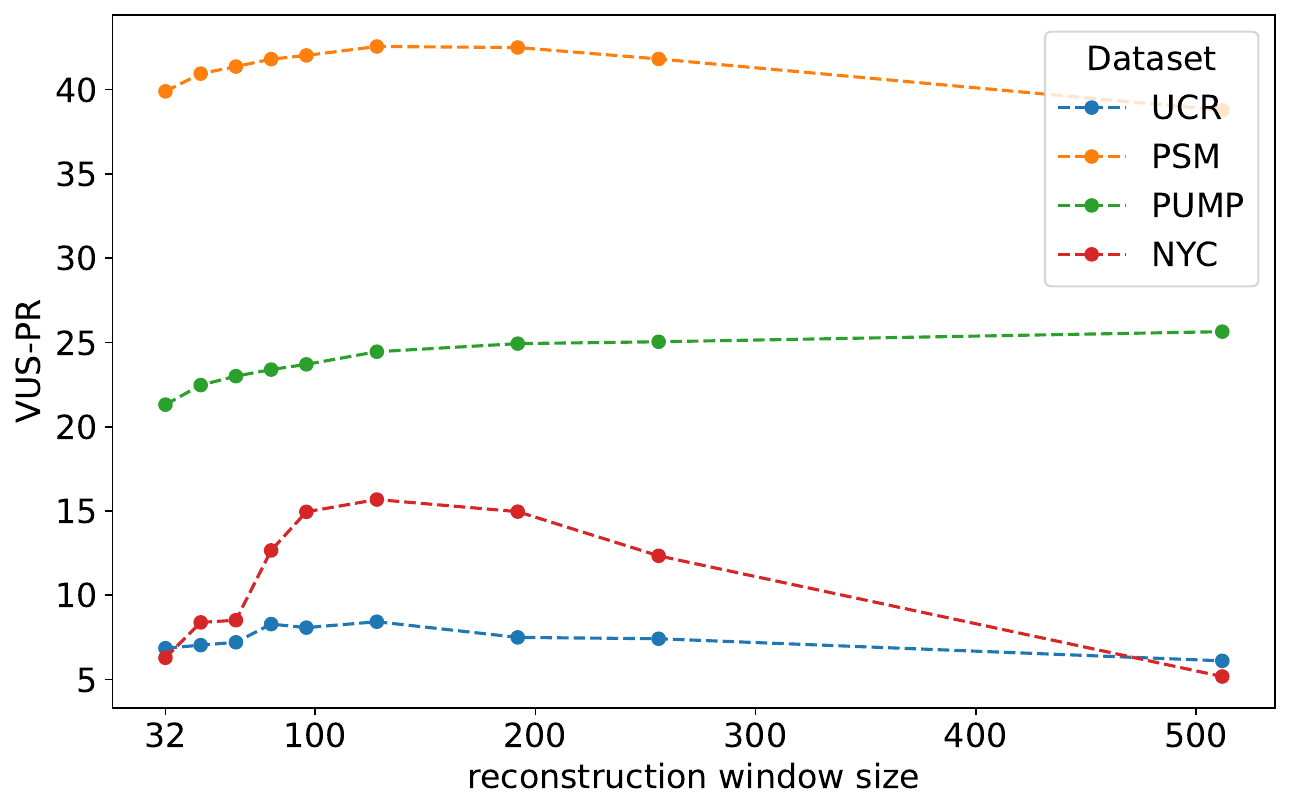} 
		
		\caption{Reconstruction-based UniTS}
		\label{UniTS_win}
	\end{subfigure}
	\hspace{0.02\textwidth}
	\begin{subfigure}[t]{0.48\textwidth}
		\centering
		\includegraphics[width=\linewidth]{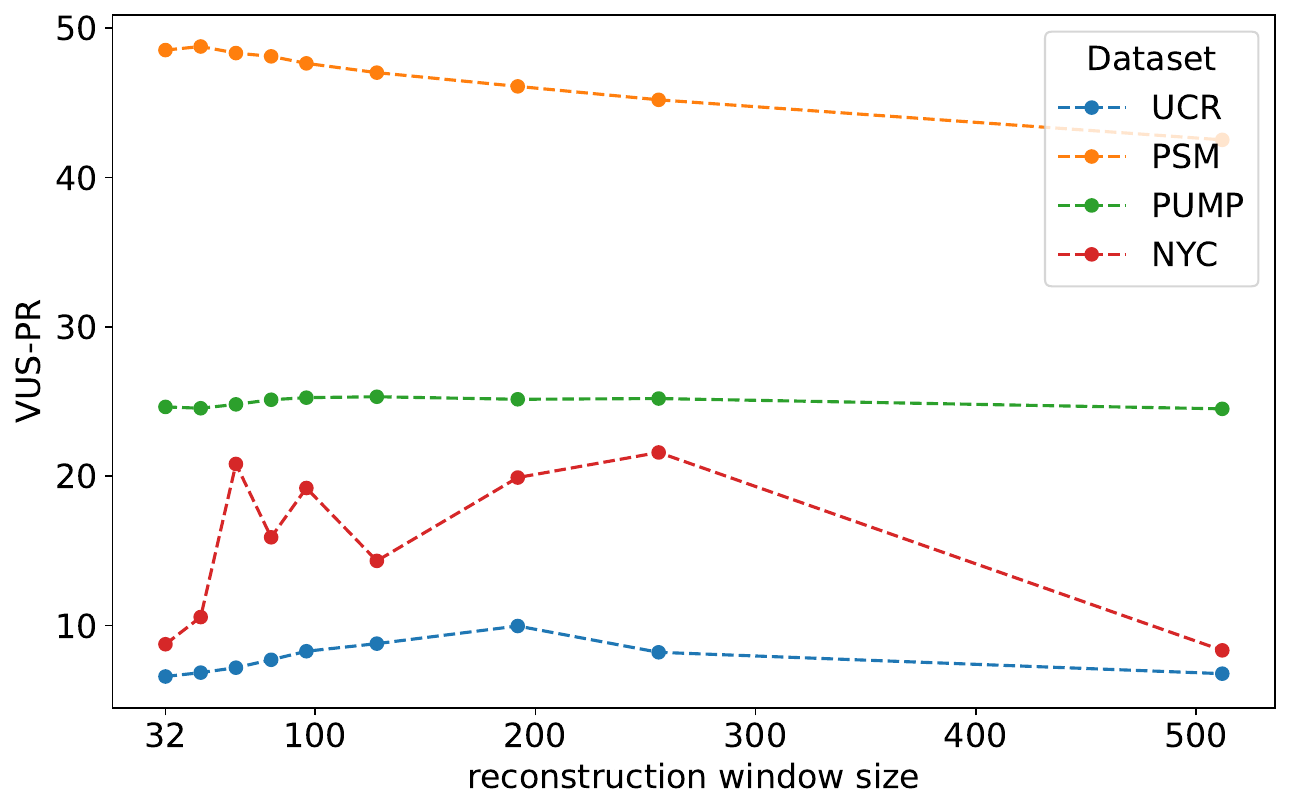}
		\caption{Reconstruction-based UniTS with NCL}
		\label{UniTS_win_NCL}
	\end{subfigure}%
	\caption{VUS-PR scores with various forecasting horizons and reconstruction window sizes, evaluated on four datasets for different TSFMs without and with using NCL.}
	\label{horizon}
\end{figure}

\section{Visualization of Anomaly Scores}\label{app:visual_anom}
\begin{figure}[!htbp]
	\centering
	\begin{subfigure}[t]{0.48\textwidth}
		\centering 
		\includegraphics[width=0.9\linewidth]{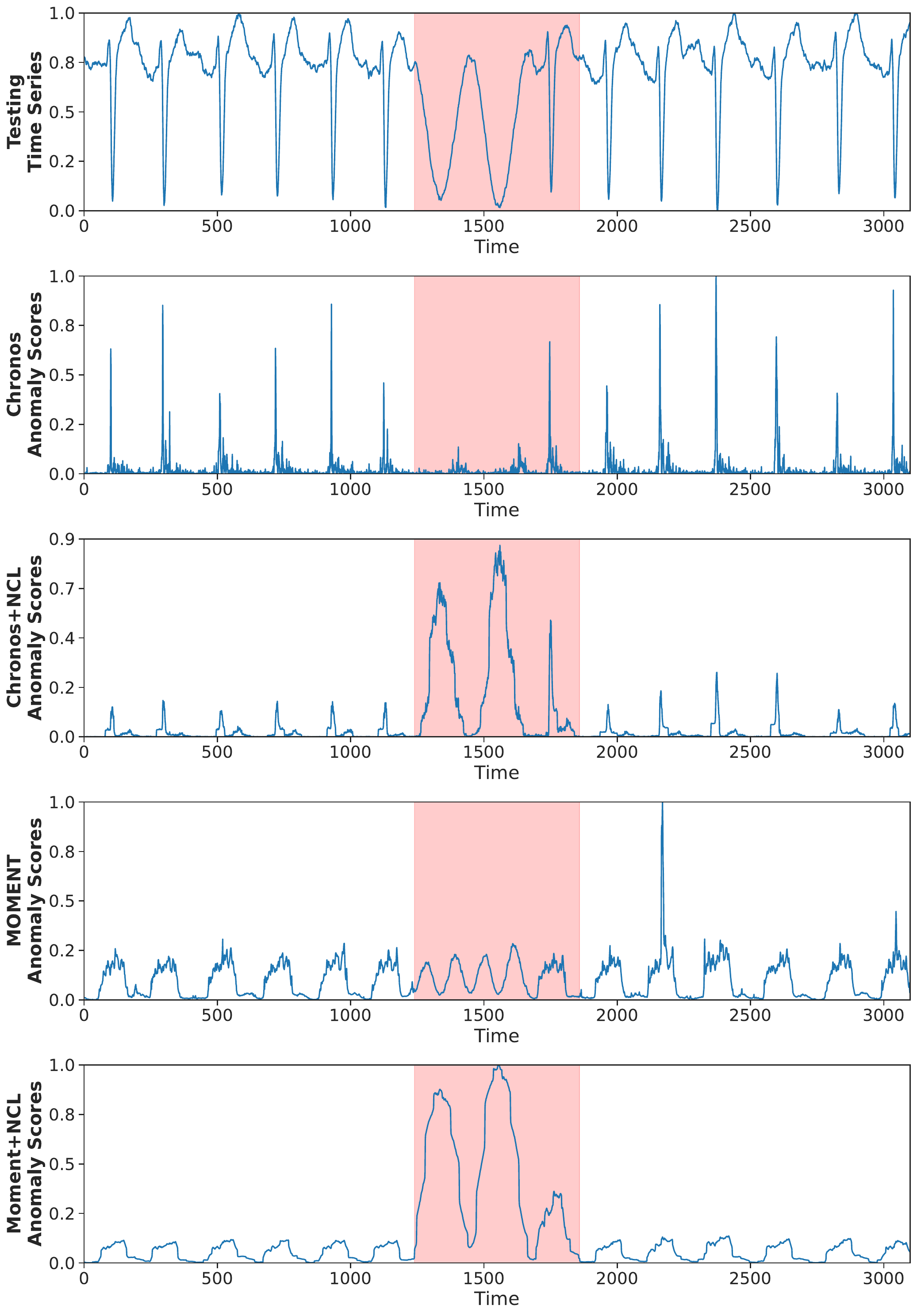} 
		
		\caption{DISTORTED1sddb40 from UCR dataset}
		\label{UCR_seg_one}
	\end{subfigure}
	\hspace{0.01\textwidth}
	\begin{subfigure}[t]{0.48\textwidth}
		\centering
		\includegraphics[width=0.9\linewidth]{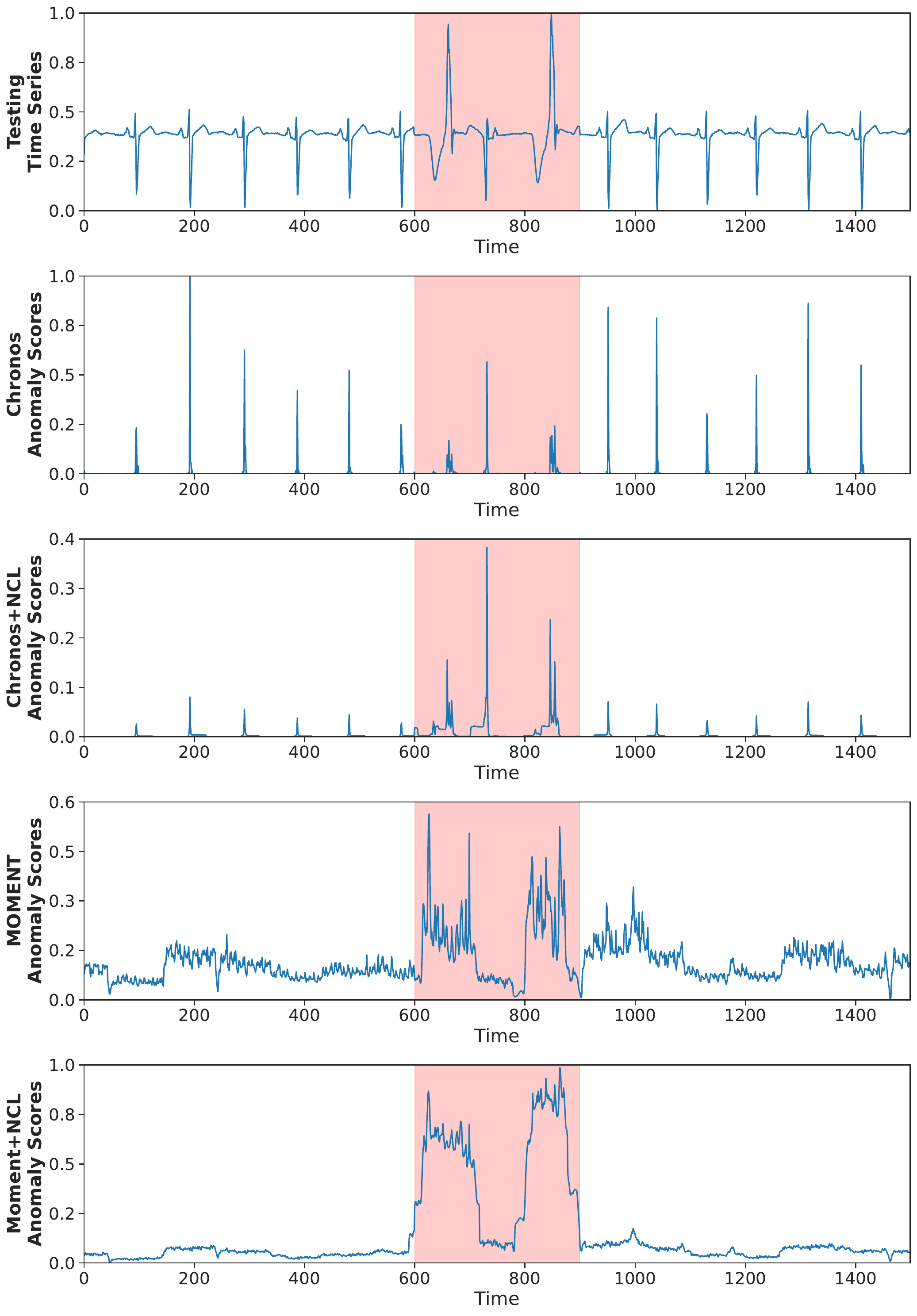}
		\caption{mit14046longtermecg from UCR dataset}
		\label{UCR_seg_two}
	\end{subfigure}  \\
	\vspace{0.02\textwidth}
	
	\begin{subfigure}[t]{0.48\textwidth}
		\centering 
		\includegraphics[width=0.9\linewidth]{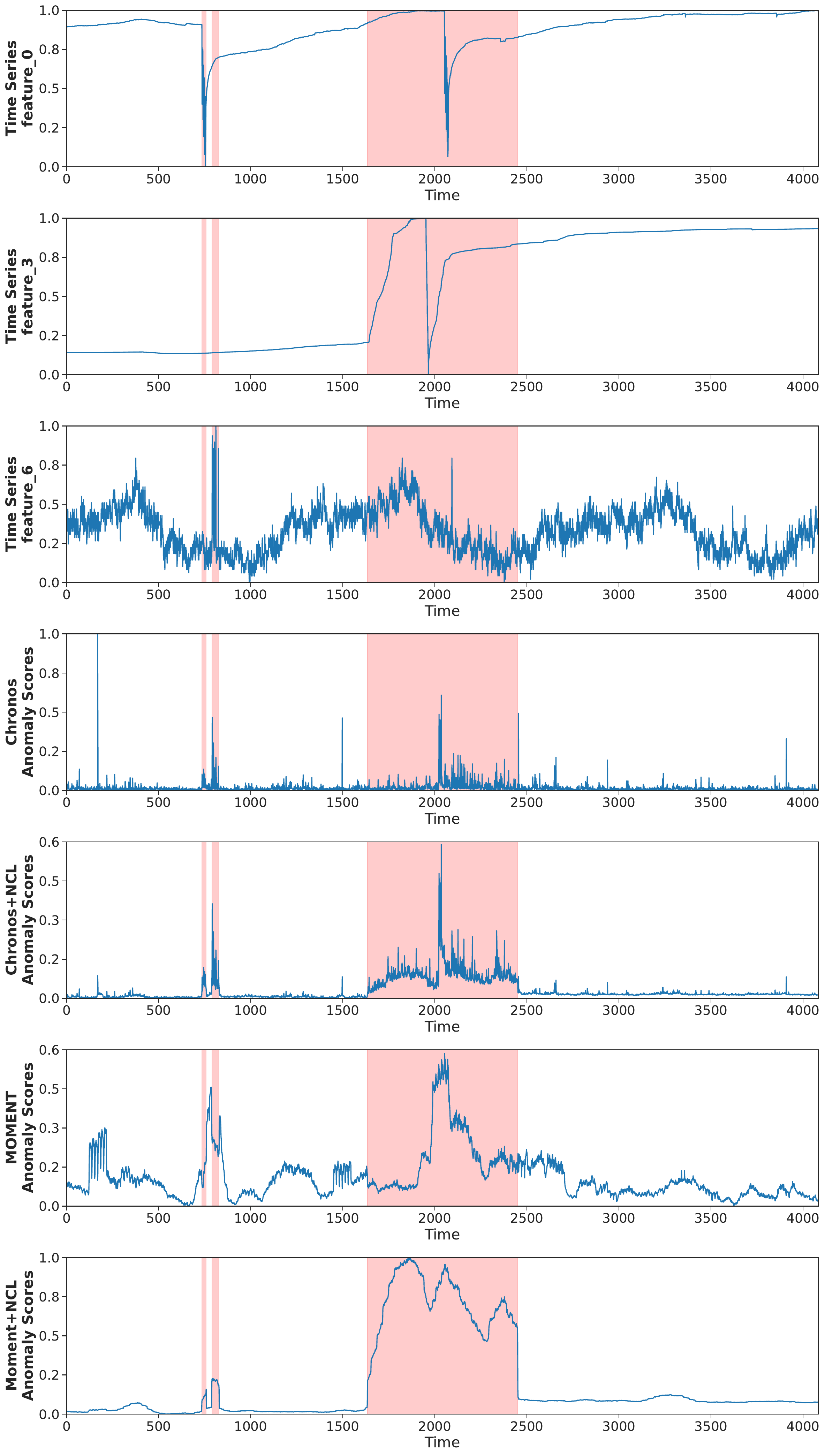} 
		
		\caption{Multivariate segment from PSM dataset}
		\label{PSM_seg}
	\end{subfigure}
	\hspace{0.01\textwidth}
	\begin{subfigure}[t]{0.48\textwidth}
		\centering
		\includegraphics[width=0.9\linewidth]{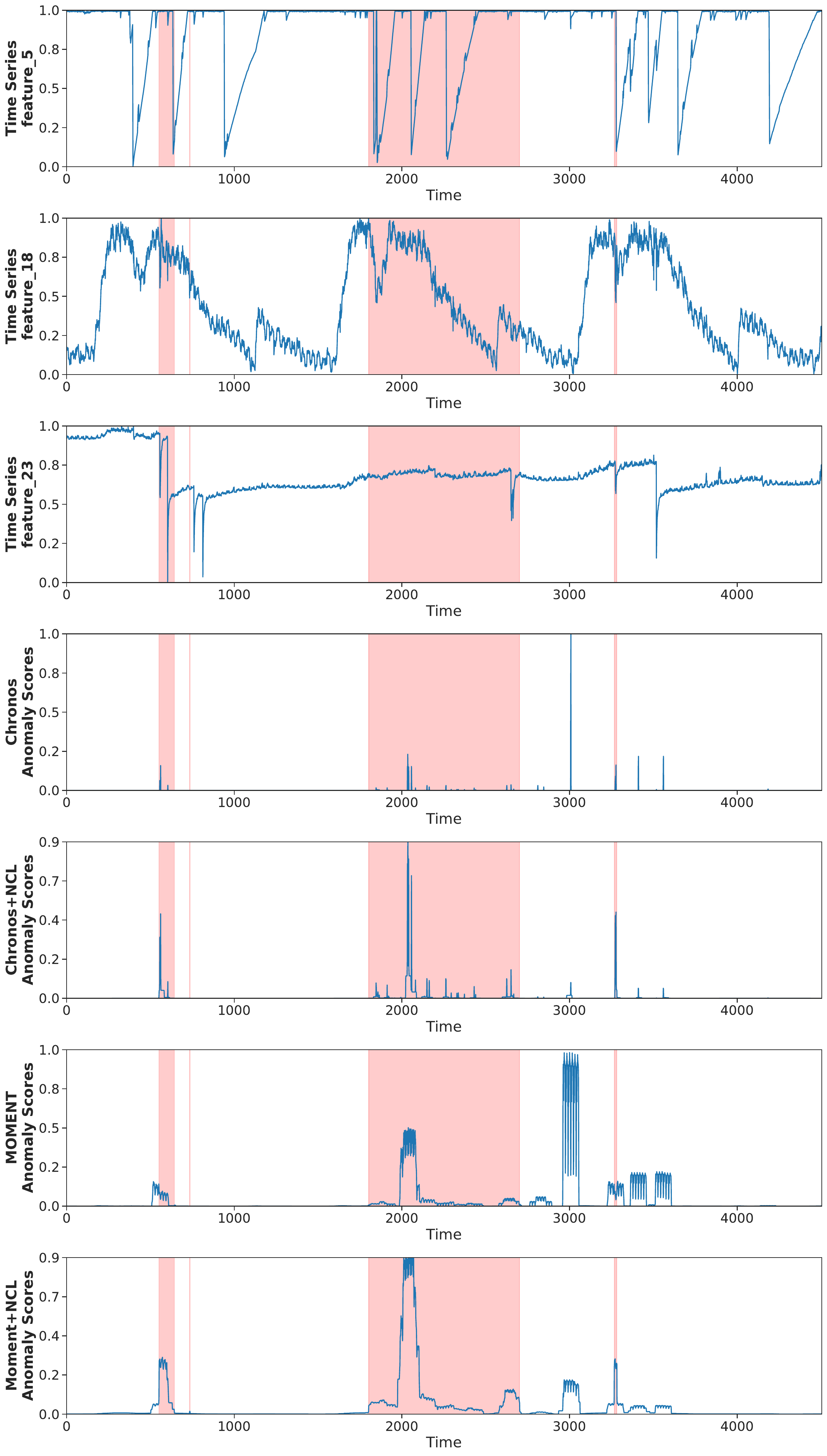}
		\caption{Multivariate segment from SMD dataset}
		\label{SMD_seg}
	\end{subfigure}%
	\caption{Visualization of anomaly scores in comparison with vanilla and NCL-augmented TSFMs on the univariate and multivariate datasets. Red regions highlight the ground truth anomaly intervals.}
	\label{visual}
\end{figure}

We visualize the anomaly scores and ground truth labels on representative time series segments from both univariate and multivariate datasets. \cref{visual} presents original time series, ground truth anomaly intervals, and anomaly scores for different TSFMs.
\cref{UCR_seg_one} and \ref{UCR_seg_two} show the detection results on the UCR time series. Notably, after applying NCL, the anomaly scores exhibit significantly improved alignment with the ground truth. The scores for normal regions are consistently suppressed close to zero, while true anomalies produce sharp, prominent peaks that clearly stand out. For instance, in \cref{UCR_seg_one}, the anomalous segment around timesteps 1250-1750 that was previously missed now generates distinctly elevated scores. This improvement validates that NCL's normality constraint effectively amplifies the normal-abnormal discrimination gap by steering features toward the normal subspace and penalizing deviations through contrastive misalignment. NCL-enhanced TSFMs also produce distinct, well-separated peaks responding to true anomalies in multivariate time series (as shown in \cref{PSM_seg} and \ref{SMD_seg}). The reduced false positive rate and improved peak sharpness on actual anomalies demonstrate NCL's effectiveness in multivariate scenarios where cross-channel dependencies add complexity. These visual improvements further confirm that NCL successfully amplifies the normal-abnormal discrimination signal that vanilla TSFMs lack.

\section{Utility of Few-Shot Normal Samples}\label{app:utility_few_shot}

A key practical question for NCL is how many normal samples are required to construct an informative normality subspace and learn accurate normality transformations. To investigate this, we evaluate NCL's performance with varying numbers of few-shot normal samples on three representative TSFMs: ROSE (forecasting-based), TMM (forecasting-based), and UniTS (reconstruction-based). \cref{Three_few_shot} presents the VUS-PR performance across ten datasets with sample sizes ranging from 1 to 20.
\begin{figure}[h]
	\centering
	\begin{subfigure}{1\linewidth}
		\includegraphics[width=1.0\textwidth]{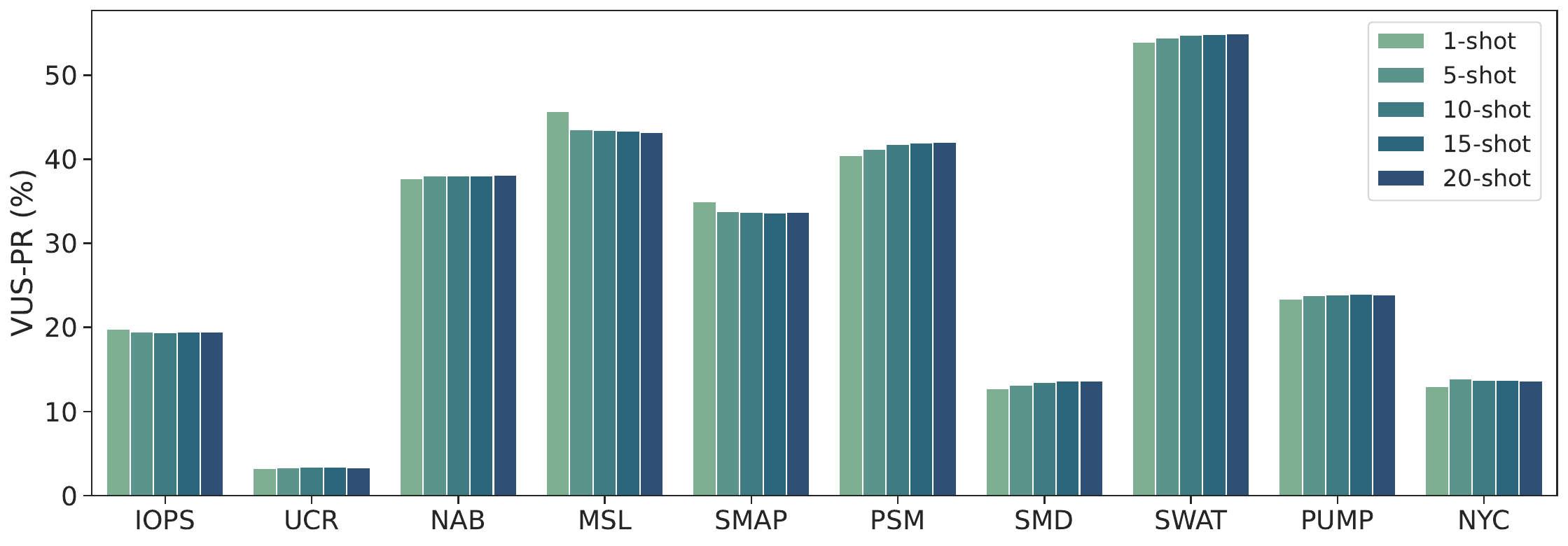}
		\caption{ROSE with NCL on ten TSAD datasets.}
		\label{ROSE_few_shot}
	\end{subfigure}
	\vspace{0.01\textwidth}
	
	\begin{subfigure}{1\linewidth}
		\includegraphics[width=1.0\textwidth]{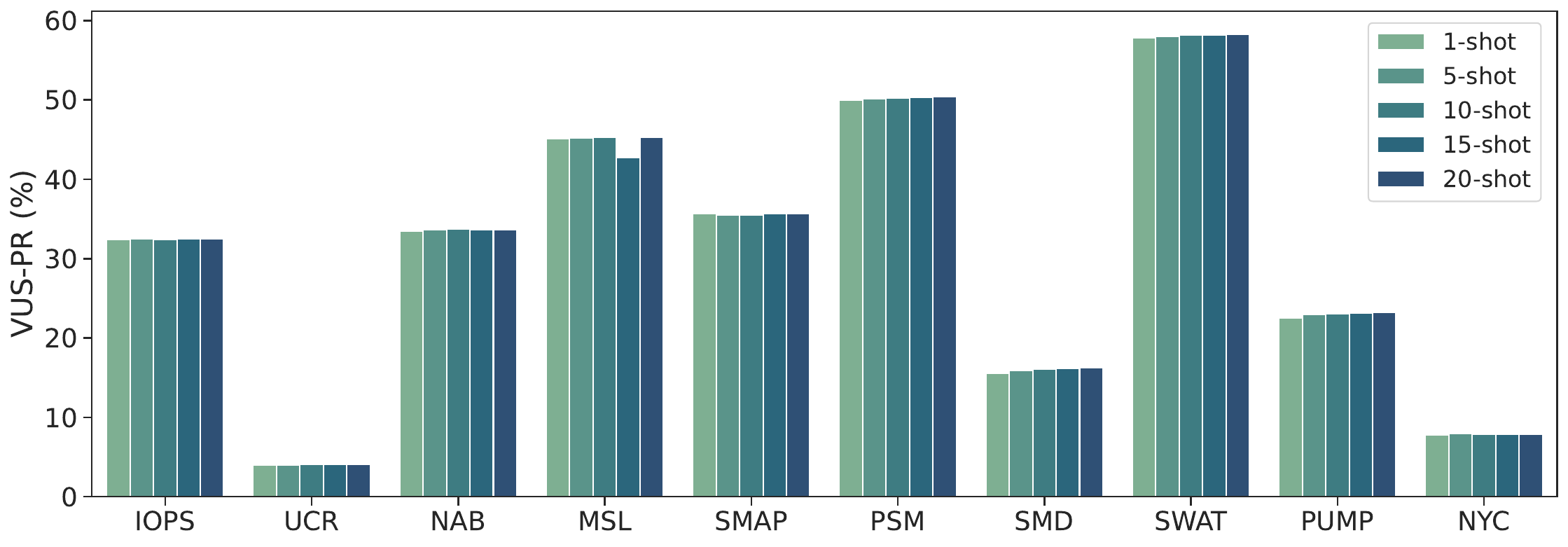}
		\caption{TMM with NCL on ten TSAD datasets.}
		\label{TMM_few_shot}
	\end{subfigure}
	\vspace{0.01\textwidth}
	
	\begin{subfigure}{1\linewidth}
		\includegraphics[width=1.0\textwidth]{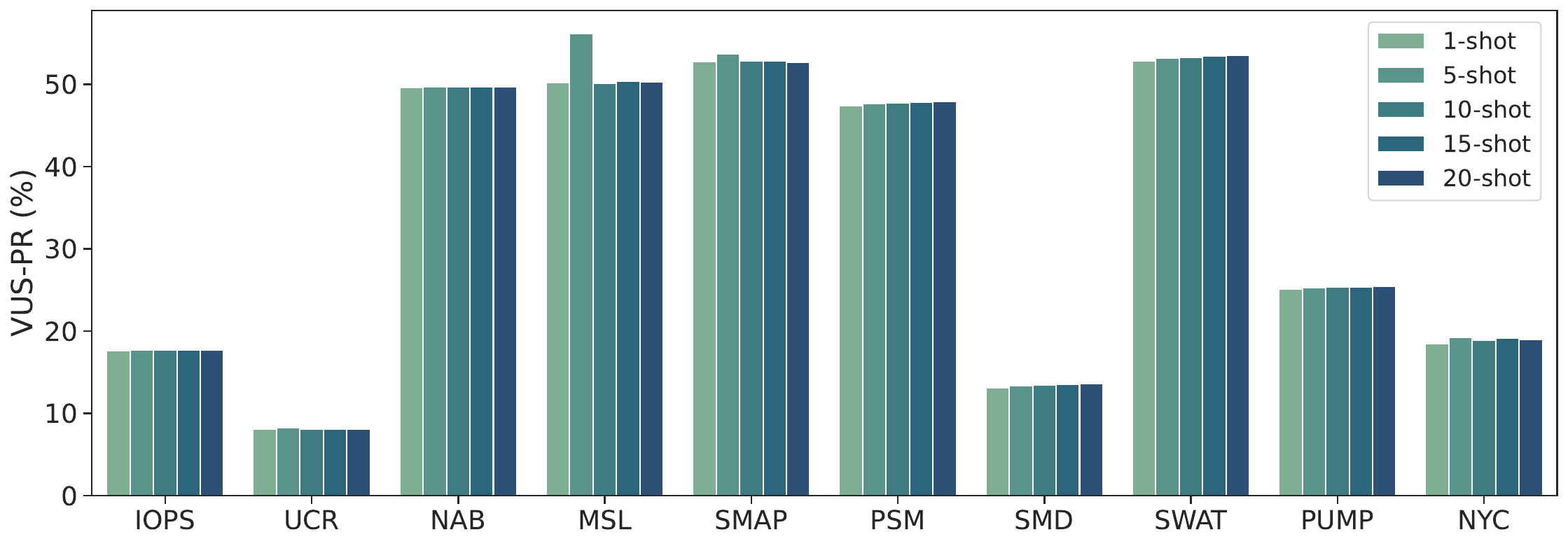}
		\caption{UniTS with NCL on ten TSAD datasets.}
		\label{UniTS_few_shot}
	\end{subfigure}
	\caption{VUS-PR performance of three NCL-augmented TSFMs w.r.t. different number of normal samples.}
	\label{Three_few_shot}
\end{figure}

Remarkably, NCL achieves competitive performance even with extremely limited normal samples. As shown in \cref{ROSE_few_shot}, ROSE+NCL with only one normal sample already achieves substantial improvements across all datasets. This demonstrates that even one representative normality sample can provide sufficient information to construct a meaningful normal subspace via SVD, which also confirms \cref{thm:few_shot_subspace} in \cref{few-shot-svd}. This extreme sample efficiency is crucial for scenarios where normal data is scarce or expensive to obtain, such as in industrial anomaly detection where collecting extensive normal operation data may require prolonged monitoring periods.

While performance generally improves with more normal samples, the gains exhibit diminishing returns beyond 5-10 samples. This pattern appears consistently across all three TSFMs, suggesting that the top singular vectors capturing the principal normal directions stabilize relatively quickly. Once the dominant axes of normal variation are identified, additional samples provide only marginal refinement.
It is worth noting that different datasets exhibit varying sensitivity to sample size. On relatively simple univariate datasets like IOPS and NAB where normal patterns are more homogeneous, 1-3 samples often suffice to capture the essential normal structure. On more complex multivariate datasets like SMD and SWAT with diverse normal modes, 5-10 samples provide better coverage of the normality manifolds. However, even on complex datasets, NCL never requires extensive data, and five samples consistently approach asymptotic performance across all tested scenarios.
\begin{figure}[!t]
	\centering
	\begin{subfigure}{1\linewidth}
		\includegraphics[width=1.0\textwidth]{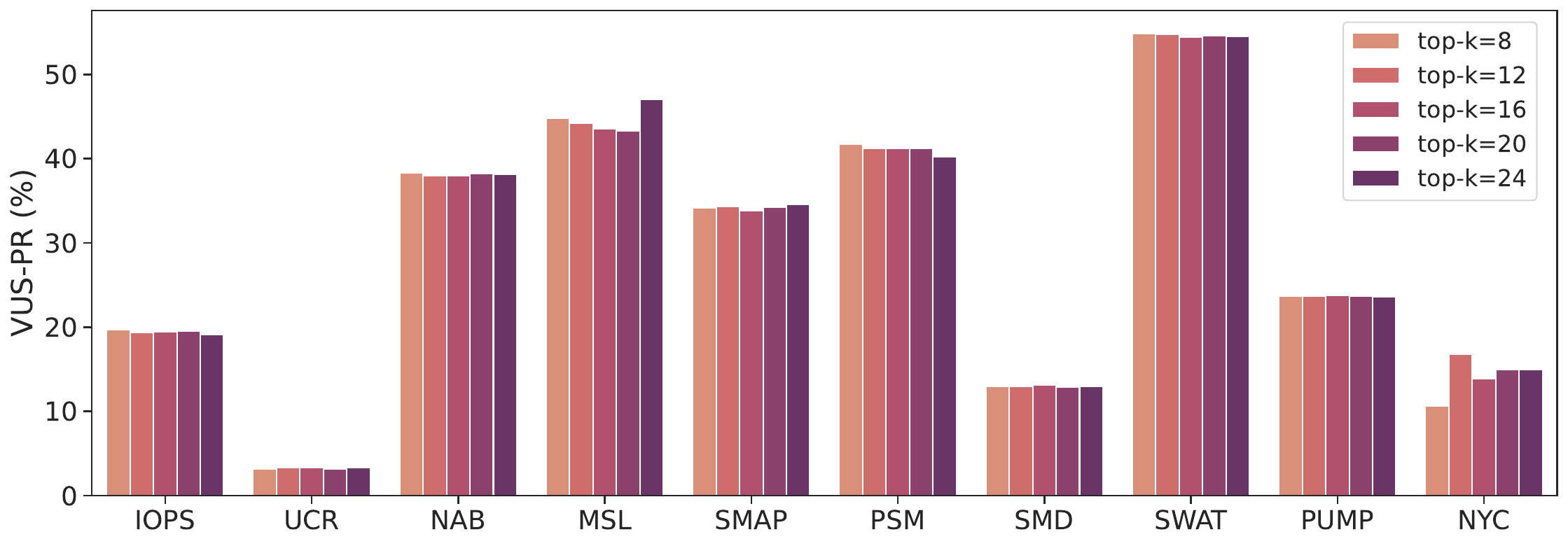}
		\caption{ROSE with NCL on ten TSAD datasets.}
		\label{ROSE_top_k}
	\end{subfigure}
	\vspace{0.01\textwidth}
	
	\begin{subfigure}{1\linewidth}
		\includegraphics[width=1.0\textwidth]{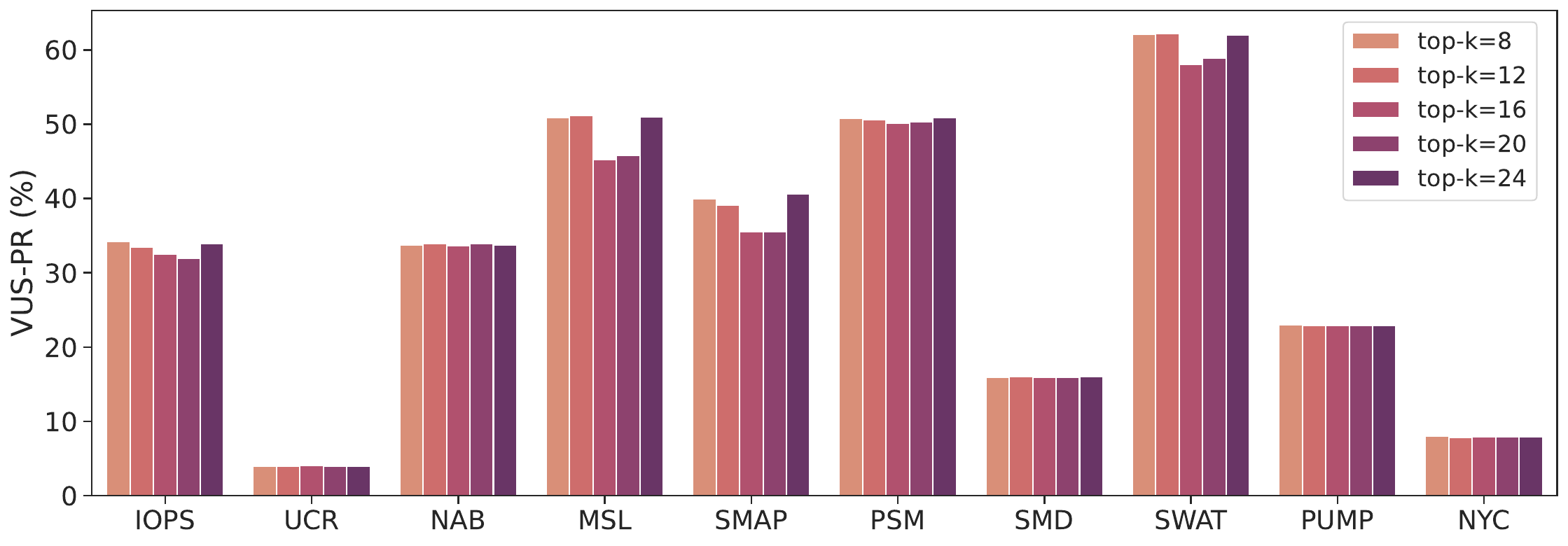}
		\caption{TMM with NCL on ten TSAD datasets.}
		\label{TMM_top_k}
	\end{subfigure}
	\vspace{0.01\textwidth}
	
	\begin{subfigure}{1\linewidth}
		\includegraphics[width=1.0\textwidth]{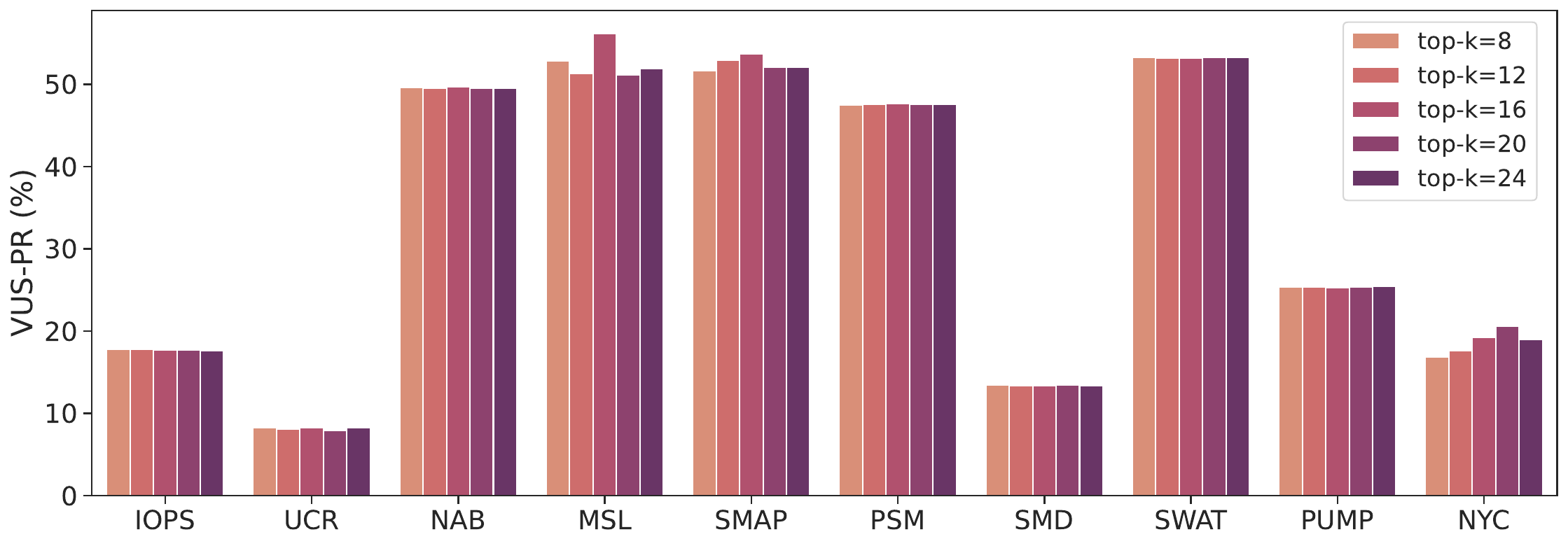}
		\caption{UniTS with NCL on ten TSAD datasets.}
		\label{UniTS_top_k}
	\end{subfigure}
	\caption{VUS-PR performance of three NCL-augmented TSFMs w.r.t. different number of singular vectors.}
	\label{Three_top_k}
\end{figure}

Based on these empirical observations, we recommend using five normal samples as the default configuration for NCL, which provides an excellent balance between sample efficiency and detection performance. This configuration achieves near-optimal performance while minimizing data collection requirements. For deployment scenarios with extremely limited access to normal data, even 1-3 samples can yield substantial improvements over vanilla TSFMs. For critical applications where maximal performance is required and data collection is feasible, using 10-20 samples provides marginal additional gains with minimal additional cost.

\section{Sensitivity Analysis for Singular Vectors}\label{app:top_k}
The number of top singular vectors $k$ retained in NormSpace directly determines the dimensionality of the normality subspace and the number of normality transformation views produced by PatchNT. To understand how this critical hyperparameter affects performance, we conduct comprehensive sensitivity analysis by varying $k \in \{8, 12, 16, 20, 24\}$ across three representative TSFMs.

\cref{Three_top_k} presents the VUS-PR performance changes. This shows that too few singular vectors result in suboptimal performance because the normality subspace becomes overly simplistic and fails to capture the full spectrum of normal variation. For instance, on the NYC dataset, UniTS only achieves VUS-PR of 16.86 using $k=8$, which is substantially lower than the peak performance. Conversely, too many singular vectors also degrade performance because minor singular vectors corresponding to noise or dataset-specific idiosyncrasies are included, diluting the discriminative power of the principal normal directions.
Importantly, different TSFMs on various datasets exhibit slightly different sensitivity patterns, which depend on the complexity of the architecture and the data itself. Simpler univariate datasets like NAB with more homogeneous normal patterns show relatively flat performance curves with less pronounced peaks, suggesting that even moderate values of $k$ capture the essential normal structure. More complex multivariate datasets like SWAT exhibit sharper peaks, indicating that these datasets require a more carefully balanced subspace dimensionality to effectively separate normal and abnormal patterns.

Beyond detection performance, the choice of $k$ also affects computational efficiency. Smaller $k$ values reduce both the number of transformation views computed in PatchNT and the number of cluster centers maintained in ConNC, leading to faster training and inference. The fact that $k=16$ delivers good performance while remaining computationally modest ($16$ parallel transformations per patch) represents an ideal trade-off. Higher $k$ would increase the transformation cost while may not improve performance.

\section{Full Experimental Results of Vanilla and NCL-Augmented TSFMs}\label{all_results}

We report the full experimental results of ten vanilla TSFMs and their NCL-augmented variants on ten TSAD benchmarks in \cref{tab:full_pr} and \ref{tab:full_roc}, evaluated by both PR-based and ROC-based metrics. Overall, NCL demonstrates strong and consistent effectiveness across a wide range of TSFMs and datasets. In the PR evaluation, NCL improves the performance of nearly all backbones in most cases. Similar trends are also observed in ROC metrics, where NCL substantially boosts the anomaly discrimination capability of the underlying TSFMs. Moreover, the paired t-tests show statistically significant improvements for almost all model groups, with very small $p$-values in both PR and ROC results, confirming that the observed improvements are not due to random variation.

To further understand the relationships among TSFMs before and after NCL augmentation, we visualize the pairwise correlation matrices in \cref{fig:PR_correlation} and \ref{fig:ROC_correlation}. In PR metrics, the vanilla TSFMs already exhibit several strong correlations, especially among models with similar forecasting/reconstruction behaviors. However, after applying NCL, the correlations among the augmented TSFMs become even more uniformly high, with many pairwise coefficients approaching 1.0. This suggests that NCL drives different TSFMs toward a more aligned notion of normality, despite their heterogeneous architectures and pretrained representations. In other words, while vanilla TSFMs may capture broad temporal patterns, NCL encourages them to converge to a shared normality-centered representation space. Meanwhile, the correlations between vanilla TSFMs and NCL-augmented TSFMs are noticeably lower than the correlations within the vanilla group, indicating that NCL-augmented TSFMs are recalibrated through explicit normality learning to enhance the anomaly detection performance.
A similar but more revealing trend is observed in ROC correlations.

Together, the full experimental results and the pairwise correlation analyses provide strong evidence for the effectiveness and universality of NCL. Across diverse TSFMs, datasets, and evaluation metrics, NCL consistently improves anomaly detection performance. This demonstrates that NCL is not merely a model-specific enhancement, but a general normality-learning framework that can reliably constrain pre-trained TSFM features toward compact and discriminative normality manifolds.
\newpage
\input{tab/full_PR.tex}
\input{tab/full_ROC.tex}

\begin{figure*}[ht]
	\centering
	\includegraphics[width=0.85\linewidth]{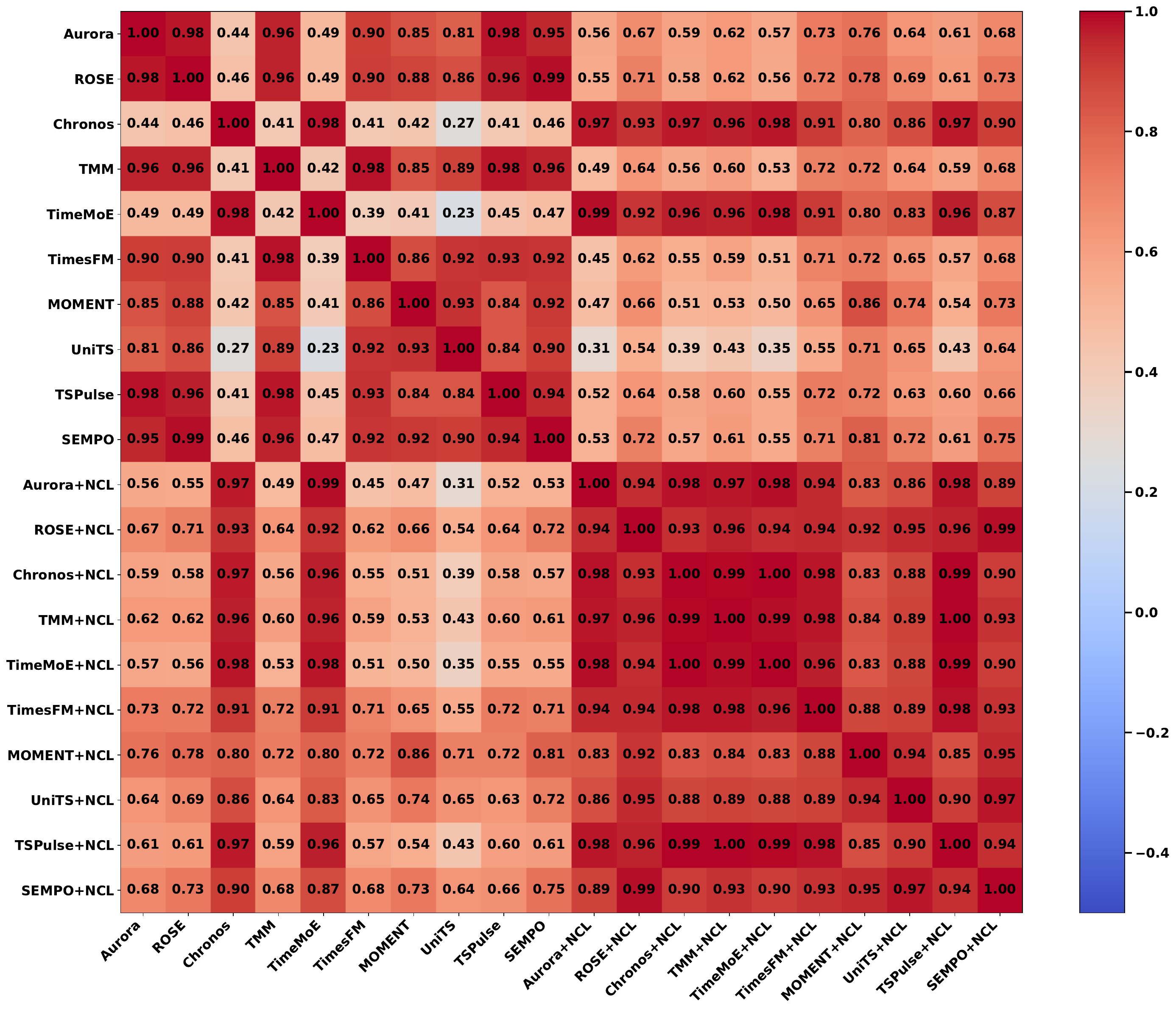}
	\caption{Pairwise correlation in PR metrics.}
	\label{fig:PR_correlation}
\end{figure*}
\begin{figure*}[ht]
	\centering
	\includegraphics[width=0.85\linewidth]{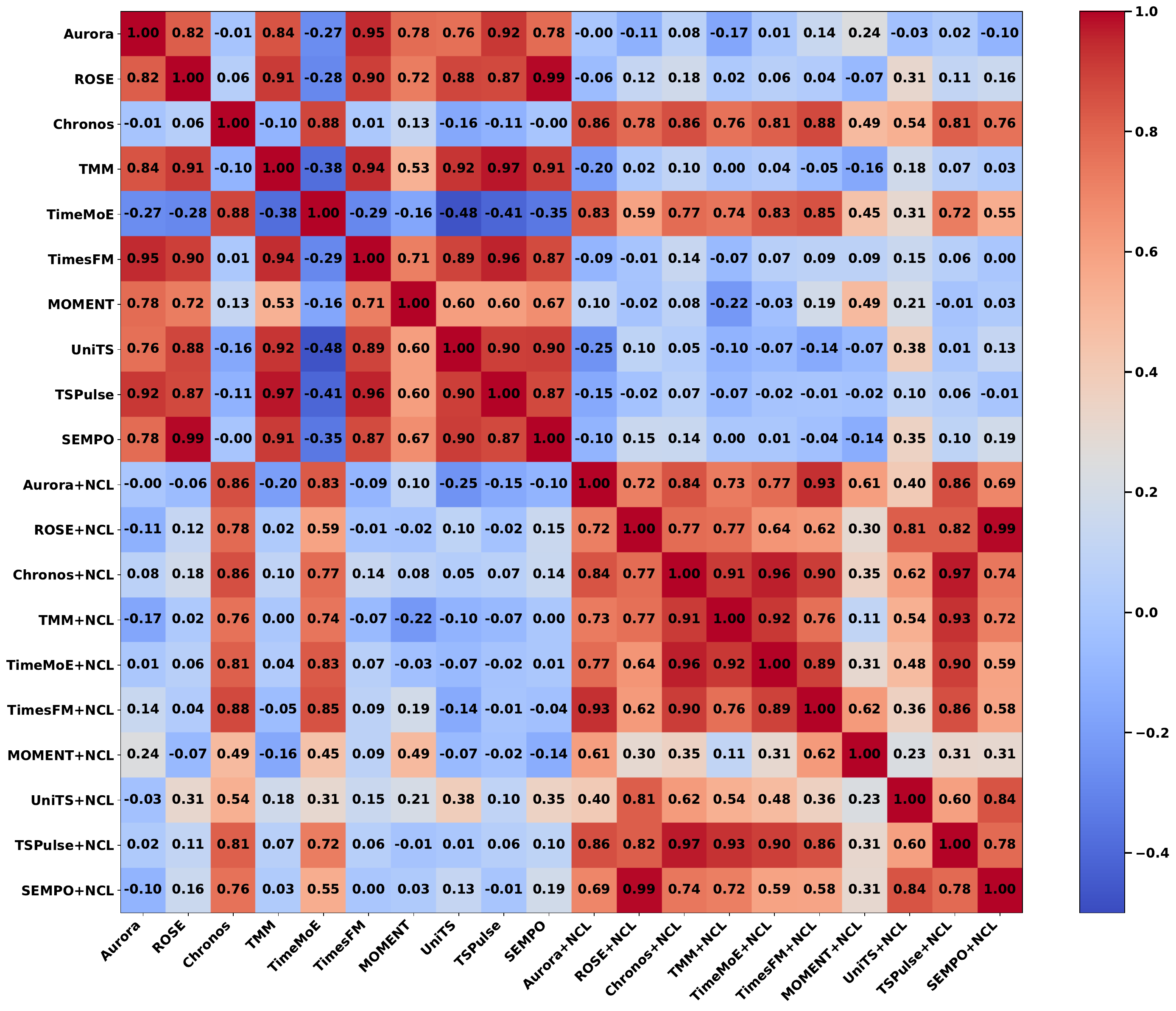}
	\caption{Pairwise correlation in ROC metrics.}
	\label{fig:ROC_correlation}
\end{figure*}

\end{document}

%% file: math_commands.tex
\usepackage{amsmath,amsfonts,bm}

\def\eqref#1{equation~\ref{#1}}

\def\1{\bm{1}}

\DeclareMathAlphabet{\mathsfit}{\encodingdefault}{\sfdefault}{m}{sl}
\SetMathAlphabet{\mathsfit}{bold}{\encodingdefault}{\sfdefault}{bx}{n}



%% file: tab/main_results.tex
\begin{table}[!t]
	\centering
	\caption{VUS-PR results of vanilla and NCL-augmented TSFMs that use forecasting/reconstruction errors to detect anomalies on ten TSAD datasets. $\bm{\triangle}_\textbf{imp}$ indicates improvement ratios gained by NCL.}
	\setlength{\tabcolsep}{3pt}
	\scalebox{0.70}{
	\begin{tabular}{c|ccl|ccl|ccl|ccl}
		\toprule
		\multirow{2.5}{*}{\textbf{Dataset}} & \multicolumn{3}{c|}{\textbf{Aurora}} & \multicolumn{3}{c|}{\textbf{ROSE}} & \multicolumn{3}{c|}{\textbf{Chronos}} & \multicolumn{3}{c}{\textbf{TMM}} \\
		& \textbf{w/o NCL} & \textbf{w/ NCL} & $\bm{\triangle}_\textbf{imp}(\textbf{\%})$ & \textbf{w/o NCL} & \textbf{w/ NCL} & $\bm{\triangle}_\textbf{imp}(\textbf{\%})$ & \textbf{w/o NCL} & \textbf{w/ NCL} & $\bm{\triangle}_\textbf{imp}(\textbf{\%})$ & \textbf{w/o NCL} & \textbf{w/ NCL} & $\bm{\triangle}_\textbf{imp}(\textbf{\%})$ \\
		\midrule
		IOPS  & 0.1333  & 0.2528  & +89.65 & 0.1307  & 0.1936  & +48.13 & 0.2387  & 0.3614  & +51.40 & 0.1877  & 0.3338  & +77.84 \\
        UCR   & 0.0216  & 0.0247  & +14.35 & 0.0242  & 0.0329  & +35.95 & 0.0566  & 0.0691  & +22.08 & 0.0400  & 0.0410  & +2.50 \\
        NAB   & 0.3004  & 0.3401  & +13.22 & 0.3365  & 0.3754  & +11.56 & 0.2766  & 0.3756  & +35.79 & 0.2894  & 0.3368  & +16.38 \\
        MSL   & 0.2982  & 0.3681  & +23.44 & 0.3604  & 0.4488  & +24.53 & 0.3943  & 0.4361  & +10.60 & 0.3653  & 0.4849  & +32.74 \\
        SMAP  & 0.1889  & 0.2823  & +49.44 & 0.2417  & 0.3484  & +44.15 & 0.3106  & 0.3741  & +20.44 & 0.2736  & 0.3686  & +34.72 \\
        PSM   & 0.3592  & 0.4743  & +32.04 & 0.3672  & 0.4107  & +11.85 & 0.3474  & 0.4926  & +41.80 & 0.3611  & 0.5039  & +39.55 \\
        SMD   & 0.1070  & 0.1483  & +38.60 & 0.1025  & 0.1331  & +29.85 & 0.0980  & 0.1462  & +49.18 & 0.1109  & 0.1591  & +43.46 \\
        SWAT  & 0.1265  & 0.6484  & +412.57 & 0.1638  & 0.5437  & +231.93 & 0.5979  & 0.6281  & +5.05 & 0.1185  & 0.6099  & +414.68 \\
        PUMP  & 0.2027  & 0.2247  & +10.85 & 0.2360  & 0.2392  & +1.36 & 0.2097  & 0.2347  & +11.92 & 0.2082  & 0.2297  & +10.33 \\
        NYC   & 0.0341  & 0.0408  & +19.65 & 0.0969  & 0.1431  & +47.68 & 0.0482  & 0.0604  & +25.31 & 0.0619  & 0.0806  & +30.21 \\
		\midrule
		\textbf{Mean}   & 0.1772  & 0.2805  & +58.28 & 0.2060  & 0.2869  & +39.27 & 0.2578  & 0.3178  & +23.29 & 0.2017  & 0.3148  & +56.12 \\
		\midrule
		\multirow{2.5}{*}{\textbf{Dataset}} & \multicolumn{3}{c|}{\textbf{MOMENT}} & \multicolumn{3}{c|}{\textbf{UniTS}} & \multicolumn{3}{c|}{\textbf{TSPulse}} & \multicolumn{3}{c}{\textbf{SEMPO}} \\
		& w/o NCL & w/ NCL & $\bm{\triangle}_\textbf{imp}(\textbf{\%})$ & w/o NCL & w/ NCL & $\bm{\triangle}_\textbf{imp}(\textbf{\%})$ & w/o NCL & w/ NCL & $\bm{\triangle}_\textbf{imp}(\textbf{\%})$ & w/o NCL & w/ NCL & $\bm{\triangle}_\textbf{imp}(\textbf{\%})$ \\
		\midrule
		IOPS  & 0.0657  & 0.1053  & +60.27 & 0.1530  & 0.1861  & +21.63 & 0.1919  & 0.3354  & +74.78 & 0.1335  & 0.1822  & +36.48 \\
        UCR   & 0.0796  & 0.0929  & +16.71 & 0.0808  & 0.0832  & +2.97 & 0.0274  & 0.0320  & +16.79 & 0.0282  & 0.0363  & +28.72 \\
        NAB   & 0.5210  & 0.5565  & +6.81 & 0.4533  & 0.4962  & +9.46 & 0.3201  & 0.3750  & +17.15 & 0.3753  & 0.4084  & +8.82 \\
        MSL   & 0.4726  & 0.5703  & +20.67 & 0.5033  & 0.5657  & +12.40 & 0.3060  & 0.4694  & +53.40 & 0.4223  & 0.5081  & +20.32 \\
        SMAP  & 0.4288  & 0.4716  & +9.98 & 0.5263  & 0.5345  & +1.56 & 0.2251  & 0.3687  & +63.79 & 0.2830  & 0.4227  & +49.36 \\
        PSM   & 0.3682  & 0.5279  & +43.37 & 0.4203  & 0.4735  & +12.66 & 0.3618  & 0.4840  & +33.78 & 0.3624  & 0.4084  & +12.69 \\
        SMD   & 0.0942  & 0.1312  & +39.28 & 0.1296  & 0.1343  & +3.63 & 0.1041  & 0.1246  & +19.69 & 0.1080  & 0.1249  & +15.65 \\
        SWAT  & 0.1895  & 0.5281  & +178.68 & 0.1215  & 0.5346  & +340.00 & 0.1190  & 0.6168  & +418.32 & 0.1702  & 0.5287  & +210.63 \\
        PUMP  & 0.2346  & 0.2764  & +17.82 & 0.2370  & 0.2534  & +6.92 & 0.1963  & 0.2324  & +18.39 & 0.2389  & 0.2512  & +5.15 \\
        NYC   & 0.0318  & 0.0327  & +2.83 & 0.1495  & 0.1850  & +23.75 & 0.0624  & 0.0670  & +7.37 & 0.0862  & 0.1332  & +54.52 \\
		\midrule
		\textbf{Mean}   & 0.2486  & 0.3293  & +32.46 & 0.2775  & 0.3447  & +24.22 & 0.1914  & 0.3105  & +62.23 & 0.2208  & 0.3004  & +36.06\\
		\bottomrule
	\end{tabular}%
	}
	\label{tab:main}%
\end{table}%

%% file: tab/ablation_results.tex
	
	\caption{Ablation study results of NCL and its variants. }
	\label{tab:ablation_studies}
	\begin{threeparttable}
		{
			\color{black}
			\arrayrulecolor{black}
			
			\setlength{\tabcolsep}{5pt}
			\renewcommand{\arraystretch}{1.1}
			
			\resizebox{0.6\columnwidth}{!}{
				\begin{tabular}{l c c c | c c c | c c c}
					\specialrule{0.8pt}{0pt}{2pt}
					\multicolumn{4}{c|}{\textbf{Dataset}} & \multicolumn{3}{c}{\textbf{IOPS}} & \multicolumn{3}{|c}{\textbf{PSM}}\\ \cmidrule{1-10}
					\textbf{Variant} & \textbf{NormSpace} &  \textbf{PatchNT} & \textbf{ConNC} & \textbf{ROSE} & \textbf{TMM} & \textbf{UniTS} & \textbf{ROSE} & \textbf{TMM} & \textbf{UniTS} \\ \midrule
					
					Vanilla & \without  & \without & \without &0.1307  & 0.1877  & 0.1530  & 0.3672  & 0.3611  & 0.4203\\   \midrule
					$\text{NCL}_{\text{N}}$ & \with  & \without & \without &0.1728  &0.2821   &0.1545   &0.3829   &0.4549   &0.4479 \\   \midrule
                    $\text{NCL}_{\text{P}}$ & \without  & \with & \without &0.1429  & 0.2356  & 0.1589  & 0.3761  & 0.4409  & 0.4375\\   \midrule
					$\text{NCL}_{\text{C}}$ & \without  & \without & \with & 0.1796  & 0.2874  & 0.1584  & 0.3789  & 0.4413  & 0.4358 \\   \midrule
					$\text{NCL}_{\text{NP}}$ & \with  & \with & \without &0.1811  &0.3092   &0.1623   &0.3912   &0.4655   &0.4507 \\   \midrule
					$\text{NCL}_{\text{PC}}$ & \without  & \with & \with & 0.1843  & 0.3059  & 0.1668  & 0.3835  & 0.4660  & 0.4512  \\   \midrule
					$\text{NCL}_{\text{NPC}}$ & \with  & \with & \with &\textbf{0.1936} & \textbf{0.3338}& \textbf{0.1861} & \textbf{0.4107} & \textbf{0.5039} & \textbf{0.4735}\\ \specialrule{0.8pt}{2pt}{2pt}
				\end{tabular}
			}
		}
	\end{threeparttable}

%% file: tab/data_statistic.tex
\begin{table}[!t]
	\centering
	\caption{Statistics of ten datasets (Dim: dimensionality).}
	\scalebox{1.0}{
		\begin{tabular}[\linewidth]{llccccc}
			\toprule
			\textbf{Dataset} & \textbf{Domain} & \textbf{Dim} & \makecell{\textbf{Avg Train}\\ \textbf{Length}} & \makecell{\textbf{Avg Test}\\ \textbf{Length}} & \makecell{\textbf{Avg Anomaly}\\ \textbf{Ratio (\%)}} & \makecell{\textbf{Series}\\ \textbf{Count}}\\
			\midrule
			IOPS  & Web & 1  & 10,142.4 & 91,288.3 & 2.15 & 11\\
			UCR   & Multiple & 1 &  18,608.6  & 47,294 & 0.58 & 243\\
			NAB   & Multiple & 1 & 3,557.8  & 3,557 & 9.84 & 45\\
			MSL  & Spacecraft & 1 & 2,269.8  & 2,896.8 & 3.85 & 22\\
			SMAP   & Spacecraft & 1 & 2,608.8  & 7,929.7 & 2.20 & 35\\
			PSM   & Server Machine & 25 & 132,481  & 87,841 & 11.07 & 1\\
			SMD   & Server Machine & 38 & 708,405  & 708,420 & 2.08 & 1\\
			SWAT & Water Treatment & 51 & 495,000  & 449,919 & 5.78 & 1\\
			PUMP & Water Treatment & 44 & 76,901  & 143,401 & 6.54 & 1\\
			NYC   & Transport & 3 & 13,104  & 4,416 & 0.57 & 1\\
			\bottomrule
		\end{tabular}%
	}
	\label{tab:datasets}%
\end{table}%

%% file: tab/all_PR_compare.tex
\begin{table}[htbp]
	\centering
	\caption{PR comparison with deep TSAD methods and one-liner baselines. The best ones are $\textbf{boldfaced}$ and the second best are $\underline{\text{underlined}}$.}
	\scalebox{0.68}{
	\begin{tabular}{c|c|cccccccc|c>{\columncolor{gray!20}}c|c>{\columncolor{gray!20}}c}
		\toprule
		Dataset & Metric & CroAD & ConAD & TsNet & Dlinear & LOF   & Last-1 & Last-3 & Var-100 & TMM   & +NCL   & UniTS & +NCL \\
		\midrule
		\multirow{3}{*}{IOPS} & A-P   & 0.3019 & 0.0226 & 0.1486 & 0.2188 & 0.0574 & 0.2147 & 0.2320 & 0.2331 & 0.2093 & \textbf{0.3393} & 0.1810 & \underline{0.3133} \\
		& R-A-P & \underline{0.2855} & 0.0285 & 0.1447 & 0.2261 & 0.0617 & 0.1764 & 0.2334 & 0.1715 & 0.1886 & \textbf{0.3349} & 0.1524 & 0.1874 \\
		& V-PR  & \underline{0.2827} & 0.0285 & 0.1437 & 0.2231 & 0.0634 & 0.1763 & 0.2282 & 0.1722 & 0.1877 & \textbf{0.3338} & 0.1530 & 0.1861 \\
		\midrule
		\multirow{3}{*}{UCR} & A-P   & 0.0269 & 0.0089 & 0.0404 & 0.0452 & 0.0586 & 0.0727 & 0.0530 & \textbf{0.0777} & 0.0446 & 0.0443 & 0.0735 & \underline{0.0766} \\
		& R-A-P & 0.0268 & 0.0198 & 0.0433 & 0.0430 & 0.0587 & 0.0635 & 0.0412 & \textbf{0.0846} & 0.0348 & 0.0356 & 0.0827 & \underline{0.0843} \\
		& V-PR  & 0.0272 & 0.0198 & 0.0481 & 0.0468 & 0.0586 & 0.0654 & 0.0457 & \underline{0.0817} & 0.0400  & 0.0410 & 0.0808 & \textbf{0.0832} \\
		\midrule
		\multirow{3}{*}{NAB} & A-P   & 0.2775 & 0.1330 & 0.2486 & 0.2402 & 0.1552 & 0.1655 & 0.1730 & \underline{0.4235} & 0.2054 & 0.2684 & 0.3890 & \textbf{0.4398} \\
		& R-A-P & 0.3623 & 0.2376 & 0.3417 & 0.3246 & 0.2795 & 0.2570 & 0.2629 & \underline{0.4929} & 0.2956 & 0.3402 & 0.4528 & \textbf{0.4933} \\
		& V-PR  & 0.3626 & 0.2356 & 0.3392 & 0.3250 & 0.2801 & 0.2565 & 0.2624 & \underline{0.4960} & 0.2894 & 0.3368 & 0.4533 & \textbf{0.4962} \\
		\midrule
		\multirow{3}{*}{MSL} & A-P   & 0.3856 & 0.0697 & 0.3542 & 0.2923 & 0.1310 & 0.1341 & 0.1707 & \textbf{0.4831} & 0.3243 & 0.4374 & 0.4317 & \underline{0.4511} \\
		& R-A-P & 0.4522 & 0.1267 & 0.4312 & 0.3468 & 0.2761 & 0.2603 & 0.2649 & \textbf{0.5769} & 0.3665 & 0.4884 & 0.5169 & \underline{0.5410} \\
		& V-PR  & 0.4464 & 0.1264 & 0.4224 & 0.3424 & 0.2758 & 0.2594 & 0.2646 & \underline{0.5591} & 0.3653 & 0.4849 & 0.5033 & \textbf{0.5657} \\
		\midrule
		\multirow{3}{*}{SMAP} & A-P   & 0.3452 & 0.0299 & 0.4350 & 0.2801 & 0.0950 & 0.0928 & 0.1332 & \textbf{0.4681} & 0.2411 & 0.3411 & \underline{0.4669} & 0.4617 \\
		& R-A-P & 0.3373 & 0.0582 & 0.4853 & 0.2698 & 0.2887 & 0.1992 & 0.2102 & \textbf{0.5457} & 0.2733 & 0.3722 & \underline{0.5385} & 0.5376 \\
		& V-PR  & 0.3366 & 0.0581 & 0.4747 & 0.2693 & 0.2905 & 0.2006 & 0.2122 & \underline{0.5326} & 0.2736 & 0.3686 & 0.5263 & \textbf{0.5345} \\
		\midrule
		\multirow{3}{*}{PSM} & A-P   & 0.4044 & 0.3202 & 0.3551 & 0.3481 & 0.3132 & 0.3320 & 0.3437 & 0.4387 & 0.3618 & \underline{0.5116} & 0.4183 & \textbf{0.5244} \\
		& R-A-P & 0.4019 & 0.3166 & 0.3591 & 0.3500  & 0.3136 & 0.3316 & 0.3426 & 0.4140 & 0.3611 & \textbf{0.5014} & 0.4191 & \underline{0.4747} \\
		& V-PR  & 0.4029 & 0.3166 & 0.3600  & 0.3497 & 0.3139 & 0.3316 & 0.3430 & 0.4159 & 0.3611 & \textbf{0.5039} & 0.4203 & \underline{0.4735} \\
		\midrule
		\multirow{3}{*}{SMD} & A-P   & 0.1612 & 0.0520 & 0.1712 & 0.1307 & 0.0717 & 0.0946 & 0.0938 & \textbf{0.2122} & 0.1134 & 0.1601 & 0.1804 & \underline{0.1898} \\
		& R-A-P & \underline{0.1602} & 0.0558 & \textbf{0.1689} & 0.1351 & 0.0737 & 0.0827 & 0.0892 & 0.1237 & 0.1111 & 0.1593 & 0.1291 & 0.1338 \\
		& V-PR  & 0.1494 & 0.0560 & \textbf{0.1687} & 0.1349 & 0.0737 & 0.0828 & 0.0891 & 0.1242 & 0.1109 & \underline{0.1591} & 0.1296 & 0.1343 \\
		\midrule
		\multirow{3}{*}{SWAT} & A-P   & 0.1614 & 0.1265 & 0.1104 & 0.1967 & 0.1303 & 0.0895 & 0.0828 & 0.1064 & 0.0837 & \underline{0.6787} & 0.0862 & \textbf{0.7237} \\
		& R-A-P & 0.2075 & 0.1534 & 0.1861 & 0.2215 & 0.1587 & 0.1199 & 0.1140 & 0.1690 & 0.1211 & \textbf{0.6363} & 0.1258 & \underline{0.5750} \\
		& V-PR  & 0.2011 & 0.1522 & 0.1758 & 0.2185 & 0.1572 & 0.1178 & 0.1115 & 0.1591 & 0.1185 & \textbf{0.6099} & 0.1215 & \underline{0.5346} \\
		\midrule
		\multirow{3}{*}{PUMP} & A-P   & 0.1514 & 0.1223 & 0.1191 & 0.1223 & 0.1108 & 0.1372 & 0.1309 & 0.1068 & 0.1228 & \underline{0.1971} & 0.1389 & \textbf{0.2046} \\
		& R-A-P & 0.1981 & 0.1645 & 0.2217 & 0.1776 & 0.1555 & 0.1879 & 0.1823 & 0.1652 & 0.2107 & 0.2322 & \underline{0.2415} & \textbf{0.2525} \\
		& V-PR  & 0.1952 & 0.1636 & 0.2156 & 0.1771 & 0.1564 & 0.1882 & 0.1821 & 0.1681 & 0.2082 & 0.2297 & \underline{0.2370} & \textbf{0.2534} \\
		\midrule
		\multirow{3}{*}{NYC} & A-P   & 0.0420 & 0.0179 & 0.0244 & 0.0325 & 0.0228 & 0.0216 & 0.0225 & \underline{0.1513} & 0.0359 & 0.0547 & 0.1447 & \textbf{0.2093} \\
		& R-A-P & 0.0650 & 0.0293 & 0.0489 & 0.0482 & 0.0857 & 0.0595 & 0.0647 & \textbf{0.2614} & 0.0652 & 0.0860 & 0.1496 & \underline{0.1846} \\
		& V-PR  & 0.0622 & 0.0289 & 0.0483 & 0.0478 & 0.0755 & 0.0548 & 0.0577 & \textbf{0.2423} & 0.0619 & 0.0806 & 0.1495 & \underline{0.1850} \\
		\midrule
		\multicolumn{2}{c|}{\textbf{Average Rank}} & 5.13 & 11.23 & 6.27 & 6.93 & 9.07 & 8.93 & 8.63 & 4.00 & 7.47 & \underline{3.77} & 4.40 & \textbf{2.13} \\
		\bottomrule
	\end{tabular}%
	}
	\label{tab:PR}%
\end{table}%

%% file: tab/all_ROC_compare.tex
\begin{table}[htbp]
	\centering
	\caption{ROC comparison with deep TSAD methods and one-liner baselines. The best ones are $\textbf{boldfaced}$ and the second best are $\underline{\text{underlined}}$.}
	\scalebox{0.68}{
	\begin{tabular}{c|c|cccccccc|c>{\columncolor{gray!20}}c|c>{\columncolor{gray!20}}c}
		\toprule
		Dataset & Metric & CroAD & ConAD & TsNet & Dlinear & LOF   & Last-1 & Last-3 & Var-100 & TMM   & +NCL   & UniTS & +NCL \\
		\midrule
		\multirow{3}{*}{IOPS} & A-R   & 0.7448 & 0.4979 & 0.6825 & 0.6935 & 0.5474 & 0.6084 & 0.6393 & 0.7727 & 0.6868 & \textbf{0.8292} & 0.7446 & \underline{0.8155} \\
		& R-A-R & \underline{0.7589} & 0.5396 & 0.7291 & 0.7435 & 0.5585 & 0.6431 & 0.6891 & 0.7417 & 0.7158 & \textbf{0.8437} & 0.7388 & 0.7453 \\
		& V-ROC & \underline{0.7582} & 0.5335 & 0.7237 & 0.7390 & 0.5534 & 0.6400  & 0.6835 & 0.7426 & 0.7117 & \textbf{0.8426} & 0.7375 & 0.7439 \\
		\midrule
		\multirow{3}{*}{UCR} & A-R   & 0.5251 & 0.5008 & 0.5262 & 0.5753 & \underline{0.5827} & 0.5291 & 0.5314 & 0.5387 & 0.5484 & \textbf{0.5851} & 0.5506 & 0.5566 \\
		& R-A-R & 0.6025 & 0.6078 & 0.6030 & 0.6633 & \underline{0.6424} & 0.6177 & 0.6245 & 0.5941 & 0.6379 & \textbf{0.6442} & 0.6055 & 0.6116 \\
		& V-ROC & 0.5984 & 0.5968 & 0.5960 & 0.6356 & \underline{0.6371} & 0.6112 & 0.6179 & 0.5915 & 0.6295 & \textbf{0.6403} & 0.6021 & 0.6083 \\
		\midrule
		\multirow{3}{*}{NAB} & A-R   & 0.6623 & 0.5057 & 0.6235 & 0.6070 & 0.5200  & 0.5251 & 0.5297 & \underline{0.7034} & 0.5530 & 0.5828 & 0.6841 & \textbf{0.7097} \\
		& R-A-R & 0.7372 & 0.6372 & 0.7240 & 0.7051 & 0.5998 & 0.6332 & 0.6393 & \underline{0.7556} & 0.6496 & 0.6598 & 0.7458 & \textbf{0.7625} \\
		& V-ROC & 0.7338 & 0.6260 & 0.7148 & 0.6983 & 0.5958 & 0.6259 & 0.6317 & \underline{0.7553} & 0.6417 & 0.6550 & 0.7434 & \textbf{0.7613} \\
		\midrule
		\multirow{3}{*}{MSL} & A-R   & 0.7532 & 0.5017 & 0.7559 & 0.6975 & 0.5451 & 0.5829 & 0.5964 & \textbf{0.8432} & 0.6386 & 0.7326 & 0.8067 & \underline{0.8216} \\
		& R-A-R & 0.8143 & 0.6212 & 0.8325 & 0.7786 & 0.6150 & 0.6396 & 0.6515 & \textbf{0.8860} & 0.7195 & 0.7737 & 0.8394 & \underline{0.8560} \\
		& V-ROC & 0.8077 & 0.6093 & 0.8236 & 0.7699 & 0.6076 & 0.6334 & 0.6451 & \textbf{0.8677} & 0.7091 & 0.7677 & 0.8333 & \underline{0.8487} \\
		\midrule
		\multirow{3}{*}{SMAP} & A-R   & 0.6618 & 0.5004 & \textbf{0.7914} & 0.6671 & 0.5232 & 0.5193 & 0.5305 & 0.7205 & 0.6070 & 0.6333 & \underline{0.7654} & 0.7603 \\
		& R-A-R & 0.7314 & 0.6236 & \textbf{0.8450} & 0.7344 & 0.5486 & 0.5622 & 0.5812 & 0.7839 & 0.6938 & 0.6962 & \underline{0.8042} & 0.8036 \\
		& V-ROC & 0.7255 & 0.6125 & \textbf{0.8379} & 0.7286 & 0.5475 & 0.5586 & 0.5771 & 0.7796 & 0.6853 & 0.6911 & \underline{0.8000} & 0.7985 \\
		\midrule
		\multirow{3}{*}{PSM} & A-R   & 0.6053 & 0.5752 & 0.5585 & 0.5557 & 0.5290 & 0.5414 & 0.5546 & 0.6317 & 0.5697 & \underline{0.6904} & 0.6387 & \textbf{0.7100} \\
		& R-A-R & 0.5979 & 0.4726 & 0.5610 & 0.5562 & 0.5258 & 0.5412 & 0.5514 & 0.6012 & 0.5691 & \textbf{0.6648} & 0.6174 & \underline{0.6402} \\
		& V-ROC & 0.5983 & 0.4695 & 0.5614 & 0.5525 & 0.5221 & 0.5390 & 0.5501 & 0.6025 & 0.5665 & \textbf{0.6677} & 0.6160 & \underline{0.6381} \\
		\midrule
		\multirow{3}{*}{SMD} & A-R   & \textbf{0.7708} & 0.5773 & 0.7668 & 0.7028 & 0.5696 & 0.6541 & 0.6712 & 0.7573 & 0.6831 & 0.7293 & 0.7590 & \underline{0.7697} \\
		& R-A-R & \textbf{0.7927} & 0.5780 & \underline{0.7855} & 0.7256 & 0.5957 & 0.6684 & 0.6881 & 0.6295 & 0.6987 & 0.7508 & 0.6875 & 0.6696 \\
		& V-ROC & \textbf{0.7905} & 0.5779 & \underline{0.7837} & 0.7242 & 0.5953 & 0.6678 & 0.6867 & 0.6294 & 0.6972 & 0.7489 & 0.6854 & 0.6681 \\
		\midrule
		\multirow{3}{*}{SWAT} & A-R   & 0.6280 & 0.5380 & 0.2933 & 0.6142 & 0.5198 & 0.3530 & 0.2785 & 0.2358 & 0.2506 & \underline{0.8245} & 0.2339 & \textbf{0.8274} \\
		& R-A-R & 0.6818 & 0.5833 & 0.4064 & 0.6563 & 0.5688 & 0.4345 & 0.3733 & 0.3456 & 0.3625 & \textbf{0.8127} & 0.3523 & \underline{0.7533} \\
		& V-ROC & 0.6728 & 0.5797 & 0.3967 & 0.6518 & 0.5649 & 0.4277 & 0.3654 & 0.3325 & 0.3537 & \textbf{0.7933} & 0.3417 & \underline{0.7148} \\
		\midrule
		\multirow{3}{*}{PUMP} & A-R   & 0.6669 & 0.5968 & 0.4481 & 0.5702 & 0.5267 & 0.6082 & 0.6045 & 0.5325 & 0.5245 & \underline{0.7904} & 0.5598 & \textbf{0.7914} \\
		& R-A-R & 0.7335 & 0.6554 & 0.6103 & 0.6574 & 0.6076 & 0.6790 & 0.6739 & 0.6386 & 0.6660 & \textbf{0.7929} & 0.6904 & \underline{0.7904} \\
		& V-ROC & 0.7273 & 0.6483 & 0.5999 & 0.6517 & 0.6028 & 0.6748 & 0.6692 & 0.6356 & 0.6546 & \textbf{0.7912} & 0.6828 & \underline{0.7873} \\
		\midrule
		\multirow{3}{*}{NYC} & A-R   & 0.6663 & 0.4055 & 0.5379 & 0.6383 & 0.5157 & 0.4749 & 0.4924 & 0.7374 & 0.6249 & 0.6824 & \underline{0.8186} & \textbf{0.8620} \\
		& R-A-R & 0.7293 & 0.5039 & 0.6679 & 0.6835 & 0.6834 & 0.6193 & 0.6333 & \textbf{0.8904} & 0.7546 & 0.8142 & 0.8354 & \underline{0.8513} \\
		& V-ROC & 0.7145 & 0.4888 & 0.6559 & 0.6767 & 0.6500  & 0.6021 & 0.6178 & \textbf{0.8634} & 0.7332 & 0.7846 & 0.8264 & \underline{0.8434} \\
		\midrule
		\multicolumn{2}{c|}{\textbf{Average Rank}} & 4.57 & 10.07 & 6.53 & 5.73 & 9.60 & 9.10 & 8.37 & 5.67 & 7.23 & \underline{3.47} & 4.93 & \textbf{2.73} \\
		\bottomrule
	\end{tabular}%
	}
	\label{tab:ROC}%
\end{table}%

%% file: tab/ad_TSFM_compare.tex
\begin{table}[!t]
	\centering
	\caption{VUS-PR comparison with TSAD-specific foundation models. The best ones are $\textbf{boldfaced}$ and the second best are $\underline{\text{underlined}}$.}
	\renewcommand{\arraystretch}{1.2}
	\scalebox{0.85}{
		\begin{tabular}{c|cccccccccc}
			\toprule
			\multirow{2.5}{*}{\textbf{Model}} & \multicolumn{10}{c}{\textbf{Dataset}} \\
			\cmidrule(lr{0pt}){2-11}
			& IOPS & UCR  & NAB & MSL & SMAP & PSM & SMD & SWAT & PUMP & NYC \\
			\midrule
			DADA$_\text{zero}$  & 0.1837 & 0.0511 & 0.4574 & 0.4395 & 0.4165 & 0.4732 & 0.1239 & 0.5988 & 0.2097 & 0.0553 \\
            DADA$_\text{five}$  & 0.1453 & 0.0783 & \underline{0.4692} & 0.4671 & 0.3656 & 0.4549 & 0.1197 & 0.4262 & \textbf{0.2541} & \textbf{0.3744} \\
            \midrule
			TimeRCD$_\text{zero}$  & 0.0691 & \underline{0.0819} & 0.3410 & 0.3328 & 0.3324 & 0.3126 & 0.0691 & 0.5721 & 0.1960 & 0.0690 \\
            TimeRCD$_\text{five}$  & 0.0584 & 0.0786 & 0.3033 & 0.2996 & 0.3056 & 0.3117 & 0.0670 & 0.5356 & 0.1916 & 0.0650 \\
            \midrule
            TimeRadar$_\text{zero}$  & 0.1111 & 0.0467 & 0.3221 & 0.4617 & 0.3207 & 0.4231 & 0.0968 & 0.5802 & 0.2263 & 0.0985 \\
            TimeRadar$_\text{five}$  & 0.1048 & 0.0483 & 0.4424 & 0.4697 & 0.3869 & 0.4259 & 0.1001 & 0.5769 & 0.2211 & 0.1012 \\
			\midrule
            Time-MoE$_\text{zero}$  & 0.2163 & 0.0568  & 0.2875 & 0.3618  & 0.2182  & 0.4017 & 0.0930  & \underline{0.6238}  & 0.1986 & 0.0392 \\
			\quad +NCL  & \textbf{0.3358} & 0.0780 & 0.3712 & 0.4342 & 0.3324 & 0.4575 & 0.1452 & \textbf{0.6384} & 0.2329 & 0.0677 \\
            \midrule
            TimesFM$_\text{zero}$  & 0.1902  & 0.0543  & 0.2804 & 0.3928  & 0.3366  & 0.3497 & 0.1123  & 0.1132  & 0.1792 & 0.0444 \\
			\quad +NCL  & 0.3244 & 0.0600  & 0.3728 & 0.4360 & 0.3862 & \underline{0.4794} & \underline{0.1498} & 0.5117 & 0.2380 & 0.0483 \\
            \midrule
			TMM$_\text{zero}$  & 0.1877  & 0.0400  & 0.2894 & 0.3653  & 0.2736  & 0.3611 & 0.1109  & 0.1185  & 0.2082 & 0.0619 \\
			\quad +NCL  & \underline{0.3338} & 0.0410 & 0.3368 & 0.4849 & 0.3686 & \textbf{0.5039} & \textbf{0.1591} & 0.6099 & 0.2297 & 0.0806 \\
			\midrule
			UniTS$_\text{zero}$  & 0.1530  & 0.0808  & 0.4533 & \underline{0.5033}  & \underline{0.5263}  & 0.4203 & 0.1296  & 0.1215  & 0.2370 & 0.1495 \\
			\quad +NCL  & 0.1861 & \textbf{0.0832} & \textbf{0.4962} & \textbf{0.5657} & \textbf{0.5345} & 0.4735 & 0.1343 & 0.5346 & \underline{0.2534} & \underline{0.1850} \\
			\bottomrule
		\end{tabular}%
	}
	\label{tab:ad_tsfm_compare}%
\end{table}%

%% file: tab/full_PR.tex
\begin{table}[htbp]
	\centering
	\caption{PR results of ten vanilla and NCL-augmented TSFMs that use forecasting/reconstruction errors to detect anomalies on ten TSAD datasets. Meanwhile, we perform the paired t-tests, and $p$-\textbf{values} are reported.}
	\setlength{\tabcolsep}{3pt}
	\scalebox{0.66}{
	\begin{tabular}{c|c|c>{\columncolor{blue!20}}c|c>{\columncolor{orange!20}}c|c>{\columncolor{green!20}}c|c>{\columncolor{pink!40}}c|c>{\columncolor{violet!20}}c}
		\toprule
		Dataset & Metric & Aurora & +NCL   & ROSE  & +NCL   & Chronos & +NCL   & TMM   & +NCL   & TimeMoE & +NCL \\
		\midrule
		\multirow{3}{*}{IOPS} & A-P   & $\text{0.1433}$ & $\text{0.2646}_{\pm\text{0.0148}}$ & $\text{0.1712}$ & $\text{0.2396}_{\pm\text{0.0118}}$ & $\text{0.2649}$ & $\text{0.3654}_{\pm\text{0.0019}}$ & $\text{0.2093}$ & $\text{0.3393}_{\pm\text{0.0144}}$ & $\text{0.2535}$ & $\text{0.3367}_{\pm\text{0.0094}}$ \\
		& R-A-P & $\text{0.1331}$ & $\text{0.2540}_{\pm\text{0.0247}}$ & $\text{0.1300}$ & $\text{0.1935}_{\pm\text{0.0137}}$ & $\text{0.2390}$ & $\text{0.3619}_{\pm\text{0.0017}}$ & $\text{0.1886}$ & $\text{0.3349}_{\pm\text{0.0130}}$ & $\text{0.2171}$ & $\text{0.3360}_{\pm\text{0.0074}}$ \\
		& V-PR  & $\text{0.1333}$ & $\text{0.2528}_{\pm\text{0.0244}}$ & $\text{0.1307}$ & $\text{0.1936}_{\pm\text{0.0138}}$ & $\text{0.2387}$ & $\text{0.3614}_{\pm\text{0.0017}}$ & $\text{0.1877}$ & $\text{0.3338}_{\pm\text{0.0130}}$ & $\text{0.2163}$ & $\text{0.3358}_{\pm\text{0.0077}}$ \\
		\midrule
		\multirow{3}{*}{UCR} & A-P   & $\text{0.0162}$ & $\text{0.0214}_{\pm\text{0.0012}}$ & $\text{0.0200}$ & $\text{0.0332}_{\pm\text{0.0013}}$ & $\text{0.0636}$ & $\text{0.0754}_{\pm\text{0.0041}}$ & $\text{0.0446}$ & $\text{0.0443}_{\pm\text{0.0008}}$ & $\text{0.0661}$ & $\text{0.0830}_{\pm\text{0.0082}}$ \\
		& R-A-P & $\text{0.0215}$ & $\text{0.0246}_{\pm\text{0.0025}}$ & $\text{0.0238}$ & $\text{0.0323}_{\pm\text{0.0008}}$ & $\text{0.0531}$ & $\text{0.0665}_{\pm\text{0.0032}}$ & $\text{0.0348}$ & $\text{0.0356}_{\pm\text{0.0007}}$ & $\text{0.0512}$ & $\text{0.0740}_{\pm\text{0.0074}}$ \\
		& V-PR  & $\text{0.0216}$ & $\text{0.0247}_{\pm\text{0.0025}}$ & $\text{0.0242}$ & $\text{0.0329}_{\pm\text{0.0007}}$ & $\text{0.0566}$ & $\text{0.0691}_{\pm\text{0.0032}}$ & $\text{0.0400}$ & $\text{0.0410}_{\pm\text{0.0007}}$ & $\text{0.0568}$ & $\text{0.0780}_{\pm\text{0.0069}}$ \\
		\midrule
		\multirow{3}{*}{NAB} & A-P   & $\text{0.2070}$ & $\text{0.2631}_{\pm\text{0.0018}}$ & $\text{0.2356}$ & $\text{0.2845}_{\pm\text{0.0054}}$ & $\text{0.1896}$ & $\text{0.2992}_{\pm\text{0.0014}}$ & $\text{0.2054}$ & $\text{0.2684}_{\pm\text{0.0006}}$ & $\text{0.1998}$ & $\text{0.2980}_{\pm\text{0.0061}}$ \\
		& R-A-P & $\text{0.3066}$ & $\text{0.3428}_{\pm\text{0.0051}}$ & $\text{0.3451}$ & $\text{0.3840}_{\pm\text{0.0086}}$ & $\text{0.2775}$ & $\text{0.3764}_{\pm\text{0.0013}}$ & $\text{0.2956}$ & $\text{0.3402}_{\pm\text{0.0015}}$ & $\text{0.2875}$ & $\text{0.3714}_{\pm\text{0.0067}}$ \\
		& V-PR  & $\text{0.3004}$ & $\text{0.3401}_{\pm\text{0.0040}}$ & $\text{0.3365}$ & $\text{0.3754}_{\pm\text{0.0062}}$ & $\text{0.2766}$ & $\text{0.3756}_{\pm\text{0.0012}}$ & $\text{0.2894}$ & $\text{0.3368}_{\pm\text{0.0012}}$ & $\text{0.2875}$ & $\text{0.3712}_{\pm\text{0.0068}}$ \\
		\midrule
		\multirow{3}{*}{MSL} & A-P   & $\text{0.2551}$ & $\text{0.3275}_{\pm\text{0.0267}}$ & $\text{0.3056}$ & $\text{0.3968}_{\pm\text{0.0248}}$ & $\text{0.3161}$ & $\text{0.3684}_{\pm\text{0.0046}}$ & $\text{0.3243}$ & $\text{0.4374}_{\pm\text{0.0254}}$ & $\text{0.2881}$ & $\text{0.3577}_{\pm\text{0.0042}}$ \\
		& R-A-P & $\text{0.2973}$ & $\text{0.3693}_{\pm\text{0.0192}}$ & $\text{0.3616}$ & $\text{0.4506}_{\pm\text{0.0242}}$ & $\text{0.4021}$ & $\text{0.4401}_{\pm\text{0.0038}}$ & $\text{0.3665}$ & $\text{0.4884}_{\pm\text{0.0311}}$ & $\text{0.3690}$ & $\text{0.4420}_{\pm\text{0.0025}}$ \\
		& V-PR  & $\text{0.2982}$ & $\text{0.3681}_{\pm\text{0.0192}}$ & $\text{0.3604}$ & $\text{0.4488}_{\pm\text{0.0214}}$ & $\text{0.3943}$ & $\text{0.4361}_{\pm\text{0.0031}}$ & $\text{0.3653}$ & $\text{0.4849}_{\pm\text{0.0304}}$ & $\text{0.3618}$ & $\text{0.4342}_{\pm\text{0.0012}}$ \\
		\midrule
		\multirow{3}{*}{SMAP} & A-P   & $\text{0.1537}$ & $\text{0.2644}_{\pm\text{0.0113}}$ & $\text{0.1896}$ & $\text{0.2980}_{\pm\text{0.0150}}$ & $\text{0.2363}$ & $\text{0.3480}_{\pm\text{0.0075}}$ & $\text{0.2411}$ & $\text{0.3411}_{\pm\text{0.0306}}$ & $\text{0.1907}$ & $\text{0.3248}_{\pm\text{0.0147}}$ \\
		& R-A-P & $\text{0.1898}$ & $\text{0.2833}_{\pm\text{0.0119}}$ & $\text{0.2430}$ & $\text{0.3512}_{\pm\text{0.0165}}$ & $\text{0.3189}$ & $\text{0.3779}_{\pm\text{0.0185}}$ & $\text{0.2733}$ & $\text{0.3722}_{\pm\text{0.0169}}$ & $\text{0.2182}$ & $\text{0.3318}_{\pm\text{0.0126}}$ \\
		& V-PR  & $\text{0.1889}$ & $\text{0.2823}_{\pm\text{0.0123}}$ & $\text{0.2417}$ & $\text{0.3484}_{\pm\text{0.0162}}$ & $\text{0.3106}$ & $\text{0.3741}_{\pm\text{0.0184}}$ & $\text{0.2736}$ & $\text{0.3686}_{\pm\text{0.0176}}$ & $\text{0.2182}$ & $\text{0.3324}_{\pm\text{0.0118}}$ \\
		\midrule
		\multirow{3}{*}{PSM} & A-P   & $\text{0.3590}$ & $\text{0.4772}_{\pm\text{0.0039}}$ & $\text{0.3703}$ & $\text{0.4181}_{\pm\text{0.0136}}$ & $\text{0.3483}$ & $\text{0.4998}_{\pm\text{0.0022}}$ & $\text{0.3618}$ & $\text{0.5116}_{\pm\text{0.0045}}$ & $\text{0.4097}$ & $\text{0.4810}_{\pm\text{0.0051}}$ \\
		& R-A-P & $\text{0.3594}$ & $\text{0.4737}_{\pm\text{0.0051}}$ & $\text{0.3668}$ & $\text{0.4096}_{\pm\text{0.0143}}$ & $\text{0.3473}$ & $\text{0.4900}_{\pm\text{0.0020}}$ & $\text{0.3611}$ & $\text{0.5014}_{\pm\text{0.0050}}$ & $\text{0.4014}$ & $\text{0.4565}_{\pm\text{0.0031}}$ \\
		& V-PR  & $\text{0.3592}$ & $\text{0.4743}_{\pm\text{0.0050}}$ & $\text{0.3672}$ & $\text{0.4107}_{\pm\text{0.0145}}$ & $\text{0.3474}$ & $\text{0.4926}_{\pm\text{0.0021}}$ & $\text{0.3611}$ & $\text{0.5039}_{\pm\text{0.0051}}$ & $\text{0.4017}$ & $\text{0.4575}_{\pm\text{0.0030}}$ \\
		\midrule
		\multirow{3}{*}{SMD} & A-P   & $\text{0.1074}$ & $\text{0.1506}_{\pm\text{0.0000}}$ & $\text{0.1012}$ & $\text{0.1395}_{\pm\text{0.0073}}$ & $\text{0.1003}$ & $\text{0.1442}_{\pm\text{0.0001}}$ & $\text{0.1134}$ & $\text{0.1601}_{\pm\text{0.0011}}$ & $\text{0.0959}$ & $\text{0.1496}_{\pm\text{0.0004}}$ \\
		& R-A-P & $\text{0.1071}$ & $\text{0.1483}_{\pm\text{0.0004}}$ & $\text{0.1024}$ & $\text{0.1330}_{\pm\text{0.0077}}$ & $\text{0.0980}$ & $\text{0.1463}_{\pm\text{0.0001}}$ & $\text{0.1111}$ & $\text{0.1593}_{\pm\text{0.0012}}$ & $\text{0.0930}$ & $\text{0.1452}_{\pm\text{0.0004}}$ \\
		& V-PR  & $\text{0.1070}$ & $\text{0.1483}_{\pm\text{0.0004}}$ & $\text{0.1025}$ & $\text{0.1331}_{\pm\text{0.0077}}$ & $\text{0.0980}$ & $\text{0.1462}_{\pm\text{0.0001}}$ & $\text{0.1109}$ & $\text{0.1591}_{\pm\text{0.0012}}$ & $\text{0.0930}$ & $\text{0.1452}_{\pm\text{0.0004}}$ \\
		\midrule
		\multirow{3}{*}{SWAT} & A-P   & $\text{0.0879}$ & $\text{0.6872}_{\pm\text{0.0002}}$ & $\text{0.1079}$ & $\text{0.5041}_{\pm\text{0.0039}}$ & $\text{0.7117}$ & $\text{0.7117}_{\pm\text{0.0000}}$ & $\text{0.0837}$ & $\text{0.6787}_{\pm\text{0.0316}}$ & $\text{0.6445}$ & $\text{0.7185}_{\pm\text{0.0003}}$ \\
		& R-A-P & $\text{0.1290}$ & $\text{0.6641}_{\pm\text{0.0079}}$ & $\text{0.1708}$ & $\text{0.5673}_{\pm\text{0.0015}}$ & $\text{0.6209}$ & $\text{0.6210}_{\pm\text{0.0001}}$ & $\text{0.1211}$ & $\text{0.6363}_{\pm\text{0.0291}}$ & $\text{0.6425}$ & $\text{0.6314}_{\pm\text{0.0015}}$ \\
		& V-PR  & $\text{0.1265}$ & $\text{0.6484}_{\pm\text{0.0102}}$ & $\text{0.1638}$ & $\text{0.5437}_{\pm\text{0.0012}}$ & $\text{0.5979}$ & $\text{0.6281}_{\pm\text{0.0002}}$ & $\text{0.1185}$ & $\text{0.6099}_{\pm\text{0.0276}}$ & $\text{0.6238}$ & $\text{0.6384}_{\pm\text{0.0015}}$ \\
		\midrule
		\multirow{3}{*}{PUMP} & A-P   & $\text{0.1203}$ & $\text{0.1881}_{\pm\text{0.0006}}$ & $\text{0.1796}$ & $\text{0.1795}_{\pm\text{0.0005}}$ & $\text{0.1870}$ & $\text{0.2094}_{\pm\text{0.0005}}$ & $\text{0.1228}$ & $\text{0.1971}_{\pm\text{0.0024}}$ & $\text{0.1700}$ & $\text{0.2077}_{\pm\text{0.0008}}$ \\
		& R-A-P & $\text{0.2065}$ & $\text{0.2276}_{\pm\text{0.0025}}$ & $\text{0.2412}$ & $\text{0.2447}_{\pm\text{0.0037}}$ & $\text{0.2092}$ & $\text{0.2360}_{\pm\text{0.0005}}$ & $\text{0.2107}$ & $\text{0.2322}_{\pm\text{0.0006}}$ & $\text{0.1996}$ & $\text{0.2351}_{\pm\text{0.0006}}$ \\
		& V-PR  & $\text{0.2027}$ & $\text{0.2247}_{\pm\text{0.0022}}$ & $\text{0.2360}$ & $\text{0.2392}_{\pm\text{0.0022}}$ & $\text{0.2097}$ & $\text{0.2347}_{\pm\text{0.0005}}$ & $\text{0.2082}$ & $\text{0.2297}_{\pm\text{0.0007}}$ & $\text{0.1986}$ & $\text{0.2329}_{\pm\text{0.0006}}$ \\
		\midrule
		\multirow{3}{*}{NYC} & A-P   & $\text{0.0200}$ & $\text{0.0298}_{\pm\text{0.0008}}$ & $\text{0.0710}$ & $\text{0.0901}_{\pm\text{0.0217}}$ & $\text{0.0285}$ & $\text{0.0422}_{\pm\text{0.0010}}$ & $\text{0.0359}$ & $\text{0.0547}_{\pm\text{0.0009}}$ & $\text{0.0277}$ & $\text{0.0478}_{\pm\text{0.0064}}$ \\
		& R-A-P & $\text{0.0315}$ & $\text{0.0407}_{\pm\text{0.0023}}$ & $\text{0.1034}$ & $\text{0.1479}_{\pm\text{0.0383}}$ & $\text{0.0498}$ & $\text{0.0635}_{\pm\text{0.0002}}$ & $\text{0.0652}$ & $\text{0.0860}_{\pm\text{0.0015}}$ & $\text{0.0399}$ & $\text{0.0685}_{\pm\text{0.0066}}$ \\
		& V-PR  & $\text{0.0341}$ & $\text{0.0408}_{\pm\text{0.0021}}$ & $\text{0.0969}$ & $\text{0.1431}_{\pm\text{0.0148}}$ & $\text{0.0482}$ & $\text{0.0604}_{\pm\text{0.0003}}$ & $\text{0.0619}$ & $\text{0.0806}_{\pm\text{0.0012}}$ & $\text{0.0392}$ & $\text{0.0677}_{\pm\text{0.0071}}$ \\
		\midrule
		\multicolumn{2}{c|}{$p$-\textbf{values}} & \textbf{---} & \textbf{0.0006} & \textbf{---} & \textbf{0.0002} & \textbf{---} & \textbf{1.92e-07} & \textbf{---} & \textbf{0.0001} & \textbf{---} & \textbf{3.74e-09} \\
        \midrule
        Dataset & Metric & TimesFM & +NCL   & MOMENT & +NCL   & UniTS & +NCL   & TSPulse & +NCL   & SEMPO & +NCL \\
		\midrule
		\multirow{3}{*}{IOPS} & A-P   & $\text{0.2180}$ & $\text{0.3242}_{\pm\text{0.0065}}$ & $\text{0.1036}$ & $\text{0.1732}_{\pm\text{0.0009}}$ & $\text{0.1810}$ & $\text{0.3133}_{\pm\text{0.0041}}$ & $\text{0.2038}$ & $\text{0.3494}_{\pm\text{0.0037}}$ & $\text{0.1772}$ & $\text{0.2360}_{\pm\text{0.0077}}$ \\
		& R-A-P & $\text{0.1904}$ & $\text{0.3249}_{\pm\text{0.0066}}$ & $\text{0.0652}$ & $\text{0.1051}_{\pm\text{0.0005}}$ & $\text{0.1524}$ & $\text{0.1874}_{\pm\text{0.0093}}$ & $\text{0.1918}$ & $\text{0.3360}_{\pm\text{0.0025}}$ & $\text{0.1326}$ & $\text{0.1820}_{\pm\text{0.0072}}$ \\
		& V-PR  & $\text{0.1902}$ & $\text{0.3244}_{\pm\text{0.0066}}$ & $\text{0.0657}$ & $\text{0.1053}_{\pm\text{0.0005}}$ & $\text{0.1530}$ & $\text{0.1861}_{\pm\text{0.0094}}$ & $\text{0.1919}$ & $\text{0.3354}_{\pm\text{0.0026}}$ & $\text{0.1335}$ & $\text{0.1822}_{\pm\text{0.0070}}$ \\
		\midrule
		\multirow{3}{*}{UCR} & A-P   & $\text{0.0636}$ & $\text{0.0612}_{\pm\text{0.0027}}$ & $\text{0.0626}$ & $\text{0.0833}_{\pm\text{0.0003}}$ & $\text{0.0735}$ & $\text{0.0766}_{\pm\text{0.0029}}$ & $\text{0.0228}$ & $\text{0.0300}_{\pm\text{0.0012}}$ & $\text{0.0247}$ & $\text{0.0364}_{\pm\text{0.0017}}$ \\
		& R-A-P & $\text{0.0492}$ & $\text{0.0580}_{\pm\text{0.0025}}$ & $\text{0.0792}$ & $\text{0.0939}_{\pm\text{0.0000}}$ & $\text{0.0827}$ & $\text{0.0843}_{\pm\text{0.0029}}$ & $\text{0.0272}$ & $\text{0.0318}_{\pm\text{0.0020}}$ & $\text{0.0278}$ & $\text{0.0357}_{\pm\text{0.0014}}$ \\
		& V-PR  & $\text{0.0543}$ & $\text{0.0600}_{\pm\text{0.0025}}$ & $\text{0.0796}$ & $\text{0.0929}_{\pm\text{0.0000}}$ & $\text{0.0808}$ & $\text{0.0832}_{\pm\text{0.0012}}$ & $\text{0.0274}$ & $\text{0.0320}_{\pm\text{0.0018}}$ & $\text{0.0282}$ & $\text{0.0363}_{\pm\text{0.0014}}$ \\
		\midrule
		\multirow{3}{*}{NAB} & A-P   & $\text{0.1941}$ & $\text{0.2947}_{\pm\text{0.0059}}$ & $\text{0.4245}$ & $\text{0.4654}_{\pm\text{0.0001}}$ & $\text{0.3890}$ & $\text{0.4398}_{\pm\text{0.0070}}$ & $\text{0.2258}$ & $\text{0.2902}_{\pm\text{0.0031}}$ & $\text{0.2614}$ & $\text{0.3067}_{\pm\text{0.0040}}$ \\
		& R-A-P & $\text{0.2808}$ & $\text{0.3731}_{\pm\text{0.0055}}$ & $\text{0.5296}$ & $\text{0.5635}_{\pm\text{0.0004}}$ & $\text{0.4528}$ & $\text{0.4933}_{\pm\text{0.0033}}$ & $\text{0.3270}$ & $\text{0.3818}_{\pm\text{0.0048}}$ & $\text{0.3876}$ & $\text{0.4187}_{\pm\text{0.0039}}$ \\
		& V-PR  & $\text{0.2804}$ & $\text{0.3728}_{\pm\text{0.0051}}$ & $\text{0.5210}$ & $\text{0.5565}_{\pm\text{0.0003}}$ & $\text{0.4533}$ & $\text{0.4962}_{\pm\text{0.0036}}$ & $\text{0.3201}$ & $\text{0.3750}_{\pm\text{0.0043}}$ & $\text{0.3753}$ & $\text{0.4084}_{\pm\text{0.0037}}$ \\
		\midrule
		\multirow{3}{*}{MSL} & A-P   & $\text{0.3122}$ & $\text{0.3715}_{\pm\text{0.0141}}$ & $\text{0.3307}$ & $\text{0.4593}_{\pm\text{0.0037}}$ & $\text{0.4317}$ & $\text{0.4511}_{\pm\text{0.0257}}$ & $\text{0.2582}$ & $\text{0.4197}_{\pm\text{0.0070}}$ & $\text{0.3882}$ & $\text{0.4458}_{\pm\text{0.0149}}$ \\
		& R-A-P & $\text{0.3975}$ & $\text{0.4386}_{\pm\text{0.0209}}$ & $\text{0.4807}$ & $\text{0.5881}_{\pm\text{0.0022}}$ & $\text{0.5169}$ & $\text{0.5410}_{\pm\text{0.0321}}$ & $\text{0.3064}$ & $\text{0.4720}_{\pm\text{0.0089}}$ & $\text{0.4199}$ & $\text{0.5123}_{\pm\text{0.0132}}$ \\
		& V-PR  & $\text{0.3928}$ & $\text{0.4360}_{\pm\text{0.0202}}$ & $\text{0.4726}$ & $\text{0.5703}_{\pm\text{0.0028}}$ & $\text{0.5033}$ & $\text{0.5657}_{\pm\text{0.0287}}$ & $\text{0.3060}$ & $\text{0.4694}_{\pm\text{0.0072}}$ & $\text{0.4223}$ & $\text{0.5081}_{\pm\text{0.0135}}$ \\
		\midrule
		\multirow{3}{*}{SMAP} & A-P   & $\text{0.2824}$ & $\text{0.3348}_{\pm\text{0.0087}}$ & $\text{0.3485}$ & $\text{0.3827}_{\pm\text{0.0021}}$ & $\text{0.4669}$ & $\text{0.4617}_{\pm\text{0.0171}}$ & $\text{0.1933}$ & $\text{0.3646}_{\pm\text{0.0042}}$ & $\text{0.2225}$ & $\text{0.3885}_{\pm\text{0.0066}}$ \\
		& R-A-P & $\text{0.3441}$ & $\text{0.3935}_{\pm\text{0.0073}}$ & $\text{0.4326}$ & $\text{0.4756}_{\pm\text{0.0012}}$ & $\text{0.5385}$ & $\text{0.5376}_{\pm\text{0.0022}}$ & $\text{0.2253}$ & $\text{0.3699}_{\pm\text{0.0041}}$ & $\text{0.2867}$ & $\text{0.4276}_{\pm\text{0.0095}}$ \\
		& V-PR  & $\text{0.3366}$ & $\text{0.3862}_{\pm\text{0.0064}}$ & $\text{0.4288}$ & $\text{0.4716}_{\pm\text{0.0013}}$ & $\text{0.5263}$ & $\text{0.5345}_{\pm\text{0.0046}}$ & $\text{0.2251}$ & $\text{0.3687}_{\pm\text{0.0048}}$ & $\text{0.2830}$ & $\text{0.4227}_{\pm\text{0.0089}}$ \\
		\midrule
		\multirow{3}{*}{PSM} & A-P   & $\text{0.3500}$ & $\text{0.4840}_{\pm\text{0.0113}}$ & $\text{0.3735}$ & $\text{0.4751}_{\pm\text{0.0006}}$ & $\text{0.4183}$ & $\text{0.5244}_{\pm\text{0.0054}}$ & $\text{0.3608}$ & $\text{0.4875}_{\pm\text{0.0015}}$ & $\text{0.3670}$ & $\text{0.4162}_{\pm\text{0.0128}}$ \\
		& R-A-P & $\text{0.3496}$ & $\text{0.4779}_{\pm\text{0.0113}}$ & $\text{0.3671}$ & $\text{0.4622}_{\pm\text{0.0006}}$ & $\text{0.4191}$ & $\text{0.4747}_{\pm\text{0.0050}}$ & $\text{0.3612}$ & $\text{0.4829}_{\pm\text{0.0019}}$ & $\text{0.3618}$ & $\text{0.4076}_{\pm\text{0.0124}}$ \\
		& V-PR  & $\text{0.3497}$ & $\text{0.4794}_{\pm\text{0.0110}}$ & $\text{0.3682}$ & $\text{0.5279}_{\pm\text{0.0006}}$ & $\text{0.4203}$ & $\text{0.4735}_{\pm\text{0.0050}}$ & $\text{0.3618}$ & $\text{0.4840}_{\pm\text{0.0024}}$ & $\text{0.3624}$ & $\text{0.4084}_{\pm\text{0.0123}}$ \\
		\midrule
		\multirow{3}{*}{SMD} & A-P   & $\text{0.1122}$ & $\text{0.1391}_{\pm\text{0.0016}}$ & $\text{0.1402}$ & $\text{0.1878}_{\pm\text{0.0007}}$ & $\text{0.1804}$ & $\text{0.1898}_{\pm\text{0.0006}}$ & $\text{0.1016}$ & $\text{0.1259}_{\pm\text{0.0003}}$ & $\text{0.1060}$ & $\text{0.1295}_{\pm\text{0.0016}}$ \\
		& R-A-P & $\text{0.1124}$ & $\text{0.1496}_{\pm\text{0.0017}}$ & $\text{0.0938}$ & $\text{0.1302}_{\pm\text{0.0005}}$ & $\text{0.1291}$ & $\text{0.1338}_{\pm\text{0.0005}}$ & $\text{0.1040}$ & $\text{0.1244}_{\pm\text{0.0001}}$ & $\text{0.1077}$ & $\text{0.1247}_{\pm\text{0.0017}}$ \\
		& V-PR  & $\text{0.1123}$ & $\text{0.1498}_{\pm\text{0.0017}}$ & $\text{0.0942}$ & $\text{0.1312}_{\pm\text{0.0004}}$ & $\text{0.1296}$ & $\text{0.1343}_{\pm\text{0.0004}}$ & $\text{0.1041}$ & $\text{0.1246}_{\pm\text{0.0001}}$ & $\text{0.1080}$ & $\text{0.1249}_{\pm\text{0.0017}}$ \\
		\midrule
		\multirow{3}{*}{SWAT} & A-P   & $\text{0.0822}$ & $\text{0.5154}_{\pm\text{0.0660}}$ & $\text{0.1403}$ & $\text{0.5594}_{\pm\text{0.0006}}$ & $\text{0.0862}$ & $\text{0.7237}_{\pm\text{0.0020}}$ & $\text{0.0855}$ & $\text{0.7087}_{\pm\text{0.0006}}$ & $\text{0.1173}$ & $\text{0.5375}_{\pm\text{0.0042}}$ \\
		& R-A-P & $\text{0.1156}$ & $\text{0.5200}_{\pm\text{0.0610}}$ & $\text{0.2032}$ & $\text{0.5571}_{\pm\text{0.0006}}$ & $\text{0.1258}$ & $\text{0.5750}_{\pm\text{0.0092}}$ & $\text{0.1217}$ & $\text{0.6473}_{\pm\text{0.0017}}$ & $\text{0.1783}$ & $\text{0.5559}_{\pm\text{0.0065}}$ \\
		& V-PR  & $\text{0.1132}$ & $\text{0.5117}_{\pm\text{0.0592}}$ & $\text{0.1895}$ & $\text{0.5281}_{\pm\text{0.0006}}$ & $\text{0.1215}$ & $\text{0.5346}_{\pm\text{0.0078}}$ & $\text{0.1190}$ & $\text{0.6168}_{\pm\text{0.0009}}$ & $\text{0.1702}$ & $\text{0.5287}_{\pm\text{0.0054}}$ \\
		\midrule
		\multirow{3}{*}{PUMP} & A-P   & $\text{0.1201}$ & $\text{0.2111}_{\pm\text{0.0024}}$ & $\text{0.1803}$ & $\text{0.1399}_{\pm\text{0.0002}}$ & $\text{0.1389}$ & $\text{0.2046}_{\pm\text{0.0002}}$ & $\text{0.1140}$ & $\text{0.1898}_{\pm\text{0.0003}}$ & $\text{0.1753}$ & $\text{0.1768}_{\pm\text{0.0013}}$ \\
		& R-A-P & $\text{0.1800}$ & $\text{0.2400}_{\pm\text{0.0046}}$ & $\text{0.2402}$ & $\text{0.2872}_{\pm\text{0.0009}}$ & $\text{0.2415}$ & $\text{0.2525}_{\pm\text{0.0009}}$ & $\text{0.1995}$ & $\text{0.2352}_{\pm\text{0.0004}}$ & $\text{0.2446}$ & $\text{0.2457}_{\pm\text{0.0012}}$ \\
		& V-PR  & $\text{0.1792}$ & $\text{0.2380}_{\pm\text{0.0042}}$ & $\text{0.2346}$ & $\text{0.2764}_{\pm\text{0.0011}}$ & $\text{0.2370}$ & $\text{0.2534}_{\pm\text{0.0011}}$ & $\text{0.1963}$ & $\text{0.2324}_{\pm\text{0.0002}}$ & $\text{0.2389}$ & $\text{0.2512}_{\pm\text{0.0012}}$ \\
		\midrule
		\multirow{3}{*}{NYC} & A-P   & $\text{0.0277}$ & $\text{0.0286}_{\pm\text{0.0023}}$ & $\text{0.0158}$ & $\text{0.0158}_{\pm\text{0.0001}}$ & $\text{0.1447}$ & $\text{0.2093}_{\pm\text{0.0052}}$ & $\text{0.0334}$ & $\text{0.0422}_{\pm\text{0.0008}}$ & $\text{0.0803}$ & $\text{0.0943}_{\pm\text{0.0296}}$ \\
		& R-A-P & $\text{0.0454}$ & $\text{0.0494}_{\pm\text{0.0027}}$ & $\text{0.0285}$ & $\text{0.0285}_{\pm\text{0.0001}}$ & $\text{0.1496}$ & $\text{0.1846}_{\pm\text{0.0060}}$ & $\text{0.0645}$ & $\text{0.0694}_{\pm\text{0.0041}}$ & $\text{0.0887}$ & $\text{0.1431}_{\pm\text{0.0333}}$ \\
		& V-PR  & $\text{0.0444}$ & $\text{0.0483}_{\pm\text{0.0022}}$ & $\text{0.0318}$ & $\text{0.0327}_{\pm\text{0.0001}}$ & $\text{0.1495}$ & $\text{0.1850}_{\pm\text{0.0063}}$ & $\text{0.0624}$ & $\text{0.0670}_{\pm\text{0.0043}}$ & $\text{0.0862}$ & $\text{0.1332}_{\pm\text{0.0232}}$ \\
		\midrule
		\multicolumn{2}{c|}{$p$-\textbf{values}} & \textbf{---} & \textbf{7.43e-05} & \textbf{---} & \textbf{0.0003} & \textbf{---} & \textbf{0.0062} & \textbf{---} & \textbf{0.0001} & \textbf{---} & \textbf{0.0004} \\
		\bottomrule
	\end{tabular}%
	}
	\label{tab:full_pr}%
\end{table}%

%% file: tab/full_ROC.tex
\begin{table}[htbp]
	\centering
	\caption{ROC results of ten vanilla and NCL-augmented TSFMs that use forecasting/reconstruction errors to detect anomalies on ten TSAD datasets. Meanwhile, we perform the paired t-tests, and $p$-\textbf{values} are reported.}
	\setlength{\tabcolsep}{3pt}
	\scalebox{0.66}{
	\begin{tabular}{c|c|c>{\columncolor{blue!20}}c|c>{\columncolor{orange!20}}c|c>{\columncolor{green!20}}c|c>{\columncolor{pink!40}}c|c>{\columncolor{violet!20}}c}
		\toprule
		Dataset & Metric & Aurora & NCL   & ROSE  & NCL   & Chronos & NCL   & TMM   & NCL   & TimeMoE & NCL \\
		\midrule
		\multirow{3}{*}{IOPS} & A-R   & $\text{0.6676}$ & $\text{0.7848}_{\pm\text{0.0237}}$ & $\text{0.6770}$ & $\text{0.7845}_{\pm\text{0.0185}}$ & $\text{0.7006}$ & $\text{0.8398}_{\pm\text{0.0005}}$ & $\text{0.6868}$ & $\text{0.8292}_{\pm\text{0.0085}}$ & $\text{0.6908}$ & $\text{0.8301}_{\pm\text{0.0237}}$ \\
		& R-A-R & $\text{0.6943}$ & $\text{0.8016}_{\pm\text{0.0219}}$ & $\text{0.6866}$ & $\text{0.7770}_{\pm\text{0.0125}}$ & $\text{0.7359}$ & $\text{0.8636}_{\pm\text{0.0004}}$ & $\text{0.7158}$ & $\text{0.8437}_{\pm\text{0.0080}}$ & $\text{0.7262}$ & $\text{0.8628}_{\pm\text{0.0202}}$ \\
		& V-ROC & $\text{0.6922}$ & $\text{0.8006}_{\pm\text{0.0221}}$ & $\text{0.6850}$ & $\text{0.7740}_{\pm\text{0.0132}}$ & $\text{0.7318}$ & $\text{0.8621}_{\pm\text{0.0003}}$ & $\text{0.7117}$ & $\text{0.8426}_{\pm\text{0.0080}}$ & $\text{0.7224}$ & $\text{0.8603}_{\pm\text{0.0209}}$ \\
		\midrule
		\multirow{3}{*}{UCR} & A-R   & $\text{0.5738}$ & $\text{0.5849}_{\pm\text{0.0056}}$ & $\text{0.5394}$ & $\text{0.5485}_{\pm\text{0.0079}}$ & $\text{0.5567}$ & $\text{0.5916}_{\pm\text{0.0042}}$ & $\text{0.5484}$ & $\text{0.5851}_{\pm\text{0.0018}}$ & $\text{0.5453}$ & $\text{0.5799}_{\pm\text{0.0143}}$ \\
		& R-A-R & $\text{0.6610}$ & $\text{0.6486}_{\pm\text{0.0067}}$ & $\text{0.6317}$ & $\text{0.6244}_{\pm\text{0.0081}}$ & $\text{0.6471}$ & $\text{0.6569}_{\pm\text{0.0037}}$ & $\text{0.6379}$ & $\text{0.6442}_{\pm\text{0.0013}}$ & $\text{0.6405}$ & $\text{0.6568}_{\pm\text{0.0096}}$ \\
		& V-ROC & $\text{0.6530}$ & $\text{0.6445}_{\pm\text{0.0070}}$ & $\text{0.6236}$ & $\text{0.6186}_{\pm\text{0.0080}}$ & $\text{0.6389}$ & $\text{0.6530}_{\pm\text{0.0037}}$ & $\text{0.6295}$ & $\text{0.6403}_{\pm\text{0.0012}}$ & $\text{0.6316}$ & $\text{0.6499}_{\pm\text{0.0105}}$ \\
		\midrule
		\multirow{3}{*}{NAB} & A-R   & $\text{0.5689}$ & $\text{0.6139}_{\pm\text{0.0075}}$ & $\text{0.5741}$ & $\text{0.6090}_{\pm\text{0.0084}}$ & $\text{0.5558}$ & $\text{0.6328}_{\pm\text{0.0030}}$ & $\text{0.5530}$ & $\text{0.5828}_{\pm\text{0.0017}}$ & $\text{0.5674}$ & $\text{0.6264}_{\pm\text{0.0054}}$ \\
		& R-A-R & $\text{0.6678}$ & $\text{0.6926}_{\pm\text{0.0051}}$ & $\text{0.6821}$ & $\text{0.6929}_{\pm\text{0.0116}}$ & $\text{0.6662}$ & $\text{0.7164}_{\pm\text{0.0029}}$ & $\text{0.6496}$ & $\text{0.6598}_{\pm\text{0.0028}}$ & $\text{0.6801}$ & $\text{0.7153}_{\pm\text{0.0061}}$ \\
		& V-ROC & $\text{0.6590}$ & $\text{0.6865}_{\pm\text{0.0056}}$ & $\text{0.6729}$ & $\text{0.6859}_{\pm\text{0.0102}}$ & $\text{0.6579}$ & $\text{0.7119}_{\pm\text{0.0028}}$ & $\text{0.6417}$ & $\text{0.6550}_{\pm\text{0.0023}}$ & $\text{0.6724}$ & $\text{0.7096}_{\pm\text{0.0058}}$ \\
		\midrule
		\multirow{3}{*}{MSL} & A-R   & $\text{0.6722}$ & $\text{0.7324}_{\pm\text{0.0130}}$ & $\text{0.6953}$ & $\text{0.7823}_{\pm\text{0.0060}}$ & $\text{0.7214}$ & $\text{0.7745}_{\pm\text{0.0012}}$ & $\text{0.6386}$ & $\text{0.7326}_{\pm\text{0.0290}}$ & $\text{0.6511}$ & $\text{0.7578}_{\pm\text{0.0053}}$ \\
		& R-A-R & $\text{0.7565}$ & $\text{0.7821}_{\pm\text{0.0067}}$ & $\text{0.7641}$ & $\text{0.8244}_{\pm\text{0.0032}}$ & $\text{0.8108}$ & $\text{0.8256}_{\pm\text{0.0016}}$ & $\text{0.7195}$ & $\text{0.7737}_{\pm\text{0.0297}}$ & $\text{0.7469}$ & $\text{0.8150}_{\pm\text{0.0117}}$ \\
		& V-ROC & $\text{0.7486}$ & $\text{0.7789}_{\pm\text{0.0086}}$ & $\text{0.7588}$ & $\text{0.8212}_{\pm\text{0.0029}}$ & $\text{0.7990}$ & $\text{0.8209}_{\pm\text{0.0013}}$ & $\text{0.7091}$ & $\text{0.7677}_{\pm\text{0.0298}}$ & $\text{0.7369}$ & $\text{0.8078}_{\pm\text{0.0083}}$ \\
		\midrule
		\multirow{3}{*}{SMAP} & A-R   & $\text{0.6556}$ & $\text{0.7304}_{\pm\text{0.0027}}$ & $\text{0.7040}$ & $\text{0.7626}_{\pm\text{0.0042}}$ & $\text{0.6211}$ & $\text{0.6878}_{\pm\text{0.0141}}$ & $\text{0.6070}$ & $\text{0.6333}_{\pm\text{0.0327}}$ & $\text{0.5424}$ & $\text{0.6124}_{\pm\text{0.0138}}$ \\
		& R-A-R & $\text{0.7699}$ & $\text{0.7866}_{\pm\text{0.0020}}$ & $\text{0.7777}$ & $\text{0.8132}_{\pm\text{0.0018}}$ & $\text{0.7269}$ & $\text{0.7430}_{\pm\text{0.0138}}$ & $\text{0.6938}$ & $\text{0.6962}_{\pm\text{0.0212}}$ & $\text{0.6368}$ & $\text{0.6756}_{\pm\text{0.0120}}$ \\
		& V-ROC & $\text{0.7594}$ & $\text{0.7822}_{\pm\text{0.0014}}$ & $\text{0.7715}$ & $\text{0.8099}_{\pm\text{0.0028}}$ & $\text{0.7166}$ & $\text{0.7399}_{\pm\text{0.0141}}$ & $\text{0.6853}$ & $\text{0.6911}_{\pm\text{0.0216}}$ & $\text{0.6280}$ & $\text{0.6709}_{\pm\text{0.0128}}$ \\
		\midrule
		\multirow{3}{*}{PSM} & A-R   & $\text{0.5807}$ & $\text{0.6778}_{\pm\text{0.0042}}$ & $\text{0.5818}$ & $\text{0.6341}_{\pm\text{0.0086}}$ & $\text{0.5602}$ & $\text{0.6838}_{\pm\text{0.0021}}$ & $\text{0.5697}$ & $\text{0.6904}_{\pm\text{0.0069}}$ & $\text{0.6047}$ & $\text{0.6816}_{\pm\text{0.0030}}$ \\
		& R-A-R & $\text{0.5748}$ & $\text{0.6616}_{\pm\text{0.0055}}$ & $\text{0.5772}$ & $\text{0.6157}_{\pm\text{0.0103}}$ & $\text{0.5576}$ & $\text{0.6597}_{\pm\text{0.0016}}$ & $\text{0.5691}$ & $\text{0.6648}_{\pm\text{0.0081}}$ & $\text{0.5943}$ & $\text{0.6591}_{\pm\text{0.0024}}$ \\
		& V-ROC & $\text{0.5730}$ & $\text{0.6605}_{\pm\text{0.0050}}$ & $\text{0.5762}$ & $\text{0.6174}_{\pm\text{0.0107}}$ & $\text{0.5566}$ & $\text{0.6611}_{\pm\text{0.0019}}$ & $\text{0.5665}$ & $\text{0.6677}_{\pm\text{0.0085}}$ & $\text{0.5934}$ & $\text{0.6614}_{\pm\text{0.0024}}$ \\
		\midrule
		\multirow{3}{*}{SMD} & A-R   & $\text{0.7001}$ & $\text{0.7420}_{\pm\text{0.0004}}$ & $\text{0.6986}$ & $\text{0.7326}_{\pm\text{0.0076}}$ & $\text{0.6660}$ & $\text{0.7278}_{\pm\text{0.0005}}$ & $\text{0.6831}$ & $\text{0.7293}_{\pm\text{0.0073}}$ & $\text{0.6556}$ & $\text{0.7302}_{\pm\text{0.0008}}$ \\
		& R-A-R & $\text{0.7149}$ & $\text{0.7595}_{\pm\text{0.0006}}$ & $\text{0.7068}$ & $\text{0.7373}_{\pm\text{0.0084}}$ & $\text{0.6833}$ & $\text{0.7496}_{\pm\text{0.0002}}$ & $\text{0.6987}$ & $\text{0.7508}_{\pm\text{0.0059}}$ & $\text{0.6731}$ & $\text{0.7494}_{\pm\text{0.0004}}$ \\
		& V-ROC & $\text{0.7136}$ & $\text{0.7575}_{\pm\text{0.0005}}$ & $\text{0.7060}$ & $\text{0.7361}_{\pm\text{0.0083}}$ & $\text{0.6825}$ & $\text{0.7474}_{\pm\text{0.0003}}$ & $\text{0.6972}$ & $\text{0.7489}_{\pm\text{0.0058}}$ & $\text{0.6721}$ & $\text{0.7474}_{\pm\text{0.0004}}$ \\
		\midrule
		\multirow{3}{*}{SWAT} & A-R   & $\text{0.2493}$ & $\text{0.8315}_{\pm\text{0.0003}}$ & $\text{0.2709}$ & $\text{0.8286}_{\pm\text{0.0001}}$ & $\text{0.8117}$ & $\text{0.8117}_{\pm\text{0.0001}}$ & $\text{0.2506}$ & $\text{0.8245}_{\pm\text{0.0026}}$ & $\text{0.8194}$ & $\text{0.8192}_{\pm\text{0.0006}}$ \\
		& R-A-R & $\text{0.3638}$ & $\text{0.8374}_{\pm\text{0.0060}}$ & $\text{0.3866}$ & $\text{0.8513}_{\pm\text{0.0023}}$ & $\text{0.7873}$ & $\text{0.7874}_{\pm\text{0.0002}}$ & $\text{0.3625}$ & $\text{0.8127}_{\pm\text{0.0093}}$ & $\text{0.8252}$ & $\text{0.7860}_{\pm\text{0.0011}}$ \\
		& V-ROC & $\text{0.3547}$ & $\text{0.8265}_{\pm\text{0.0064}}$ & $\text{0.3752}$ & $\text{0.8342}_{\pm\text{0.0018}}$ & $\text{0.7670}$ & $\text{0.7670}_{\pm\text{0.0001}}$ & $\text{0.3537}$ & $\text{0.7933}_{\pm\text{0.0082}}$ & $\text{0.8124}$ & $\text{0.7612}_{\pm\text{0.0013}}$ \\
		\midrule
		\multirow{3}{*}{PUMP} & A-R   & $\text{0.5130}$ & $\text{0.7778}_{\pm\text{0.0001}}$ & $\text{0.7476}$ & $\text{0.7476}_{\pm\text{0.0010}}$ & $\text{0.7738}$ & $\text{0.8084}_{\pm\text{0.0007}}$ & $\text{0.5245}$ & $\text{0.7904}_{\pm\text{0.0036}}$ & $\text{0.7390}$ & $\text{0.7981}_{\pm\text{0.0011}}$ \\
		& R-A-R & $\text{0.6664}$ & $\text{0.7860}_{\pm\text{0.0009}}$ & $\text{0.7781}$ & $\text{0.7793}_{\pm\text{0.0016}}$ & $\text{0.7651}$ & $\text{0.8010}_{\pm\text{0.0005}}$ & $\text{0.6660}$ & $\text{0.7929}_{\pm\text{0.0020}}$ & $\text{0.7392}$ & $\text{0.7935}_{\pm\text{0.0008}}$ \\
		& V-ROC & $\text{0.6522}$ & $\text{0.7841}_{\pm\text{0.0002}}$ & $\text{0.7749}$ & $\text{0.7757}_{\pm\text{0.0010}}$ & $\text{0.7665}$ & $\text{0.7996}_{\pm\text{0.0005}}$ & $\text{0.6546}$ & $\text{0.7912}_{\pm\text{0.0020}}$ & $\text{0.7396}$ & $\text{0.7920}_{\pm\text{0.0009}}$ \\
		\midrule
		\multirow{3}{*}{NYC} & A-R   & $\text{0.4508}$ & $\text{0.6291}_{\pm\text{0.0146}}$ & $\text{0.6599}$ & $\text{0.7849}_{\pm\text{0.0459}}$ & $\text{0.5654}$ & $\text{0.6870}_{\pm\text{0.0169}}$ & $\text{0.6249}$ & $\text{0.6824}_{\pm\text{0.0110}}$ & $\text{0.5669}$ & $\text{0.6386}_{\pm\text{0.0167}}$ \\
		& R-A-R & $\text{0.5394}$ & $\text{0.6422}_{\pm\text{0.0175}}$ & $\text{0.7481}$ & $\text{0.8302}_{\pm\text{0.0428}}$ & $\text{0.6764}$ & $\text{0.7484}_{\pm\text{0.0039}}$ & $\text{0.7546}$ & $\text{0.8142}_{\pm\text{0.0025}}$ & $\text{0.6161}$ & $\text{0.7322}_{\pm\text{0.0072}}$ \\
		& V-ROC & $\text{0.5355}$ & $\text{0.6428}_{\pm\text{0.0152}}$ & $\text{0.7241}$ & $\text{0.8136}_{\pm\text{0.0395}}$ & $\text{0.6608}$ & $\text{0.7263}_{\pm\text{0.0069}}$ & $\text{0.7332}$ & $\text{0.7846}_{\pm\text{0.0038}}$ & $\text{0.6043}$ & $\text{0.7177}_{\pm\text{0.0048}}$ \\
		\midrule
		\multicolumn{2}{c|}{$p$-\textbf{values}} & \textbf{---} & \textbf{0.0002} & \textbf{---} & \textbf{0.0017} & \textbf{---} & \textbf{4.12e-08} & \textbf{---} & \textbf{0.0001} & \textbf{---} & \textbf{4.58e-08} \\
        \midrule
        Dataset & Metric & TimesFM & NCL   & MOMENT & NCL   & UniTS & NCL   & TSPulse & NCL   & SEMPO & NCL \\
		\midrule
		\multirow{3}{*}{IOPS} & A-R   & $\text{0.6625}$ & $\text{0.8219}_{\pm\text{0.0037}}$ & $\text{0.6757}$ & $\text{0.7454}_{\pm\text{0.0005}}$ & $\text{0.7446}$ & $\text{0.8155}_{\pm\text{0.0026}}$ & $\text{0.7234}$ & $\text{0.8583}_{\pm\text{0.0015}}$ & $\text{0.7046}$ & $\text{0.7699}_{\pm\text{0.0132}}$ \\
		& R-A-R & $\text{0.6966}$ & $\text{0.8461}_{\pm\text{0.0035}}$ & $\text{0.5501}$ & $\text{0.6539}_{\pm\text{0.0010}}$ & $\text{0.7388}$ & $\text{0.7453}_{\pm\text{0.0071}}$ & $\text{0.7549}$ & $\text{0.8689}_{\pm\text{0.0018}}$ & $\text{0.6968}$ & $\text{0.7563}_{\pm\text{0.0094}}$ \\
		& V-ROC & $\text{0.6929}$ & $\text{0.8447}_{\pm\text{0.0032}}$ & $\text{0.5478}$ & $\text{0.6542}_{\pm\text{0.0010}}$ & $\text{0.7375}$ & $\text{0.7439}_{\pm\text{0.0072}}$ & $\text{0.7521}$ & $\text{0.8679}_{\pm\text{0.0018}}$ & $\text{0.6958}$ & $\text{0.7540}_{\pm\text{0.0090}}$ \\
		\midrule
		\multirow{3}{*}{UCR} & A-R   & $\text{0.5617}$ & $\text{0.5970}_{\pm\text{0.0098}}$ & $\text{0.5411}$ & $\text{0.5203}_{\pm\text{0.0002}}$ & $\text{0.5506}$ & $\text{0.5566}_{\pm\text{0.0051}}$ & $\text{0.5803}$ & $\text{0.6005}_{\pm\text{0.0084}}$ & $\text{0.5513}$ & $\text{0.5584}_{\pm\text{0.0107}}$ \\
		& R-A-R & $\text{0.6490}$ & $\text{0.6639}_{\pm\text{0.0066}}$ & $\text{0.5790}$ & $\text{0.5833}_{\pm\text{0.0006}}$ & $\text{0.6055}$ & $\text{0.6116}_{\pm\text{0.0053}}$ & $\text{0.6680}$ & $\text{0.6655}_{\pm\text{0.0090}}$ & $\text{0.6262}$ & $\text{0.6241}_{\pm\text{0.0038}}$ \\
		& V-ROC & $\text{0.6415}$ & $\text{0.6591}_{\pm\text{0.0069}}$ & $\text{0.5781}$ & $\text{0.5776}_{\pm\text{0.0004}}$ & $\text{0.6021}$ & $\text{0.6083}_{\pm\text{0.0053}}$ & $\text{0.6591}$ & $\text{0.6601}_{\pm\text{0.0076}}$ & $\text{0.6210}$ & $\text{0.6198}_{\pm\text{0.0044}}$ \\
		\midrule
		\multirow{3}{*}{NAB} & A-R   & $\text{0.5648}$ & $\text{0.6328}_{\pm\text{0.0023}}$ & $\text{0.7029}$ & $\text{0.7476}_{\pm\text{0.0005}}$ & $\text{0.6841}$ & $\text{0.7097}_{\pm\text{0.0035}}$ & $\text{0.5758}$ & $\text{0.6117}_{\pm\text{0.0045}}$ & $\text{0.5862}$ & $\text{0.6148}_{\pm\text{0.0029}}$ \\
		& R-A-R & $\text{0.6725}$ & $\text{0.7192}_{\pm\text{0.0015}}$ & $\text{0.7461}$ & $\text{0.7973}_{\pm\text{0.0004}}$ & $\text{0.7458}$ & $\text{0.7625}_{\pm\text{0.0025}}$ & $\text{0.6665}$ & $\text{0.6890}_{\pm\text{0.0067}}$ & $\text{0.6932}$ & $\text{0.7026}_{\pm\text{0.0051}}$ \\
		& V-ROC & $\text{0.6653}$ & $\text{0.7141}_{\pm\text{0.0009}}$ & $\text{0.7435}$ & $\text{0.7944}_{\pm\text{0.0002}}$ & $\text{0.7434}$ & $\text{0.7613}_{\pm\text{0.0028}}$ & $\text{0.6586}$ & $\text{0.6812}_{\pm\text{0.0060}}$ & $\text{0.6807}$ & $\text{0.6927}_{\pm\text{0.0040}}$ \\
		\midrule
		\multirow{3}{*}{MSL} & A-R   & $\text{0.7101}$ & $\text{0.7798}_{\pm\text{0.0163}}$ & $\text{0.8043}$ & $\text{0.8204}_{\pm\text{0.0005}}$ & $\text{0.8067}$ & $\text{0.8216}_{\pm\text{0.0146}}$ & $\text{0.6664}$ & $\text{0.7575}_{\pm\text{0.0029}}$ & $\text{0.6869}$ & $\text{0.7762}_{\pm\text{0.0174}}$ \\
		& R-A-R & $\text{0.7980}$ & $\text{0.8395}_{\pm\text{0.0154}}$ & $\text{0.8326}$ & $\text{0.8740}_{\pm\text{0.0003}}$ & $\text{0.8394}$ & $\text{0.8560}_{\pm\text{0.0149}}$ & $\text{0.7431}$ & $\text{0.8064}_{\pm\text{0.0023}}$ & $\text{0.7525}$ & $\text{0.8134}_{\pm\text{0.0177}}$ \\
		& V-ROC & $\text{0.7873}$ & $\text{0.8321}_{\pm\text{0.0148}}$ & $\text{0.8322}$ & $\text{0.8653}_{\pm\text{0.0004}}$ & $\text{0.8333}$ & $\text{0.8487}_{\pm\text{0.0141}}$ & $\text{0.7356}$ & $\text{0.8019}_{\pm\text{0.0019}}$ & $\text{0.7473}$ & $\text{0.8114}_{\pm\text{0.0153}}$ \\
		\midrule
		\multirow{3}{*}{SMAP} & A-R   & $\text{0.6195}$ & $\text{0.6883}_{\pm\text{0.0132}}$ & $\text{0.8342}$ & $\text{0.7859}_{\pm\text{0.0011}}$ & $\text{0.7654}$ & $\text{0.7603}_{\pm\text{0.0047}}$ & $\text{0.6752}$ & $\text{0.7069}_{\pm\text{0.0068}}$ & $\text{0.7308}$ & $\text{0.7646}_{\pm\text{0.0041}}$ \\
		& R-A-R & $\text{0.7175}$ & $\text{0.7559}_{\pm\text{0.0070}}$ & $\text{0.8605}$ & $\text{0.8302}_{\pm\text{0.0010}}$ & $\text{0.8042}$ & $\text{0.8036}_{\pm\text{0.0004}}$ & $\text{0.7832}$ & $\text{0.7773}_{\pm\text{0.0059}}$ & $\text{0.8191}$ & $\text{0.8246}_{\pm\text{0.0035}}$ \\
		& V-ROC & $\text{0.7062}$ & $\text{0.7487}_{\pm\text{0.0068}}$ & $\text{0.8584}$ & $\text{0.8265}_{\pm\text{0.0010}}$ & $\text{0.8000}$ & $\text{0.7985}_{\pm\text{0.0012}}$ & $\text{0.7732}$ & $\text{0.7722}_{\pm\text{0.0063}}$ & $\text{0.8098}$ & $\text{0.8202}_{\pm\text{0.0030}}$ \\
		\midrule
		\multirow{3}{*}{PSM} & A-R   & $\text{0.5610}$ & $\text{0.6826}_{\pm\text{0.0131}}$ & $\text{0.6192}$ & $\text{0.7000}_{\pm\text{0.0014}}$ & $\text{0.6387}$ & $\text{0.7100}_{\pm\text{0.0062}}$ & $\text{0.5947}$ & $\text{0.6991}_{\pm\text{0.0014}}$ & $\text{0.5959}$ & $\text{0.6449}_{\pm\text{0.0095}}$ \\
		& R-A-R & $\text{0.5580}$ & $\text{0.6648}_{\pm\text{0.0126}}$ & $\text{0.5731}$ & $\text{0.6570}_{\pm\text{0.0014}}$ & $\text{0.6174}$ & $\text{0.6402}_{\pm\text{0.0060}}$ & $\text{0.5902}$ & $\text{0.6793}_{\pm\text{0.0031}}$ & $\text{0.5780}$ & $\text{0.6202}_{\pm\text{0.0125}}$ \\
		& V-ROC & $\text{0.5572}$ & $\text{0.6653}_{\pm\text{0.0125}}$ & $\text{0.5742}$ & $\text{0.6565}_{\pm\text{0.0014}}$ & $\text{0.6160}$ & $\text{0.6381}_{\pm\text{0.0060}}$ & $\text{0.5896}$ & $\text{0.6789}_{\pm\text{0.0036}}$ & $\text{0.5776}$ & $\text{0.6206}_{\pm\text{0.0124}}$ \\
		\midrule
		\multirow{3}{*}{SMD} & A-R   & $\text{0.6719}$ & $\text{0.7511}_{\pm\text{0.0016}}$ & $\text{0.7568}$ & $\text{0.7875}_{\pm\text{0.0009}}$ & $\text{0.7590}$ & $\text{0.7697}_{\pm\text{0.0009}}$ & $\text{0.7170}$ & $\text{0.7438}_{\pm\text{0.0004}}$ & $\text{0.7194}$ & $\text{0.7378}_{\pm\text{0.0024}}$ \\
		& R-A-R & $\text{0.6885}$ & $\text{0.7755}_{\pm\text{0.0011}}$ & $\text{0.5617}$ & $\text{0.6780}_{\pm\text{0.0037}}$ & $\text{0.6875}$ & $\text{0.6696}_{\pm\text{0.0039}}$ & $\text{0.7263}$ & $\text{0.7475}_{\pm\text{0.0003}}$ & $\text{0.7211}$ & $\text{0.7310}_{\pm\text{0.0020}}$ \\
		& V-ROC & $\text{0.6875}$ & $\text{0.7735}_{\pm\text{0.0012}}$ & $\text{0.5616}$ & $\text{0.6774}_{\pm\text{0.0035}}$ & $\text{0.6854}$ & $\text{0.6681}_{\pm\text{0.0036}}$ & $\text{0.7249}$ & $\text{0.7465}_{\pm\text{0.0003}}$ & $\text{0.7207}$ & $\text{0.7302}_{\pm\text{0.0018}}$ \\
		\midrule
		\multirow{3}{*}{SWAT} & A-R   & $\text{0.2567}$ & $\text{0.8202}_{\pm\text{0.0048}}$ & $\text{0.2730}$ & $\text{0.8258}_{\pm\text{0.0013}}$ & $\text{0.2339}$ & $\text{0.8274}_{\pm\text{0.0005}}$ & $\text{0.2444}$ & $\text{0.8283}_{\pm\text{0.0002}}$ & $\text{0.2721}$ & $\text{0.8298}_{\pm\text{0.0017}}$ \\
		& R-A-R & $\text{0.3604}$ & $\text{0.8228}_{\pm\text{0.0084}}$ & $\text{0.3918}$ & $\text{0.8079}_{\pm\text{0.0015}}$ & $\text{0.3523}$ & $\text{0.7533}_{\pm\text{0.0043}}$ & $\text{0.3585}$ & $\text{0.8044}_{\pm\text{0.0031}}$ & $\text{0.3861}$ & $\text{0.8296}_{\pm\text{0.0033}}$ \\
		& V-ROC & $\text{0.3519}$ & $\text{0.8116}_{\pm\text{0.0100}}$ & $\text{0.3784}$ & $\text{0.7885}_{\pm\text{0.0015}}$ & $\text{0.3417}$ & $\text{0.7148}_{\pm\text{0.0049}}$ & $\text{0.3489}$ & $\text{0.7845}_{\pm\text{0.0030}}$ & $\text{0.3729}$ & $\text{0.8086}_{\pm\text{0.0027}}$ \\
		\midrule
		\multirow{3}{*}{PUMP} & A-R   & $\text{0.5532}$ & $\text{0.7942}_{\pm\text{0.0037}}$ & $\text{0.7587}$ & $\text{0.5406}_{\pm\text{0.0009}}$ & $\text{0.5598}$ & $\text{0.7914}_{\pm\text{0.0009}}$ & $\text{0.5066}$ & $\text{0.7812}_{\pm\text{0.0002}}$ & $\text{0.7436}$ & $\text{0.7475}_{\pm\text{0.0034}}$ \\
		& R-A-R & $\text{0.6501}$ & $\text{0.7962}_{\pm\text{0.0018}}$ & $\text{0.7814}$ & $\text{0.7370}_{\pm\text{0.0006}}$ & $\text{0.6904}$ & $\text{0.7904}_{\pm\text{0.0006}}$ & $\text{0.6459}$ & $\text{0.7895}_{\pm\text{0.0004}}$ & $\text{0.7800}$ & $\text{0.7820}_{\pm\text{0.0018}}$ \\
		& V-ROC & $\text{0.6440}$ & $\text{0.7941}_{\pm\text{0.0020}}$ & $\text{0.7780}$ & $\text{0.7201}_{\pm\text{0.0006}}$ & $\text{0.6828}$ & $\text{0.7873}_{\pm\text{0.0006}}$ & $\text{0.6334}$ & $\text{0.7875}_{\pm\text{0.0003}}$ & $\text{0.7757}$ & $\text{0.7779}_{\pm\text{0.0020}}$ \\
		\midrule
		\multirow{3}{*}{NYC} & A-R   & $\text{0.5439}$ & $\text{0.5591}_{\pm\text{0.0129}}$ & $\text{0.3474}$ & $\text{0.3531}_{\pm\text{0.0051}}$ & $\text{0.8186}$ & $\text{0.8620}_{\pm\text{0.0036}}$ & $\text{0.5723}$ & $\text{0.7032}_{\pm\text{0.0220}}$ & $\text{0.7382}$ & $\text{0.7915}_{\pm\text{0.0131}}$ \\
		& R-A-R & $\text{0.6510}$ & $\text{0.6660}_{\pm\text{0.0105}}$ & $\text{0.4945}$ & $\text{0.4983}_{\pm\text{0.0027}}$ & $\text{0.8354}$ & $\text{0.8513}_{\pm\text{0.0016}}$ & $\text{0.7318}$ & $\text{0.7605}_{\pm\text{0.0211}}$ & $\text{0.7813}$ & $\text{0.8205}_{\pm\text{0.0074}}$ \\
		& V-ROC & $\text{0.6361}$ & $\text{0.6513}_{\pm\text{0.0087}}$ & $\text{0.4989}$ & $\text{0.5040}_{\pm\text{0.0027}}$ & $\text{0.8264}$ & $\text{0.8434}_{\pm\text{0.0017}}$ & $\text{0.7106}$ & $\text{0.7490}_{\pm\text{0.0229}}$ & $\text{0.7718}$ & $\text{0.8093}_{\pm\text{0.0095}}$ \\
		\midrule
		\multicolumn{2}{c|}{$p$-\textbf{values}} & \textbf{---} & \textbf{4.28e-05} & \textbf{---} & \textbf{0.0237} & \textbf{---} & \textbf{0.0092} & \textbf{---} & \textbf{0.0002} & \textbf{---} & \textbf{0.0066} \\
		\bottomrule
	\end{tabular}%
	}
	\label{tab:full_roc}%
\end{table}%